\documentclass{article}

\PassOptionsToPackage{numbers, compress}{natbib}

\usepackage[preprint]{neurips_2025}

\usepackage[utf8]{inputenc} % allow utf-8 input
\usepackage[T1]{fontenc}    % use 8-bit T1 fonts
\usepackage{hyperref}       % hyperlinks
\usepackage{url}            % simple URL typesetting
\usepackage{booktabs}       % professional-quality tables
\usepackage{amsfonts}       % blackboard math symbols
\usepackage{nicefrac}       % compact symbols for 1/2, etc.
\usepackage{microtype}      % microtypography
\usepackage{xcolor}         % colors
\usepackage{amsmath}        % for \underset and other math commands

\usepackage{graphicx}
\usepackage{float}
\usepackage{booktabs}
\usepackage{amsthm}
\usepackage{algorithm}
\let\origAND\AND
\usepackage{algorithmic}

\let\AND\origAND
\usepackage{listings}

\newtheorem{lemma}{Lemma}

\title{Coupled Distributional Random Expert Distillation \\for World Model Online Imitation Learning}

\author{%
\begin{minipage}[t]{\textwidth}
\centering
\textbf{Shangzhe Li\textsuperscript{1}\thanks{This work was done during an internship at University of California, San Diego.}} \quad
\textbf{Zhiao Huang\textsuperscript{2}} \quad
\textbf{Hao Su\textsuperscript{2,3}} \\
\vspace{0.3cm}
\textsuperscript{1}\textnormal{UNC Chapel Hill} \quad
\textsuperscript{2}\textnormal{Hillbot} \quad
\textsuperscript{3}\textnormal{University of California, San Diego}
\end{minipage}
}

\begin{document}

\maketitle

\begin{abstract}
  Imitation Learning (IL) has achieved remarkable success across various domains, including robotics, autonomous driving, and healthcare, by enabling agents to learn complex behaviors from expert demonstrations. However, existing IL methods often face instability challenges, particularly when relying on adversarial reward or value formulations in world model frameworks. In this work, we propose a novel approach to online imitation learning that addresses these limitations through a reward model based on random network distillation (RND) for density estimation. Our reward model is built on the joint estimation of expert and behavioral distributions within the latent space of the world model. We evaluate our method across diverse benchmarks, including DMControl, Meta-World, and ManiSkill2, showcasing its ability to deliver stable performance and achieve expert-level results in both locomotion and manipulation tasks. Our approach demonstrates improved stability over adversarial methods while maintaining expert-level performance.
\end{abstract}

\section{Introduction}

Imitation Learning (IL) has recently shown remarkable effectiveness across a wide range of domains, particularly in addressing complex real-world challenges. In robotics, IL has significantly advanced the state of the art in manipulation tasks \citep{zhu2022viola,wan2024lotus,stepputtis2020language,chi2023diffusionpolicy}, enabling robots to perform intricate operations with precision and adaptability. Similarly, IL has achieved impressive results in locomotion tasks \citep{chiu2024learning,seo2023deep,huang2024diffuseloco}, where it has facilitated the development of robust and agile motion controllers for various robotic platforms. Beyond robotics, IL has also demonstrated its versatility in domains such as autonomous driving \citep{pan2017agile,bronstein2022hierarchical,cheng2024pluto}, where it is used to model complex decision-making processes and ensure safe and efficient vehicle navigation. Moreover, IL has started making meaningful contributions to healthcare \citep{deuschel2023contextualized}, providing support in medical decision-making and enhancing the interpretability of complex diagnostic processes. These achievements highlight the broad applicability of IL and its potential to drive transformative progress across diverse fields.

The simplest approach to imitation learning is to apply behavioral cloning directly to the provided expert dataset, as demonstrated in prior works like IBC \citep{florence2022implicit} and Diffusion Policy \citep{chi2023diffusionpolicy}. However, this approach is not dynamics aware and may result in lack of generalization when encountering out-of-distribution states. To address these shortcomings, methods like GAIL \citep{ho2016generative}, SQIL \citep{reddy2019sqil}, IQ-Learn \citep{garg2021iq}, MAIL \citep{baram2016model} and CFIL \citep{freund2023coupled} have introduced value or reward estimation to facilitate a deeper understanding of the environment, while leveraging online interactions to enhance exploration. Specifically, GAIL, MAIL, and IQ-Learn frame the imitation learning problem as an adversarial training process, distinguishing between the state-action distributions of the expert and the learner.

Recent advancements in latent world models for imitation learning have made significant progress. Several prior works, including V-MAIL \citep{rafailov2021visual}, CMIL \citep{kolev2024efficient}, Ditto \citep{demoss2023ditto}, EfficientImitate \citep{yin2022planning}, and IQ-MPC \citep{li2024rewardfreeworldmodelsonline}, have integrated adversarial imitation learning frameworks with world models to address imitation learning tasks. However, as discussed in Section \ref{sec:drawbacks-iqmpc}, we found that even with world models, the adversarial objectives can still suffer from instability in certain scenarios. To overcome this issue, we propose replacing the adversarial reward or value formulation with a novel density estimation approach based on random network distillation (RND) \citep{burda2018exploration}, which mitigates the instability. Specifically, we perform density estimation in the latent space of the world model, leveraging the superior properties of latent representations and their enhanced dynamics-awareness, as the latent dynamics model is trained directly within this space. Unlike existing methods that use RND for imitation learning \citep{wang2019random}, our approach jointly learns the reward model and other components of the world model, estimating both the expert and behavioral distributions simultaneously. In contrast, the existing Random Expert Distillation \citep{wang2019random} estimates distributions in the original observation and action spaces and decouples the reward model learning from the downstream RL process, making it hard to solve complex tasks with high dimensional observation and action spaces. We evaluate our approach across a range of tasks in DMControl \citep{tassa2018deepmind}, Meta-World \citep{yu2020meta}, and ManiSkill2 \citep{gu2023maniskill2}, demonstrating stable performance and achieving expert-level results.

In conclusion, the contributions of our work are summarized as follows:
\begin{itemize}
    \item We propose a novel reward model formulation for world model online imitation learning based on random network distillation for density estimation.
    \item We demonstrate that our approach exhibits superior stability compared to previous approaches with adversarial formulations and achieves expert-level performance across a range of imitation learning tasks, including both locomotion and manipulation.
\end{itemize}

\section{Preliminary}
We formulate our decision-making problem as Markov Decision Processes (MDPs). MDPs can be defined via a tuple $\langle\mathcal{S},\mathcal{A},p_0,\mathcal{P},r,\gamma\rangle$. In details, $\mathcal{S}$ and $\mathcal{A}$ represent the state and action spaces, $p_0$ is the initial state distribution, $\mathcal{P}:\mathcal{S}\times\mathcal{A}\rightarrow\Delta_\mathcal{S}$ depicts the transition probability, $r(\mathbf{s},\mathbf{a})$ is the reward function, and $\gamma\in(0,1)$ is the discount factor. Let $\mathcal{Z}$ denote the latent state space of the world model. The expert latent state-action distribution and the behavioral latent state-action distribution (induced by the behavioral policy $\pi$) over $\mathcal{Z} \times \mathcal{A}$ are denoted by $\rho_E$ and $\rho_\pi$, respectively.
\subsection{Random Network Distillation}
Random Network Distillation (RND) \citep{burda2018exploration} is a technique for promoting exploration. In details, it leverages a fixed randomly parameterized network $f_{\bar\theta}(x)$ and a learnable predictor network $f_{\theta}(x)$. During training, RND minimizes the following MSE loss for dataset $\mathcal{D}$ for certain data distribution $\rho$:
\begin{equation}
    \mathcal{L}_{RND}(\theta) = \mathbb{E}_{x\sim\mathcal{D}}\Vert f_{\bar\theta}(x) - f_{\theta}(x)\Vert_2^2
\end{equation}
During the evaluation, we obtain a data point $x'$ for unknown data distribution $\rho'$. By computing the L2 norm $\Vert f_{\bar\theta}(x') - f_{\theta}(x')\Vert_2^2$, we can estimate the difference between distribution $\rho$ and $\rho'$. This can also be interpreted as performing density estimation for the new data point $x'$ within the original distribution $\rho$. A similar methodology has been used in imitation learning and inverse reinforcement learning \citep{wang2019random}.
\subsection{World Models}
Recent world models in the context of robotics control and reinforcement learning often represent a model-based RL method with latent spaces. The model learns a latent state transition model $\mathbf{z}'=d_\theta(\mathbf{z},\mathbf{a})$, along with a encoder $\mathbf{z}=h_\theta(\mathbf{z})$ and a policy model $\mathbf{a}=\pi_\theta(\mathbf{z})$. The decision-making process often includes planning with latent unrolling. For models based on the Recurrent State-space Model (RSSM) \citep{hafner2019learning}, the latent states often are split into a deterministic part and a stochastic part. PlaNet \citep{hafner2019dream} and Dreamer series \citep{hafner2019learning, hafner2020mastering, hafner2023mastering} leverage decoders for observation reconstruction, while TD-MPC series \citep{hansen2022temporal, hansen2023td} leverages a decoder-free architecture and conducts planning solely in the latent space.

\begin{figure}[t]
    \centering
    \includegraphics[scale=0.4]{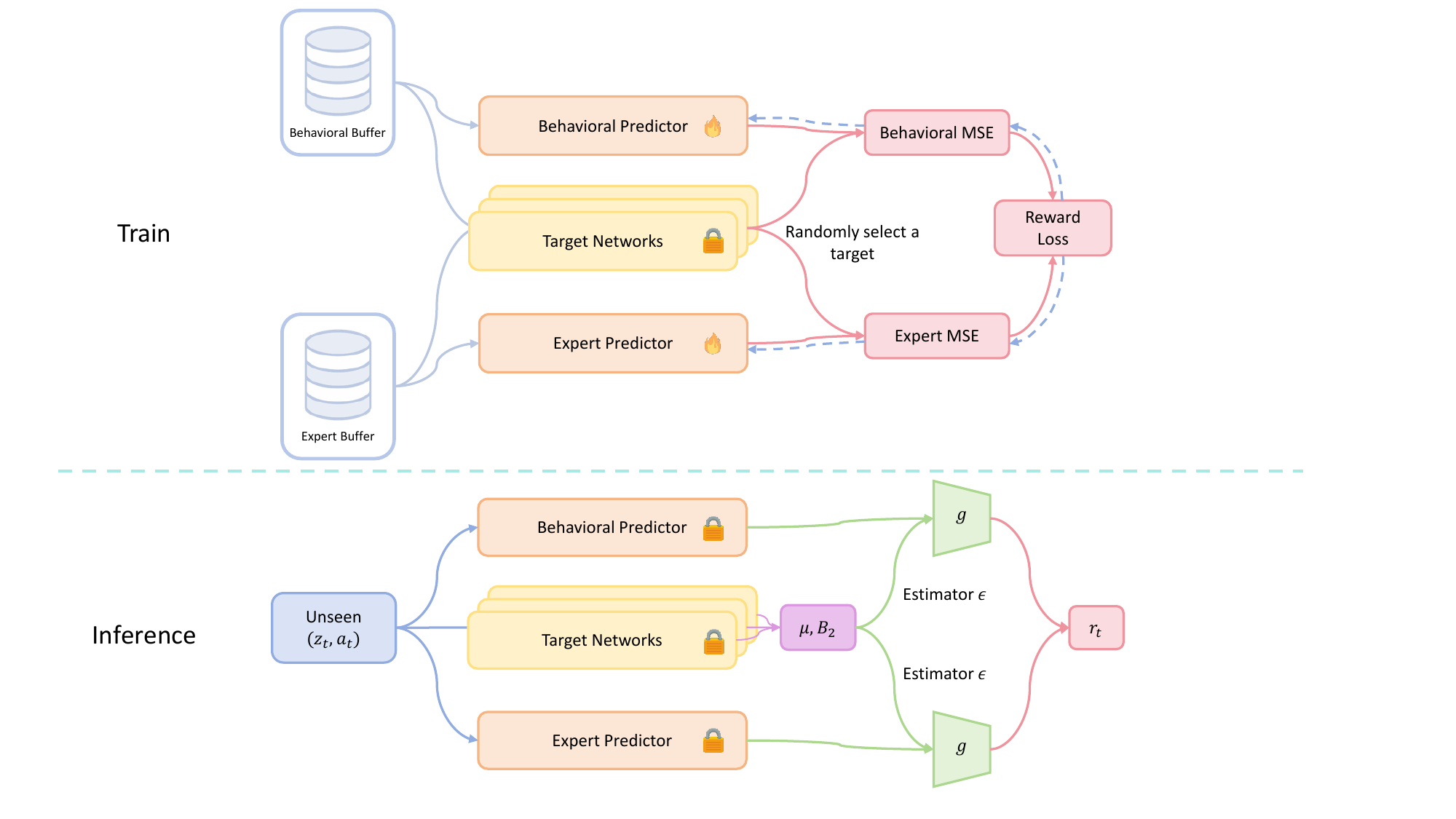}
    \caption{\textbf{Coupled Distributional Random Expert Distillation} We present the architecture of our CDRED reward model. During training, the behavioral and expert predictors are trained using latent representations encoded from observations and actions sampled from the behavioral and expert buffers. The {\color{blue}dotted blue lines} indicate the gradient backpropagation paths. During inference, rewards are estimated by the outputs of the behavioral and expert predictors, along with the mean and second-order moments of the target network's output, for an unseen latent state-action pair.}
    \label{fig:pipeline}
\end{figure}

\section{Methodology}
In this section, we will go over the motivation and detailed methodology of our method, \textbf{C}oupled \textbf{D}istributional \textbf{R}andom \textbf{E}xpert \textbf{D}istillation, or \textbf{CDRED} as an abbreviation. We show that our method is stabler and more reasonable compared to naively apply Random Expert Distillation (RED) \citep{wang2019random} on imitation learning with world models.
\begin{figure}
    \centering
    \includegraphics[scale=0.42]{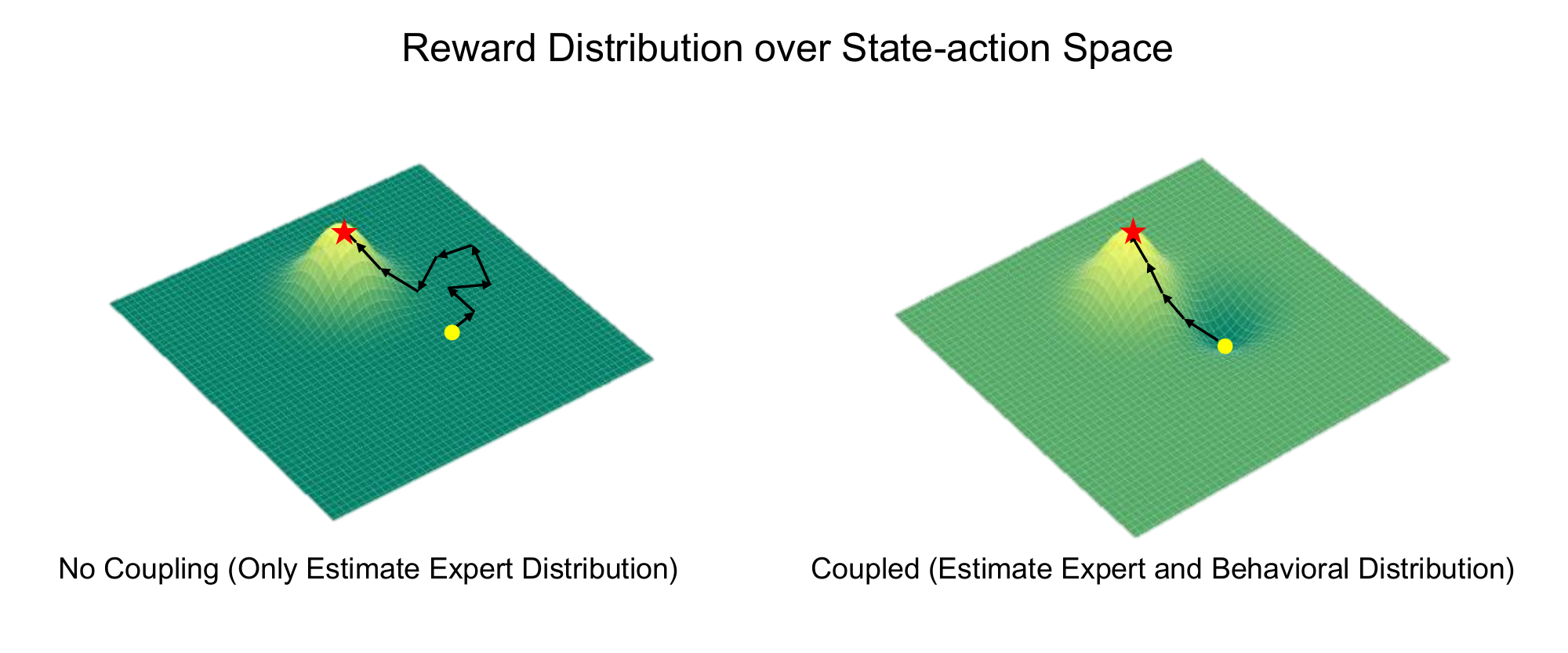}
    \caption{\textbf{Intuitive Illustration for Coupled Distribution Estimation} When the state-action distribution of the initial policy differs significantly from that of the expert distribution, the initial rewards tend to approach zero. This often leads to a slower or even unsuccessful learning process. By estimating the behavioral distribution in conjunction with the expert distribution, we can effectively model the rewards to guide the behavioral distribution closer to that of the expert.}
    \label{fig:illustration}
\end{figure}
\subsection{Motivation}
\label{sec:drawback-red}
Random Expert Distillation (RED) \citep{wang2019random} performs imitation learning by estimating the support of expert policy distribution. During training, it minimizes $K$ pairs of predictors and fixed random targets in expert dataset with $N$ data points $\mathcal{D}_E=\{\mathbf{s}_i,\mathbf{a}_i\}_{0:N}$:
\begin{equation}
    \hat\theta_k = \underset{\theta}{\text{argmin}}~\frac{1}{N}\sum_{i=0}^{N-1}(f_\theta(\mathbf{s}_i,\mathbf{a}_i)-f_{\bar\theta_k}(\mathbf{s}_i,\mathbf{a}_i))^2
\end{equation}
In order to determine if a state-action pair is within the support of expert policy, it computes the L2 norm deviation for an unknown state-action pair $(\mathbf{s},\mathbf{a})$ using $K$ pairs of predictors and targets:
\begin{equation}
    \mathcal{L}_{RED}(\mathbf{s},\mathbf{a})=\frac{1}{K}\sum_{k=0}^{K-1}(f_{\hat\theta_k}(\mathbf{s},\mathbf{a})-f_{\bar\theta_k}(\mathbf{s},\mathbf{a}))^2
\end{equation}
By leveraging a reward in the shape of \( r(\mathbf{s}, \mathbf{a}) = \exp(-\sigma~\mathcal{L}_{RED}(\mathbf{s}, \mathbf{a})) \), the approach effectively guides the downstream RL policy towards the expert distribution. However, this method may encounter challenges when the initial behavioral policy distribution is far from the expert distribution or when RED is applied naively on large latent spaces in world models.

To address these difficulties, we introduce a coupled approach. This approach jointly estimates both the expert distribution and the behavioral distribution; it encourages policy exploration during the early stages of training.  We provide an intuitive illustration in Figure \ref{fig:illustration} and describe the detailed methodology in Section \ref{sec:coupled-RED}. 
In this coupled approach, we need to estimate the behavioral distribution during online training,  which naturally raises the problem of inconsistent final rewards, as noted by \cite{yang2024exploration}. Thus, we adopt their method for tracking the frequency of data occurrence, which we describe in Section \ref{sec:inconsistent-reward}.

% Additionally, RED decouples the reward model training from the downstream RL process. 

% Since the reward model is trained on the latent space of the world model and the encoder is initially untrained, we cannot pre-train the reward model before the world model. 

\begin{figure}[h]
    \centering
    \begin{minipage}{0.45\textwidth}
        \centering
        \includegraphics[width=\textwidth]{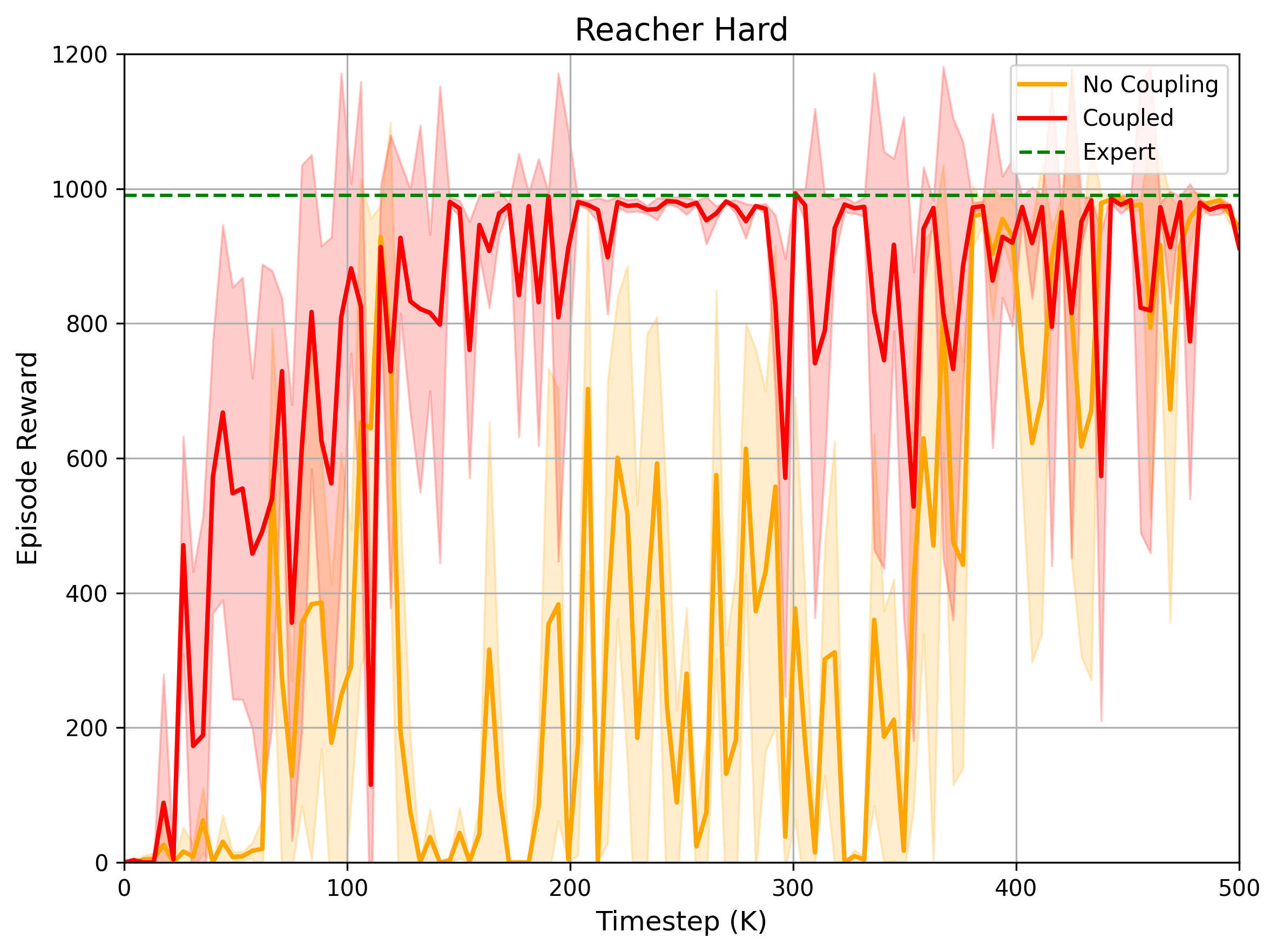}
        \\ \textbf{Faster convergence}
    \end{minipage}
    \begin{minipage}{0.45\textwidth}
        \centering
        \includegraphics[width=\textwidth]{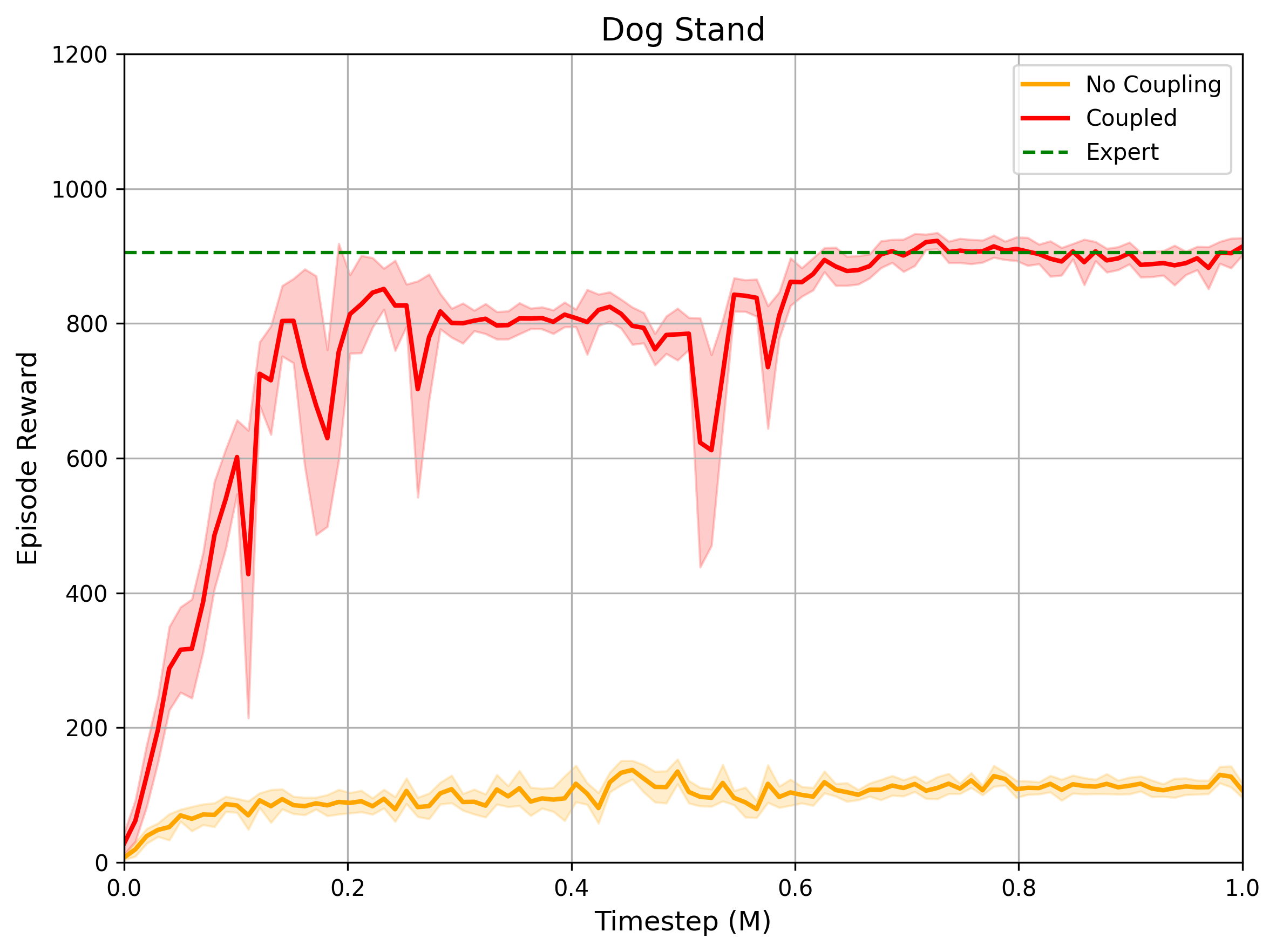}
        \\ \textbf{Better performance in complex settings}
    \end{minipage}
    \caption{\textbf{Advantages of Coupled Density Estimation} We demonstrate the empirical performance boost of our coupled density estimation in terms of leveraging random network distillations for reward modeling based on state-action distribution estimation. With coupled estimation, we observe faster convergence to optimal in many simple cases (Left) and better performance in complex tasks (Right).}
    \label{fig:iqmpc-drawbacks}
\end{figure}

\subsection{Mitigating Inconsistent Reward Estimation}
\label{sec:inconsistent-reward}
Inconsistencies can arise at various stages of RND model training \citep{yang2024exploration}. During the initial stage, these inconsistencies stem from extreme values in one network, which can be mitigated by using multiple target networks (denoted as $K$ target networks). In the final stage, inconsistencies occur when the resulting reward distribution does not align with the actual state-action distribution. To address this, an unbiased estimator for the state-action occurrence count $n$ is necessary. We should track state-action occurrence frequencies in order to maintain consistency when the distributional RND model is trained online. In this section, we replace the original state $\mathbf{s}_t$ with the latent representation $\mathbf{z}_t$ for the input of the RND model. Following \cite{yang2024exploration}, we denote the random variable $c(\mathbf{z}_t, \mathbf{a}_t)$ as the output of a target network $f_{\bar{\theta}_k}$, where $k$ is sampled uniformly from the interval $[0, K)$. For a predictor $f$ estimating a distribution $\rho$ (which can be either the expert distribution $\rho_E$ or the behavioral distribution $\rho_\pi$), by minimizing the $L_2$-norm loss $\| f(\mathbf{z}_t, \mathbf{a}_t) - c(\mathbf{z}_t, \mathbf{a}_t) \|_2^2$, the optimal predictor $f^*(\mathbf{z}_t, \mathbf{a}_t)$ is given by:

\begin{equation}
\label{eqn:optimal-predictor}
    f^*(\mathbf{z}_t,\mathbf{a}_t)=\frac{1}{n}\sum_{i=1}^{n}c_i(\mathbf{z}_t,\mathbf{a}_t)
\end{equation}

where $c_i(\mathbf{z}_t,\mathbf{a}_t)$ is representing the $c(\mathbf{z}_t,\mathbf{a}_t)$ for the $i$-th occurrence for state-action pair $(\mathbf{z}_t,\mathbf{a}_t)$ in distribution $\rho$.  In order to track the occurrence count $n$, we adopt a lemma proposed by \cite{yang2024exploration}:

\begin{lemma}[Unbiased Estimator]
    For a state-action distribution $\rho$, $f^*$ is the optimal predictor on this distribution defined in Eq. \ref{eqn:optimal-predictor}, the following statistic is an unbiased estimator of $1/n$ with consistency for this distribution:
    \begin{equation*}
        y(\mathbf{z}_t,\mathbf{a}_t) = \frac{[f^*(\mathbf{z}_t,\mathbf{a}_t)]^2-[\mu_{\bar\theta}(\mathbf{z}_t,\mathbf{a}_t)]^2}{B_2(\mathbf{z}_t,\mathbf{a}_t)-[\mu_{\bar\theta}(\mathbf{z}_t,\mathbf{a}_t)]^2}
    \end{equation*}
    where the second-order moment is:
    \begin{equation*}
        B_2(\mathbf{z}_t,\mathbf{a}_t)=\frac{1}{K}\sum_{k=0}^{K-1}[f_{\bar\theta_k}(\mathbf{z}_t,\mathbf{a}_t)]^2
    \end{equation*}
\label{lem:unbiased-estimator}
\end{lemma}
\begin{proof}
    See Appendix \ref{sec:additional-proof} or prior work \citep{yang2024exploration}.
\end{proof}

In this way, we are able to estimate the data distribution with higher consistency as the training proceeds.  Following \cite{yang2024exploration}, we construct the following estimator for $\sqrt{1/n}$ as an additional bonus correction term:
\begin{equation}
\label{eqn:correction-term}
    \epsilon(\mathbf{z}_t,\mathbf{a}_t, f) = \sqrt{\frac{[f(\mathbf{z}_t,\mathbf{a}_t)]^2-[\mu_{\bar\theta}(\mathbf{z}_t,\mathbf{a}_t)]^2}{B_2(\mathbf{z}_t,\mathbf{a}_t)-[\mu_{\bar\theta}(\mathbf{z}_t,\mathbf{a}_t)]^2}}
\end{equation}
This bonus correction is incorporated into the reward model construction discussed in Section \ref{sec:coupled-RED}.

\subsection{Coupled Distributional Random Expert Distillation}
\label{sec:coupled-RED}
We construct a reward model with two predictor networks that share the same random target ensemble on the latent space of a world model.  The distributional random target ensemble consists of $K$ random networks $\{f_{\bar\theta_k}\}_{0:K}$ with fixed parameters. Regarding the predictors, one of them is the expert predictor $f_\phi$ while the other is the behavioral predictor $f_\psi$. A predictor $f$ is defined by $f:\mathcal{Z}\times\mathcal{A}\rightarrow\mathbb{R}^{p}$, while $p$ is the dimension of the low-dimensional embedding space for L2 norm distance computation. 
Following \cite{yang2024exploration}, we ask these two predictors to learn the random targets sampled. This is different to RED which learn $K$ predictors for $K$ targets. Given an expert buffer $\mathcal{B}_E$ and a behavioral buffer $\mathcal{B}_\pi$, we aim to optimize through the following objective:
\begin{equation}
\label{eqn:reward-obj}
\begin{split}
        \mathcal{L}^r(\phi,\psi) = \sum_{t=0}^{H}\lambda^t~&\mathbb{E}_{k\sim\text{Uniform}(0,K)}\Bigg[\mathbb{E}_{(\mathbf{s}_t,\mathbf{a}_t)\sim\mathcal{B}_E}\Big[\Vert f_\phi(\mathbf{z}_t,\mathbf{a}_t)-f_{\bar\theta_k}(\mathbf{z}_t,\mathbf{a}_t)\Vert_2^2\Big]\\
        &+\mathbb{E}_{(\mathbf{s}_t,\mathbf{a}_t)\sim\mathcal{B}_\pi}\Big[\Vert f_\psi(\mathbf{z}_t,\mathbf{a}_t)-f_{\bar\theta_k}(\mathbf{z}_t,\mathbf{a}_t)\Vert_2^2\Big]\Bigg]
\end{split}
\end{equation}
We sample short trajectories with horizon $H$ from the replay buffers and sum up the loss for every step with a discounting factor $\lambda$. Note that this factor is different from the environment discount factor $\gamma$. We update every time with one target network $f_{\bar\theta_k}$, where index $k$ is sampled from a uniform distribution over integers ranging $[0,K)$. $\mathbf{z}_t$ is the latent representation of $\mathbf{s}_t$ with an encoder mapping $\mathbf{z}_t=h(\mathbf{s}_t)$. In this way, we can obtain the estimation for expert distribution $\rho_E$ and behavioral distribution $\rho_\pi$. Furthermore, it enables us to construct a reward model based on the distribution estimations. Incorporating the bias correction term introduced in Eq. \ref{eqn:correction-term}, we are able to compute the reward via:
\begin{equation}
\label{eqn:reward-formulation-modified}
R(\mathbf{z}_t,\mathbf{a}_t)=\zeta~g(-\sigma~b(\mathbf{z}_t,\mathbf{a}_t, f_\phi))-(1-\zeta)~g(-\sigma~b(\mathbf{z}_t,\mathbf{a}_t, f_\psi))
\end{equation}
where
\begin{equation}
\label{eqn:bonus}
    b(\mathbf{z}_t,\mathbf{a}_t, f) = \alpha~\Vert f(\mathbf{z}_t,\mathbf{a}_t)-\mu_{\bar\theta}(\mathbf{z}_t,\mathbf{a}_t)\Vert_2^2 + (1 - \alpha)~\epsilon(\mathbf{z}_t,\mathbf{a}_t, f)
\end{equation}
\begin{equation}
    \mu_{\bar\theta}(\mathbf{z}_t,\mathbf{a}_t)=\frac{1}{K}\sum_{k=0}^{K-1}f_{\bar\theta_k}(\mathbf{z}_t,\mathbf{a}_t)
\end{equation}
The first term in Eq. \ref{eqn:reward-formulation-modified} measures the distance between the current and expert distributions, while the second term encourages exploration by penalizing exploitation. A scaling factor $ \zeta $ balances these terms, with the second term dominating during early training when the policy is sub-optimal, promoting exploration. As the policy approaches optimality, the first term takes over, stabilizing the policy near the expert distribution. Typically, $ \zeta $ is close to 1, allowing the first term to dominate after initial exploration.

The coefficient $ \sigma $ controls the decay rate of the reward function, which is based on the expert distribution for the first term and the behavioral distribution for the second. To ensure stability, the reward is computed using the mean output of $ K $ random target networks. The function $ g(x) $ is monotonically increasing, and both $ g(x) = \exp(x) $ and $ g(x) = x $ work, with slight differences in behavior, as discussed in Appendix \ref{sec:ablation}.

The scalar coefficient $ \alpha $ in Eq. \ref{eqn:bonus} balances the contributions of the first term (the $ L_2 $-norm) and the second term (an estimator for $ \sqrt{1/n} $). Following \cite{yang2024exploration}, we let the first term dominate initially, switching to the second term as training progresses. This can be achieved with a fixed $ \alpha $, rather than a dynamic coefficient. This modification enables consistent online estimation of the state-action distribution, directly supporting reward modeling for online imitation learning.

\subsection{Integrating into World Models for Imitation Learning}
World models learn the policy and underlying environment dynamics by encoding the observations into a latent space and learning the transition model in the latent space. Decoder-free world models such as TD-MPC series \citep{hansen2022temporal,hansen2023td} has proved to be a powerful tool for complex reinforcement learning tasks. We leverage a decoder-free world model containing the following components:
\begin{align}
\text{Encoder:} &\quad \mathbf{z}_t = h(\mathbf{s}_t) \label{eq:encoder} \\
\text{Latent dynamics:} &\quad \mathbf{z}'_t = d(\mathbf{z}_t, \mathbf{a}_t) \label{eq:latent_dynamics} \\
\text{Value function:} &\quad \hat{q}_t = Q(\mathbf{z}_t, \mathbf{a}_t) \label{eq:value_function} \\
\text{Policy prior:} &\quad \mathbf{\hat a}_t = \pi(\mathbf{z}_t) \label{eq:policy_prior}\\
\text{CDRED model:} &\quad \hat r_t = R(\mathbf{z}_t,\mathbf{a}_t) \label{eq:reward-model}
\end{align}
The reward model, i.e., the CDRED model, consists of two predictors and $K$ target networks, estimating the expert and behavioral distributions for reward approximation. The encoder $h:\mathcal{S}\rightarrow\mathcal{Z}$ maps the observation (state-based or vision based) to latent representation. The latent dynamics model $d:\mathcal{Z}\times\mathcal{A}\rightarrow\mathcal{Z}$ learns the transition dynamics over the latent representations, implicitly modeling the environment dynamics. The value function learns to estimate the future return by training on temporal difference objective with the assist of the estimated rewards from the CDRED model. The policy prior learns a stochastic policy which guides the planning process of the world model. The training procedure is outlined in Algorithm \ref{alg:training}, while the planning process is detailed in Algorithm \ref{alg:inference}.
\paragraph{Model Training}
The learnable parameters of the world model are denoted as three parts. While $\phi$ and $\psi$ denote the parameterization of expert predictor and behavioral predictor in the CDRED reward model, the rest of the parameters related to the encoder, latent dynamics, value model and policy prior are represented as $\xi$. Note that the parameters of the target networks $\bar\theta_k$ are not learnable. We train the encoder, dynamics model, value model, and reward model jointly with the following objective:
\begin{equation}
\label{eqn:model-loss}
\mathcal{L}(\phi,\psi,\xi)=\sum_{t=0}^H ~\mathbb{E}_{(\mathbf{s}_t,\mathbf{a}_t,\mathbf{s}'_t)\sim\mathcal{B}_E\cup\mathcal{B}_\pi}\underbrace{\Bigg[\lambda^t
\Big(\Vert\mathbf{z}'_t-\text{sg}(h(\mathbf{s}'_t))\Vert^2_2 + \text{CE}(\hat q_t, q_t)\Big)\Bigg]}_{\text{Consistency and TD Loss}} 
+ \underbrace{\mathcal{L}^r(\phi,\psi)}_{\text{CDRED Loss}}
\end{equation}
The first term contains consistency loss and temporal difference loss to ensure the prediction consistency of the dynamics model and the accuracy for value function estimation. the temporal difference target is computed by $q_t=R(\mathbf{z}_t,\mathbf{a}_t)+\gamma Q(\mathbf{z}'_t,\pi(\mathbf{z}'_t))$ where $R(\mathbf{z}_t,\mathbf{a}_t)$ is the output of the reward model. We convert the regression TD objective into a classification problem for stabler value estimation, which is also used by the TD-MPC series and mentioned by \cite{farebrother2024stop}. $\text{CE}(\hat q_t,q_t)$ is the cross entropy loss between target Q value and current predicted value. The second term is the reward loss, which is shown in Eq.\ref{eqn:reward-obj}. Similar to the computation of reward loss, we also sum up the consistency and TD loss with factor $\lambda$ over a horizon $H$. 
\paragraph{Policy Prior Learning}
Regarding the policy prior update, we adopt maximum entropy objective \citep{haarnoja2018soft} to train a stochastic policy:
\begin{equation}
\label{eqn:policy-learning}
\mathcal{L}^{\pi}(\xi)=\sum_{t=0}^H \lambda^t\Bigg[\mathbb{E}_{(\mathbf{s}_t,\mathbf{a}_t) \sim \mathcal{B}_E\cup\mathcal{B}_\pi}\Big[-Q(\mathbf{z}_t,\pi(\mathbf{z}_t))+\beta\log(\pi(\cdot|\mathbf{z}_t))\Big]\Bigg]
\end{equation}
We use short trajectories with horizon $H$ sampled from both expert and behavioral buffers for policy updates. We sum up the policy loss over the horizon with the same discount factor $\lambda$. $\beta$ is a fixed scalar coefficient to balance the entropy term and the Q value.
\paragraph{Planning} 
Following TD-MPC series \citep{hansen2022temporal,hansen2023td}, we also leverage model predictive path integral (MPPI) \citep{williams2015model} for planning. We optimize using the sampled action sequences $(\mathbf{a}_t,\mathbf{a}_{t+1},...,\mathbf{a}_{t+H})$ in a derivative-free style, maximizing the estimated return for the latent trajectories that have been rolled out using our dynamics model. Mathematically, our objective can be describe as a return maximizing process \citep{hansen2023td}:
\begin{equation}
    \mu^*,\sigma^*=\underset{(\mu,\sigma)}{\text{argmax}}\underset{(\mathbf{a}_t,\mathbf{a}_{t+1},...,\mathbf{a}_{t+H})\sim\mathcal{N}(\mu,\sigma^2)}{\mathbb{E}}\Bigg[\gamma^H Q(\mathbf{z}_{t+H},\mathbf{a}_{t+H})+\sum_{h=t}^{H-1}\gamma^h R(\mathbf{z}_h,\mathbf{a}_h)\Bigg]
\end{equation}
After planning, the agent interacts with the environment using the first action $\mathbf{a}_t\sim\mathcal{N}(\mu^*,(\sigma^*)^2)$ to obtain new observations. New trajectories are stored in behavioral buffer $\mathcal{B}_\pi$ for following training.

\section{Experiments}
\label{sec:experiments}
We conduct experiments across a diverse range of tasks, including locomotion, manipulation, and tasks with both visual and state-based observations. We evaluate our approach using the DMControl \citep{tassa2018deepmind}, Meta-World \citep{yu2020meta} and ManiSkill2 \citep{gu2023maniskill2} environments. As for the baselines, we compare our approach with IQ-MPC \citep{li2024rewardfreeworldmodelsonline}, which integrates a world model architecture, as well as with model-free approaches, specifically IQ-Learn+SAC \citep{garg2021iq} (referred to as IQL+SAC in the plots), CFIL+SAC \citep{freund2023coupled}, HyPE \citep{ren2024hybrid} (In Appendix \ref{sec:hype}) and SAIL \citep{SAIL} (In Appendix \ref{sec:sail}). Additionally, we also incorporate behavioral cloning (BC) as a baseline method in our evaluation. For all experiments, we sample expert trajectories from a trained TD-MPC2 \citep{hansen2023td}. Additionally, we conduct ablation studies on the number of expert trajectories and the choice of function $g$, as detailed in Appendix \ref{sec:ablation}. We further evaluate the robustness of our algorithm in noisy environment dynamics (Appendix \ref{sec:noisy-dynamics}), examine the benefits of constructing the reward model in the latent space (Appendix \ref{sec:latent-space}), and highlight its advantages over existing adversarial training methods (Appendix \ref{sec:drawbacks-iqmpc}). We also provide the quantitative results measuring the training stability in Appendix \ref{sec:training-stability}. All of the experiments are conducted on a single RTX3090 graphic card.

\begin{figure}[h]
    \centering
    \begin{minipage}{0.3\textwidth}
        \centering
        \includegraphics[width=\textwidth]{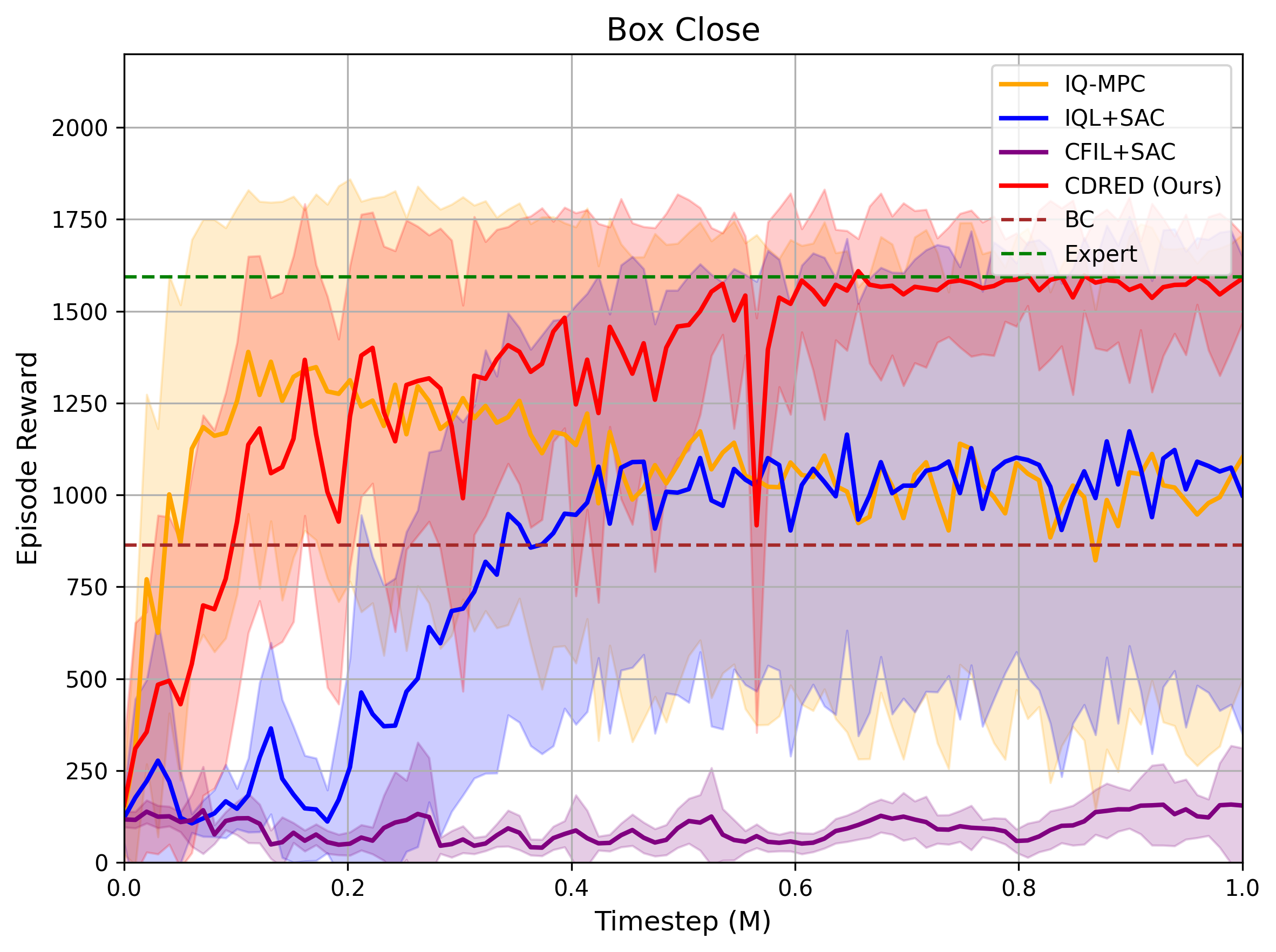}
    \end{minipage}
    \begin{minipage}{0.3\textwidth}
        \centering
        \includegraphics[width=\textwidth]{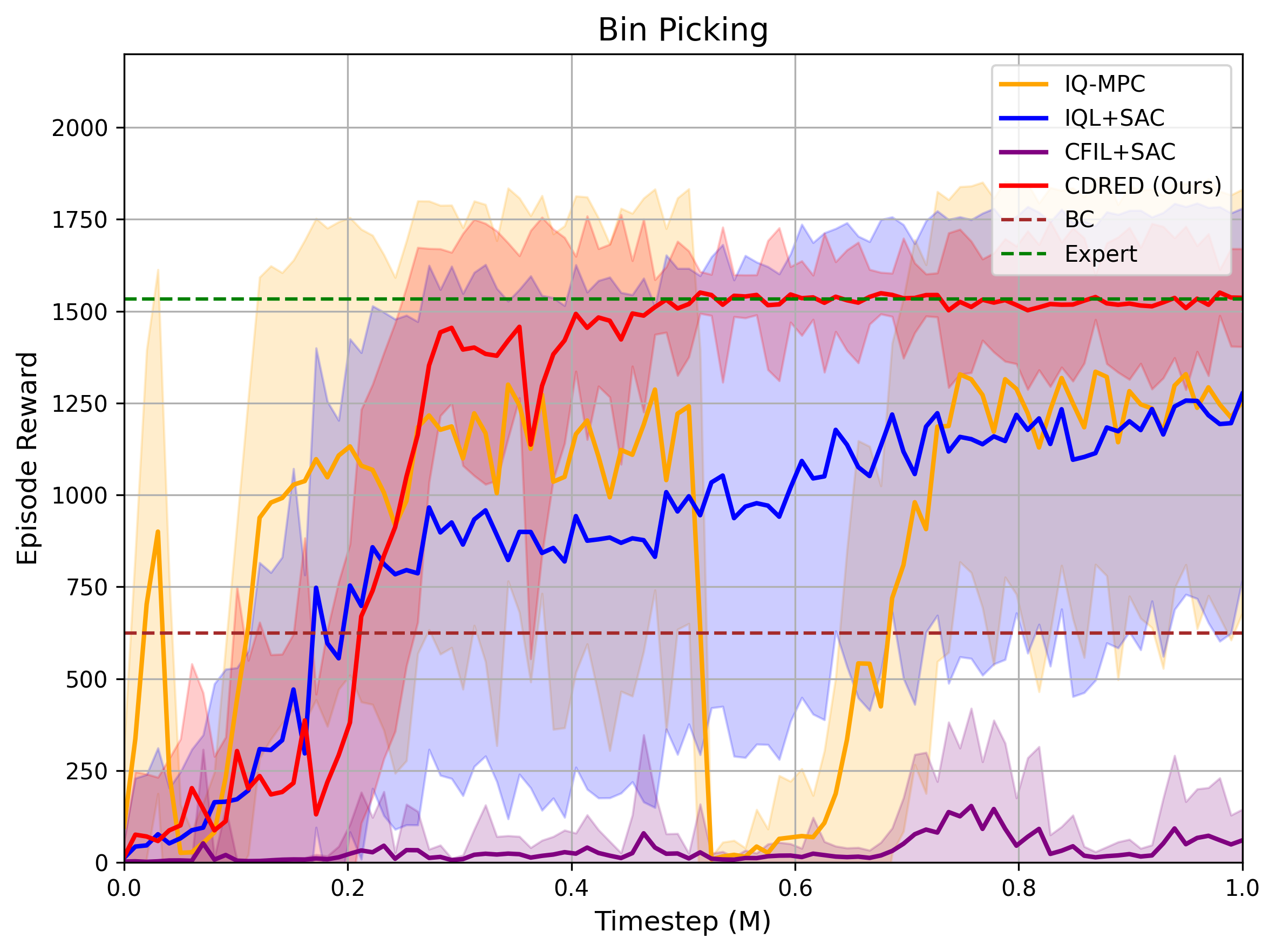}
    \end{minipage}
    \begin{minipage}{0.3\textwidth}
        \centering
        \includegraphics[width=\textwidth]{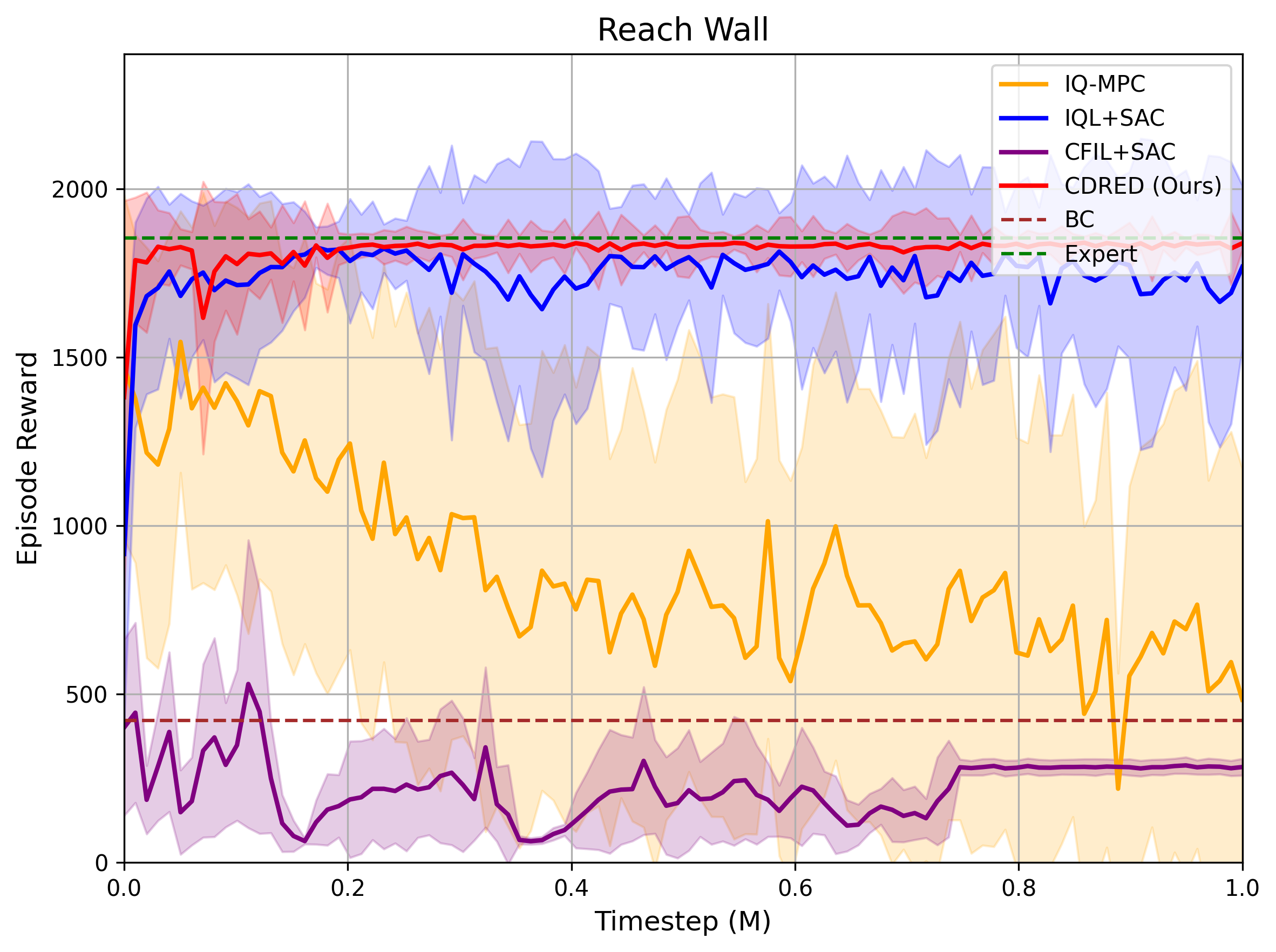}
    \end{minipage}
    
    \vspace{10pt} % Add some space between rows
    
    \begin{minipage}{0.3\textwidth}
        \centering
        \includegraphics[width=\textwidth]{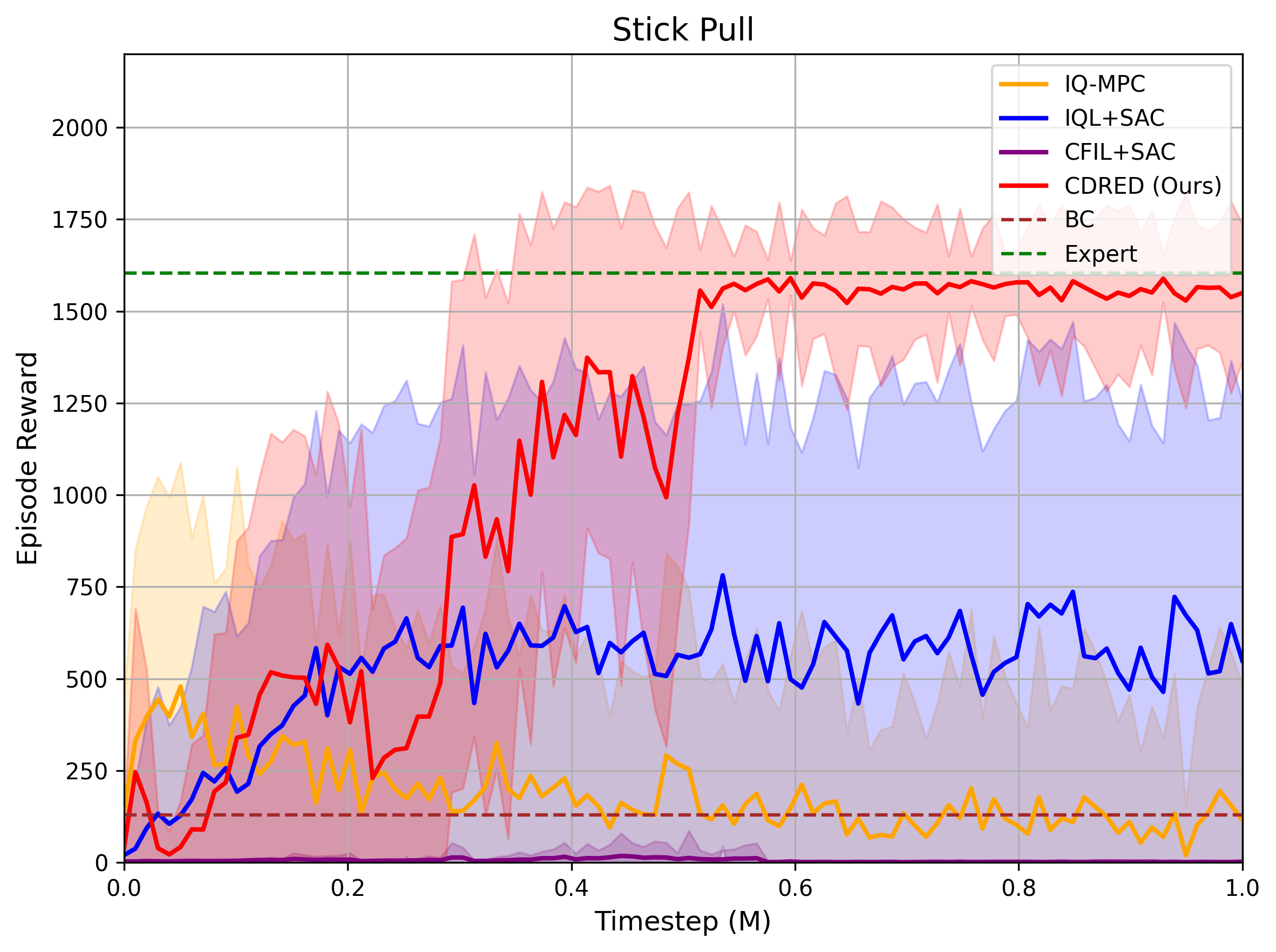}
    \end{minipage}
    \begin{minipage}{0.3\textwidth}
        \centering
        \includegraphics[width=\textwidth]{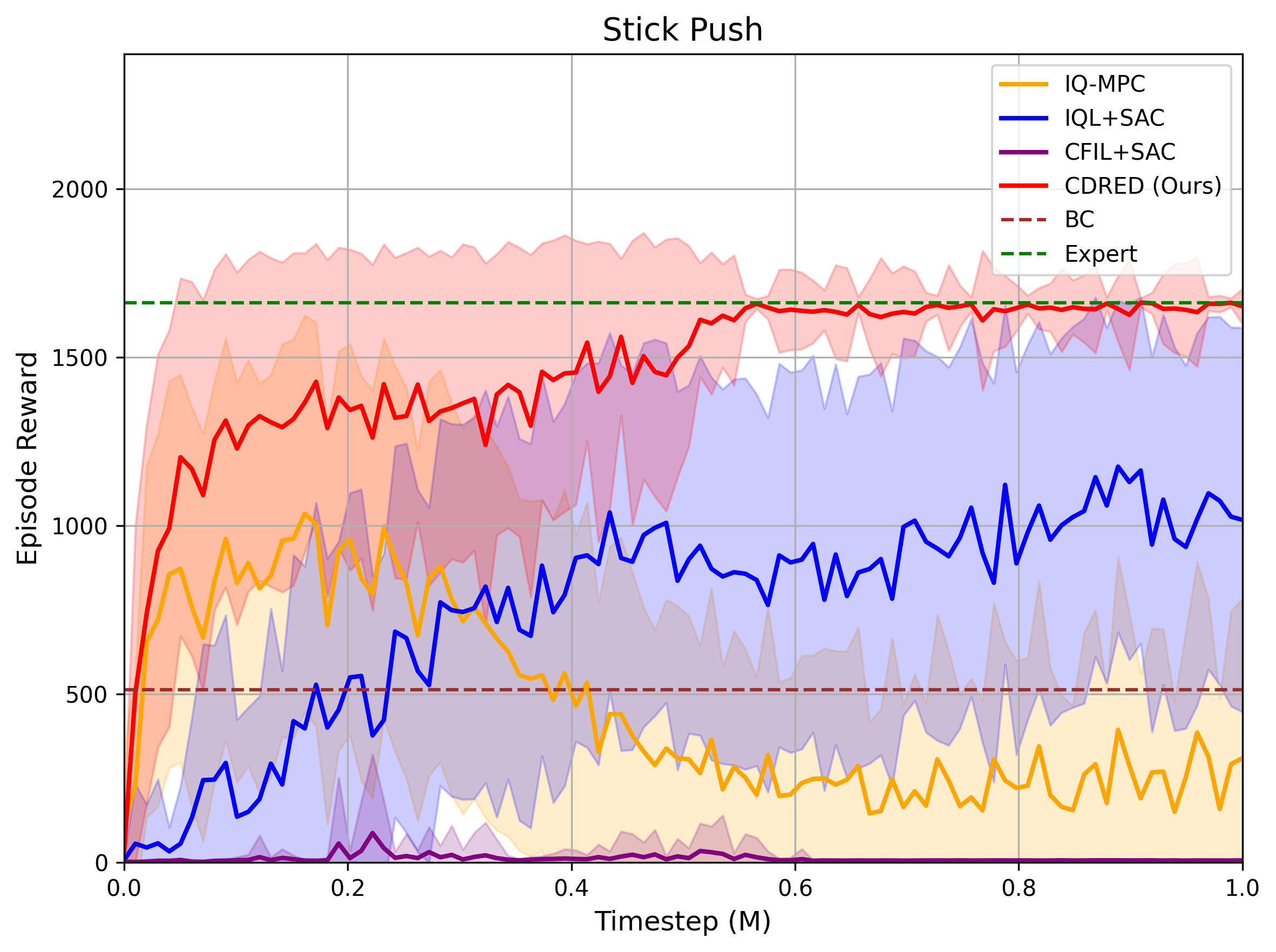}
    \end{minipage}
    \begin{minipage}{0.3\textwidth}
        \centering
        \includegraphics[width=\textwidth]{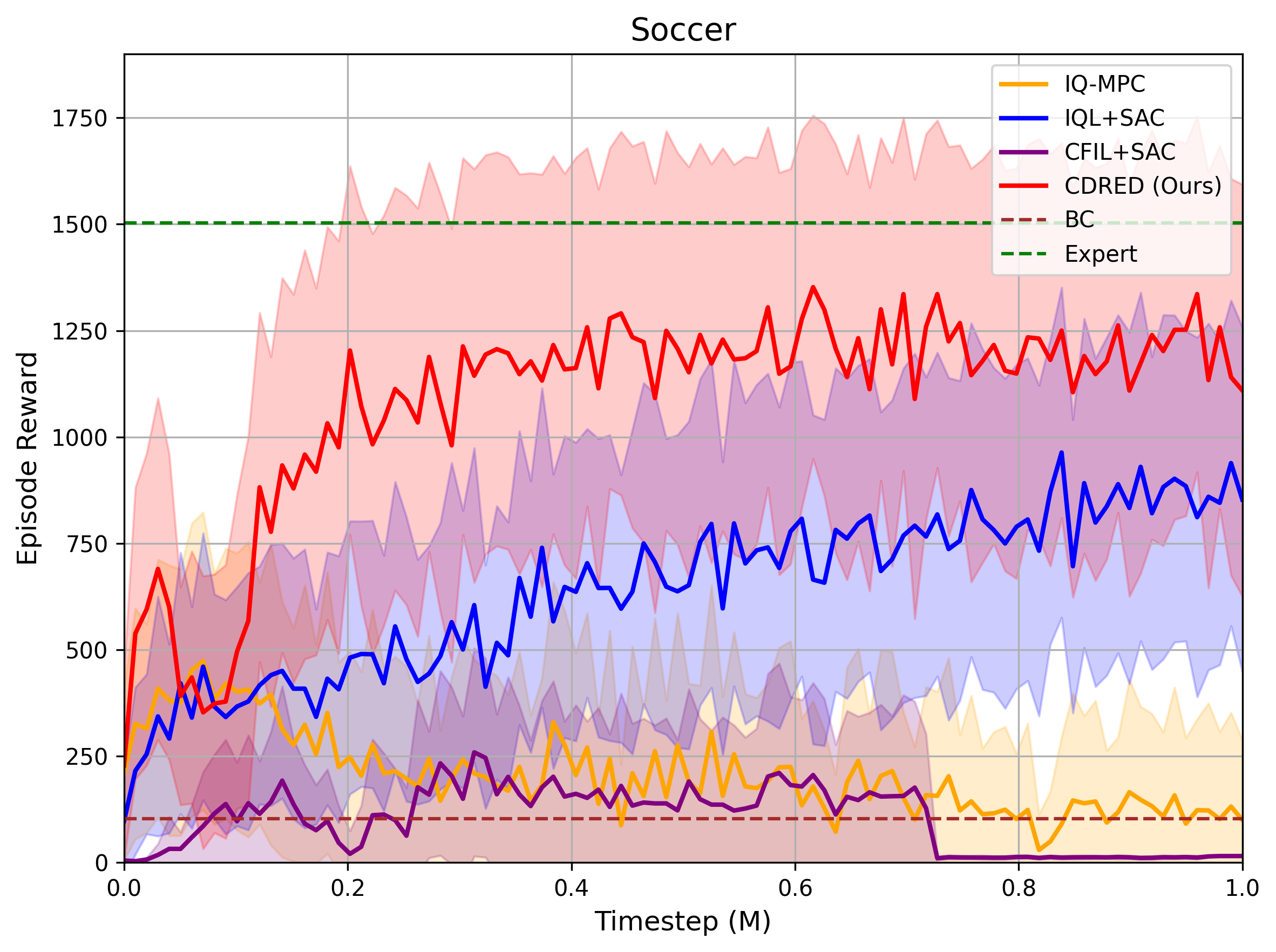}
    \end{minipage}
    \caption{\textbf{Meta-World Results} We evaluate our CDRED method ({\color{red}red lines}) on 6 tasks in Meta-World environments. We show stabler performance on these tasks, outperforming the baselines. IQ-MPC ({\color{orange}orange lines}) suffers from overly powerful discriminator problem mentioned in Section \ref{sec:drawbacks-iqmpc}. We conduct the experiments on 3 random seeds.}
    \label{fig:manipulation-results}
\end{figure}

\subsection{Meta-World Experiments}
We conduct experiments on 6 tasks in Meta-World environments. We use 100 expert trajectories for each task, ensuring that the expert data remains consistent across all algorithms for fair comparison within each task. IQ-MPC suffers from overly powerful discriminators in these tasks, even with gradient penalty applied, due to the adversarial training methodology. CFIL+SAC \citep{freund2023coupled} encounters instability in the training process due to the challenges inherent in training flow models. We show stable and expert-level performance, outperforming these baselines in these tasks. We show the episode reward results in Figure \ref{fig:manipulation-results} and success rate results in Table \ref{tab:manipulation-success-meta-world}.

\begin{table}[h]
    \centering
    \begin{tabular}{c|ccccc}
        \toprule
        \textbf{Method} & \textbf{BC} & \textbf{IQL+SAC} & \textbf{CFIL+SAC} & \textbf{IQ-MPC} & \textbf{CDRED(Ours)}\\
        \midrule
        Box Close & 0.58$\pm$0.12 & 0.61$\pm$0.09 & 0.00$\pm$0.00 & 0.53$\pm$0.18 & \textbf{0.96$\pm$0.03}\\
        Bin Picking & 0.43$\pm$0.18 & 0.75 $\pm$ 0.11 & 0.01$\pm$0.01 & 0.79$\pm$0.05 & \textbf{0.99$\pm$0.01}\\
        Reach Wall & 0.10$\pm$0.08 & 0.90$\pm$0.04 & 0.01$\pm$0.01 & 0.31$\pm$0.14& \textbf{0.98$\pm$0.01}\\
        Stick Pull & 0.02$\pm$0.02 & 0.34$\pm$0.11 & 0.00$\pm$0.00 & 0.13$\pm$0.08 & \textbf{0.92$\pm$0.05}\\
        Stick Push & 0.42$\pm$0.14 & 0.76$\pm$0.14 & 0.00$\pm$0.00 & 0.23$\pm$0.10 & \textbf{0.94$\pm$0.03}\\
        Soccer & 0.04$\pm$0.03 & 0.73$\pm$0.09 & 0.01$\pm$0.01 & 0.12$\pm$0.07 & \textbf{0.81$\pm$0.05}\\
        \bottomrule
    \end{tabular}
    \vspace{+5pt}
    \caption{\textbf{Manipulation Success Rate Results in Meta-World} We show the success rate comparison on 6 tasks in Meta-World. Our CDRED model demonstrates outperforming results compared to existing methods. We compute the success rates over 100 episodes. We evaluate our model and other baselines on 3 random seeds.}
    \label{tab:manipulation-success-meta-world}
\end{table}

\subsection{DMControl Experiments}
We conduct experiments on 6 tasks in DMControl \citep{tassa2018deepmind} environments. For low-dimensional tasks, we utilize 100 expert trajectories, while for high-dimensional tasks, we use 500 expert trajectories. Details on environment dimensionality can be found in Appendix \ref{sec:env-spec}. Our CDRED model performs comparably to IQ-MPC on the Hopper Hop, Walker Run, and Humanoid Walk tasks. However, in Cheetah Run and Dog Stand, IQ-MPC experiences long-term instability, causing the agent to fail after extensive online training. On the Reacher Hard task, IQ-MPC struggles with an overly powerful discriminator, which prevents it from learning an expert-level policy. The model-free methods in baseline algorithms fail to achieve stable, expert-level performance on these tasks. The episode reward results are shown in Figure \ref{fig:locomotion-results}.

\begin{figure}[h]
    \centering
    \begin{minipage}{0.3\textwidth}
        \centering
        \includegraphics[width=\textwidth]{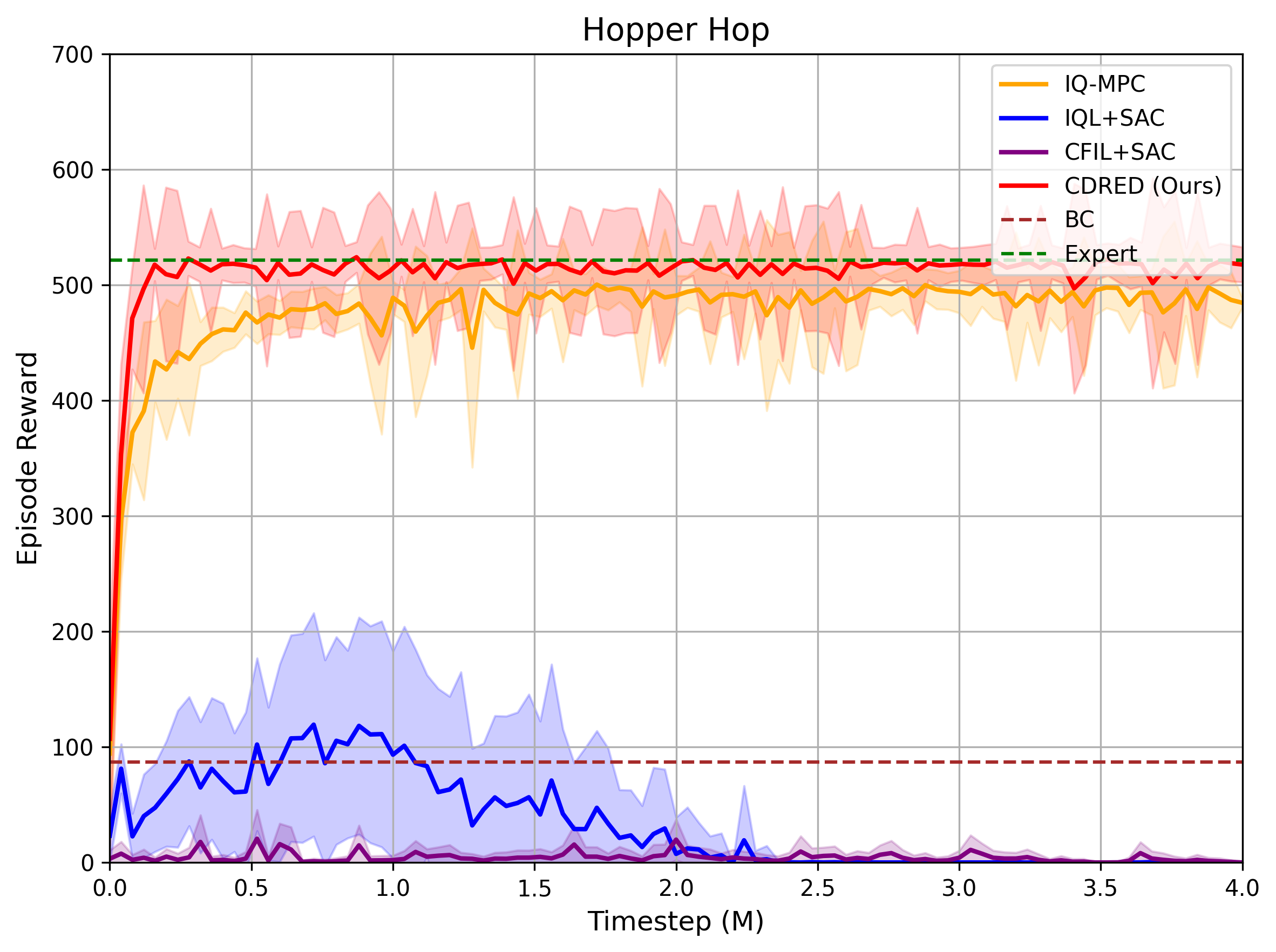}
    \end{minipage}
    \begin{minipage}{0.3\textwidth}
        \centering
        \includegraphics[width=\textwidth]{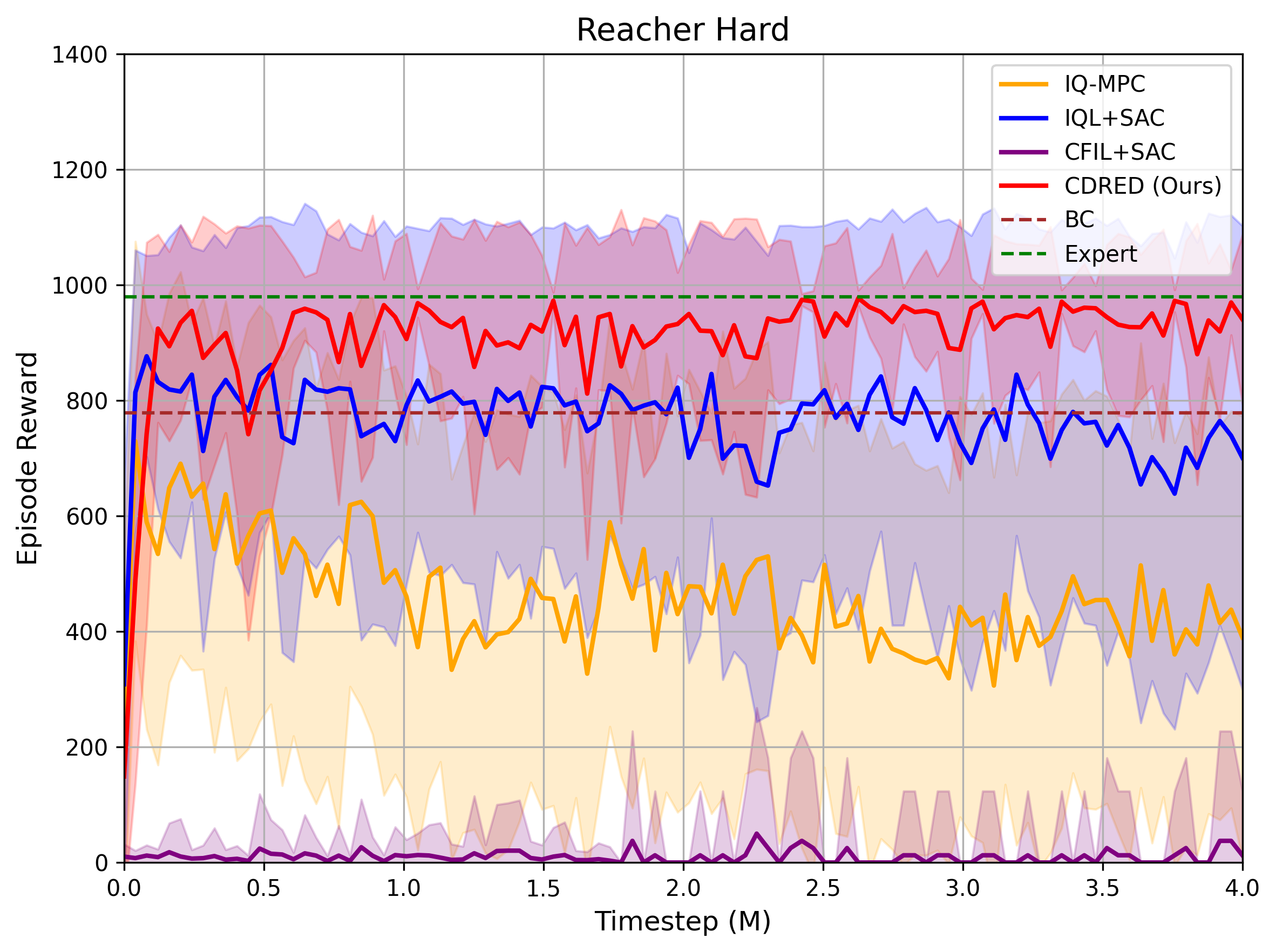}
    \end{minipage}
    \begin{minipage}{0.3\textwidth}
        \centering
        \includegraphics[width=\textwidth]{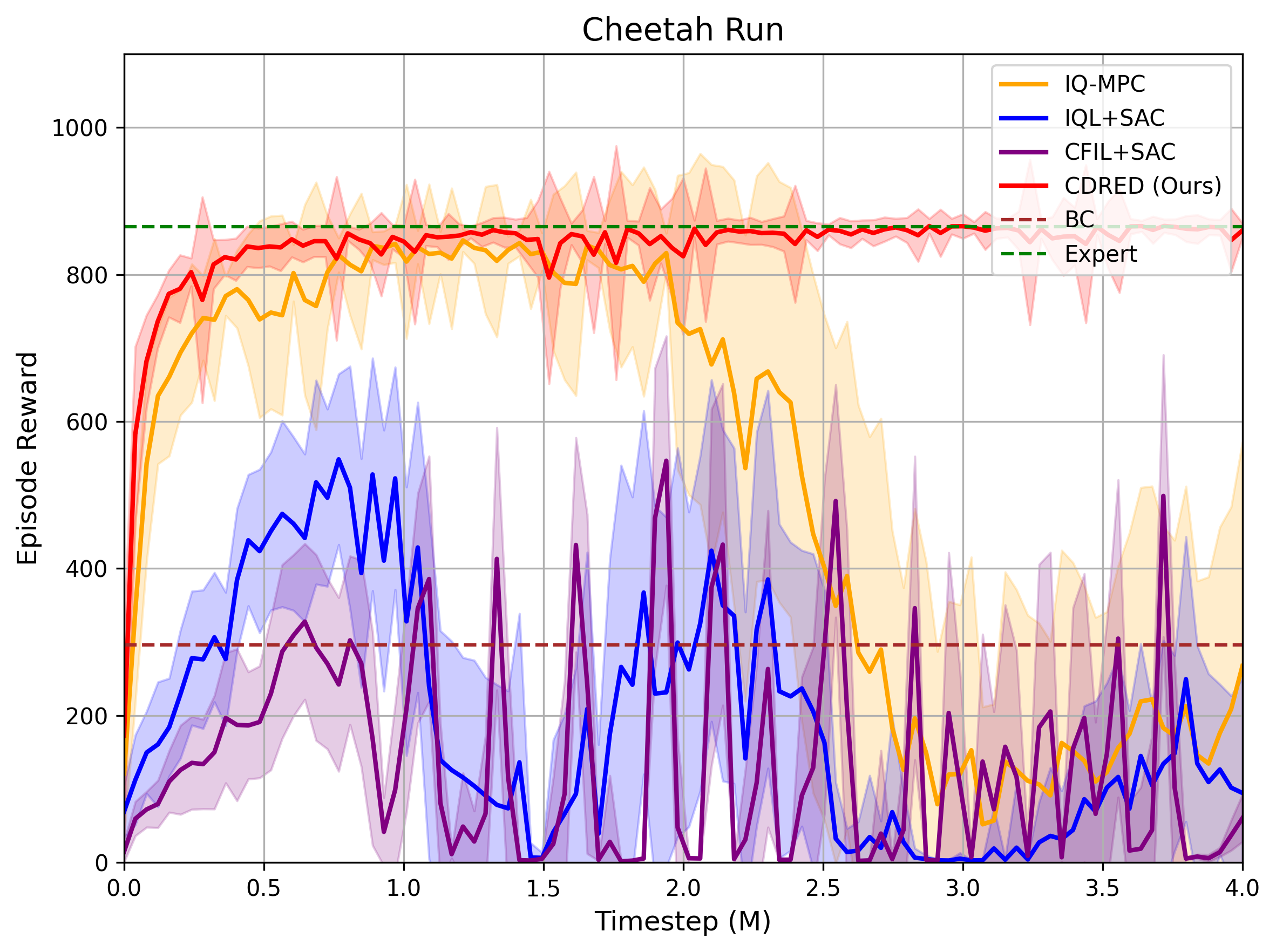}
    \end{minipage}
    
    \vspace{10pt} % Add some space between rows
    
    \begin{minipage}{0.3\textwidth}
        \centering
        \includegraphics[width=\textwidth]{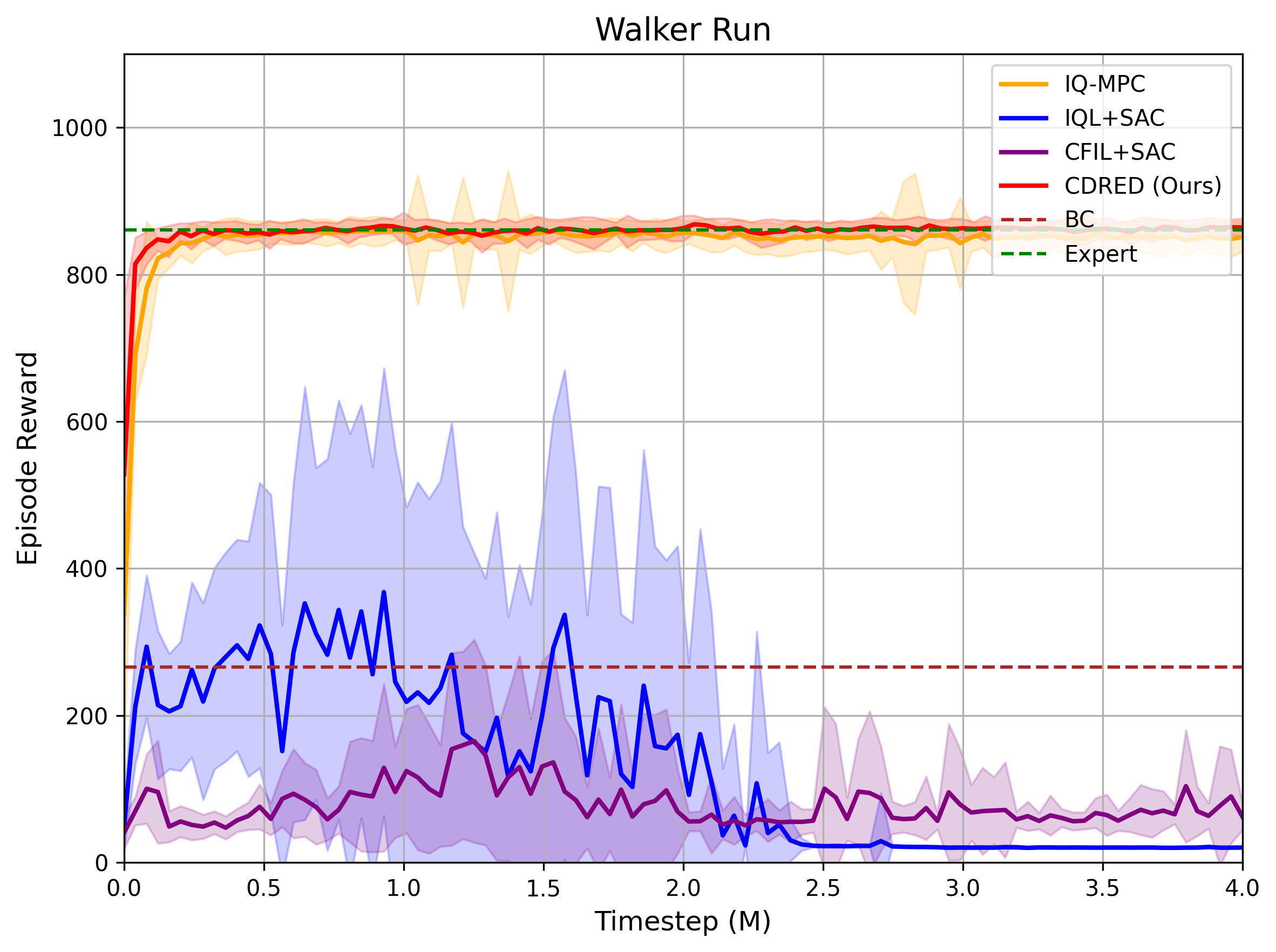}
    \end{minipage}
    \begin{minipage}{0.3\textwidth}
        \centering
        \includegraphics[width=\textwidth]{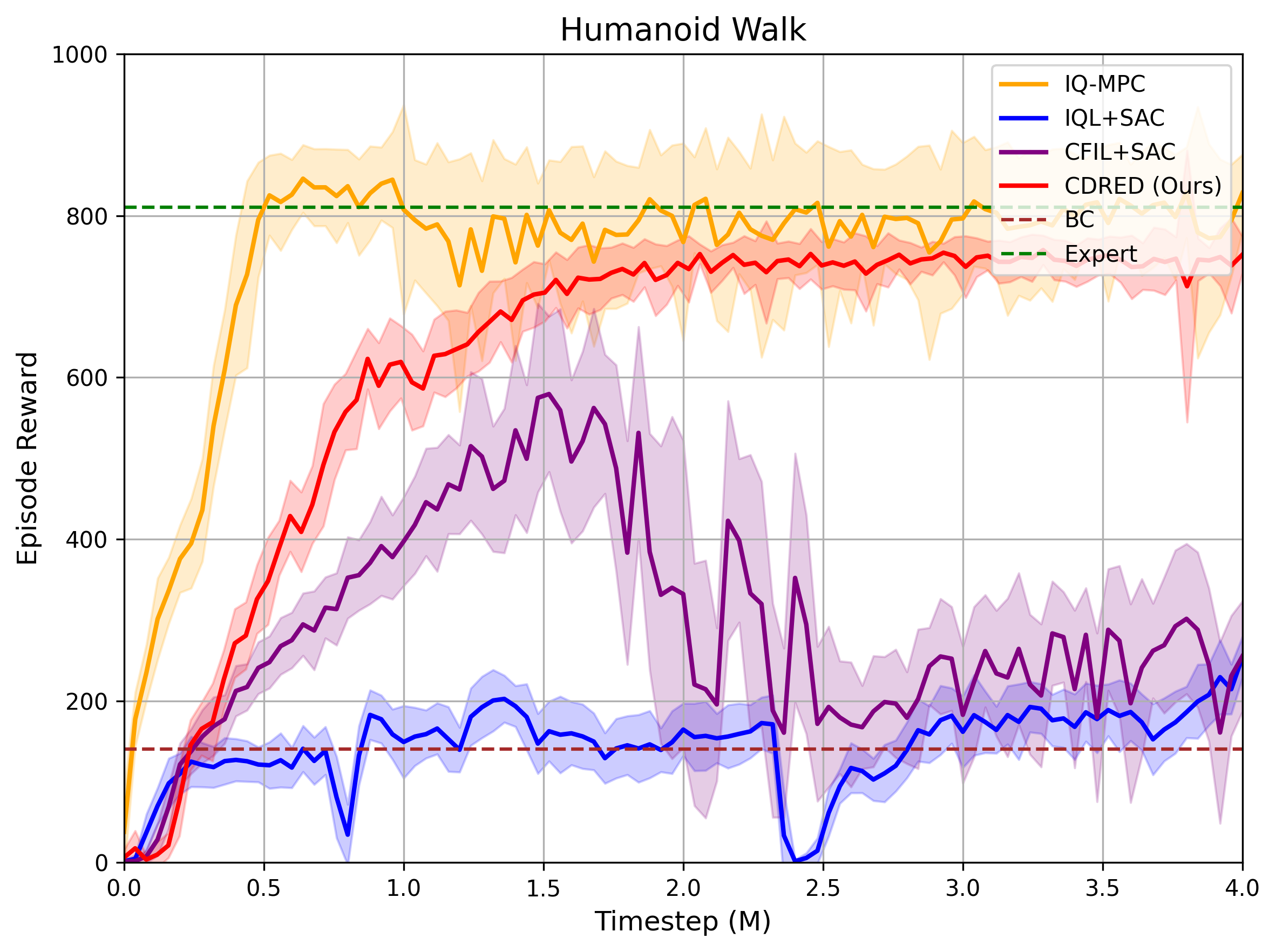}
    \end{minipage}
    \begin{minipage}{0.3\textwidth}
        \centering
        \includegraphics[width=\textwidth]{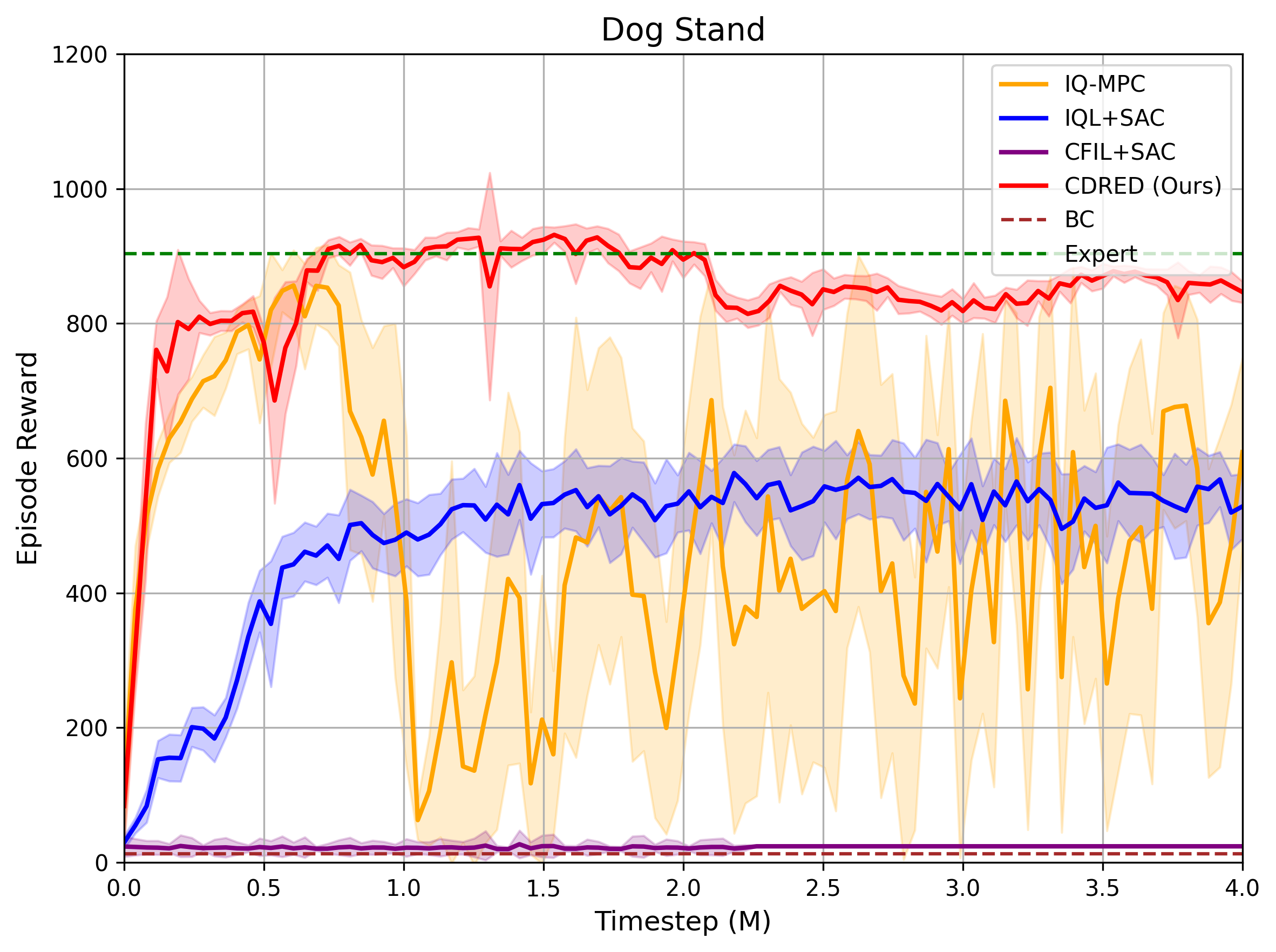}
    \end{minipage}
    \caption{\textbf{DMControl Results} We evaluate our CDRED method ({\color{red}red lines}) on 6 tasks in DMControl environments. Our approach achieves results comparable to IQ-MPC ({\color{orange}orange lines}) in Hopper Hop, Walker Run, and Humanoid Walk, while demonstrating greater stability across the remaining tasks. We conduct the experiments on 3 random seeds.}
    \label{fig:locomotion-results}
\end{figure}
\subsection{Vision-based Experiments}
In addition to experiments using state-based observations, we also benchmark our method on tasks with visual observations. Specifically, we select three tasks from DMControl \citep{tassa2018deepmind} with visual observations. To create these visual datasets, we render visual observations based on state-based expert trajectories, replacing the original state-based observations in the expert data. For each task, we use 100 expert trajectories generated by a trained TD-MPC2 model \citep{hansen2023td}. We show our results in Figure \ref{fig:visual-locomotion-results}. Interestingly, we observe that visual IQ-MPC encounters an issue with an overly powerful discriminator in the Cheetah Run task when using trajectories generated by a trained state-based TD-MPC2 policy, where state observations are replaced by RGB images rendered from those states. However, IQ-MPC performs well when using expert trajectories generated by a TD-MPC2 policy trained directly on visual observations.

\begin{figure}[h]
    \centering
    \begin{minipage}{0.3\textwidth}
        \centering
        \includegraphics[width=\textwidth]{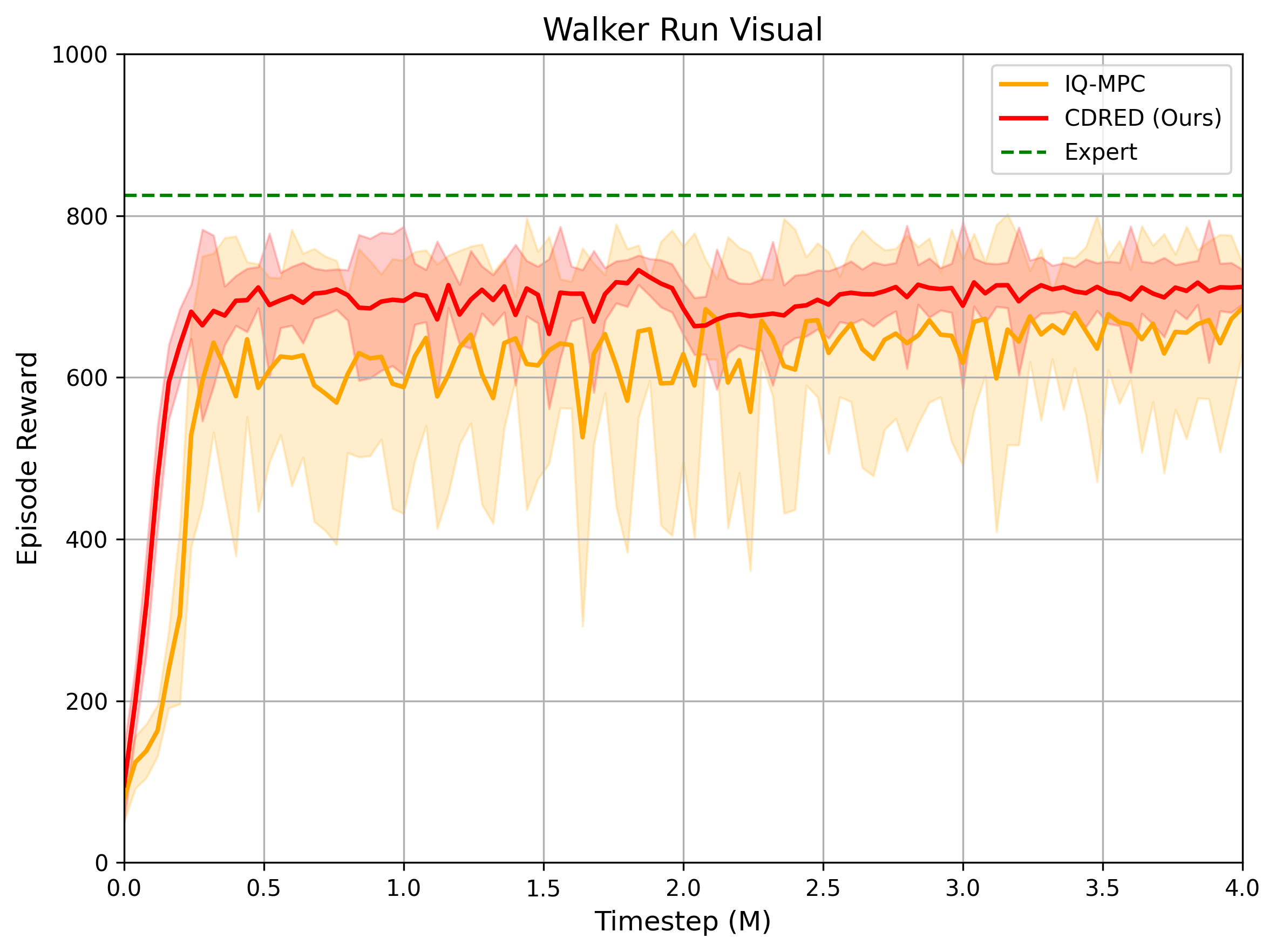}
    \end{minipage}
    \begin{minipage}{0.3\textwidth}
        \centering
        \includegraphics[width=\textwidth]{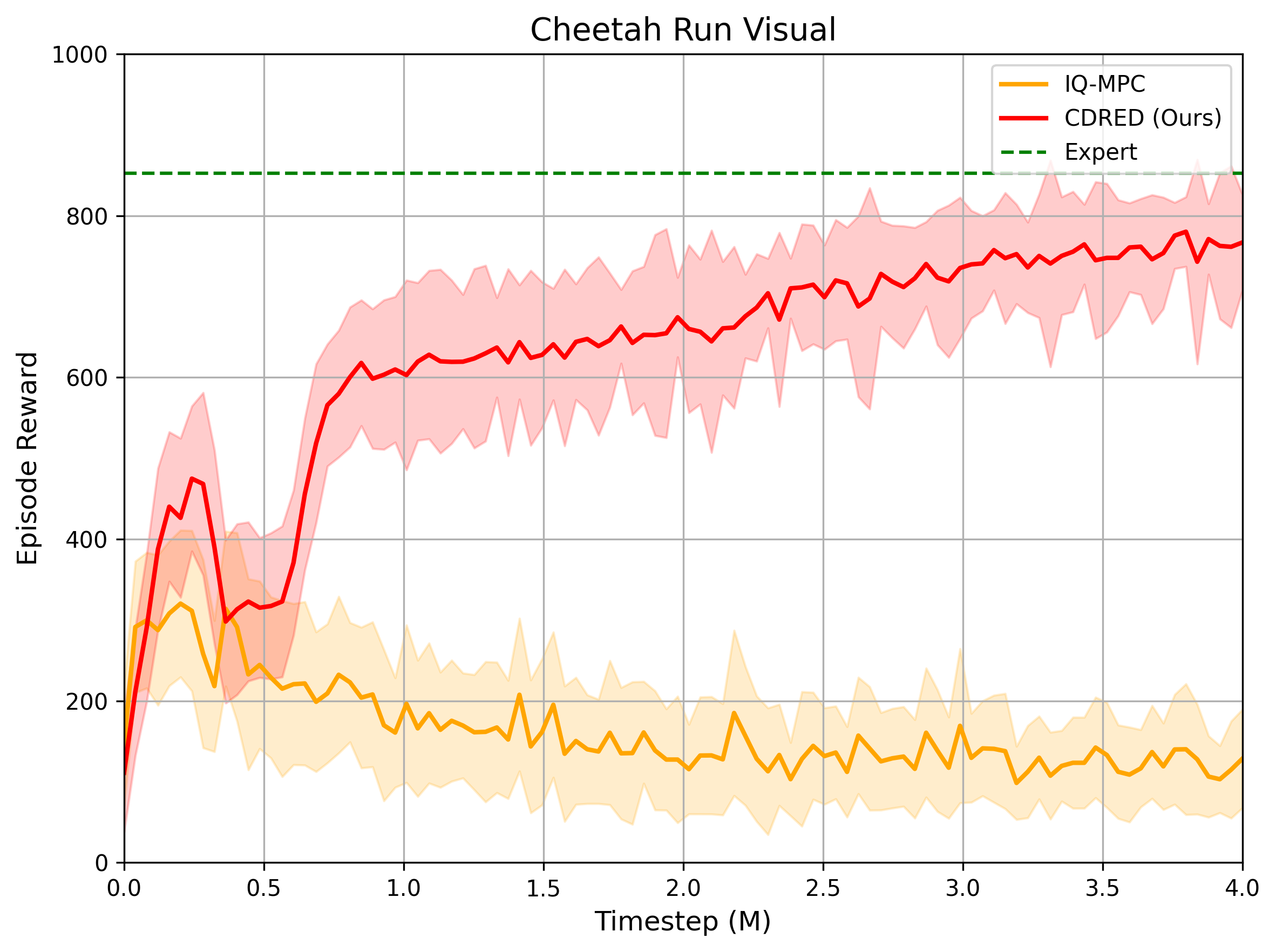}
    \end{minipage}
    \begin{minipage}{0.3\textwidth}
        \centering
        \includegraphics[width=\textwidth]{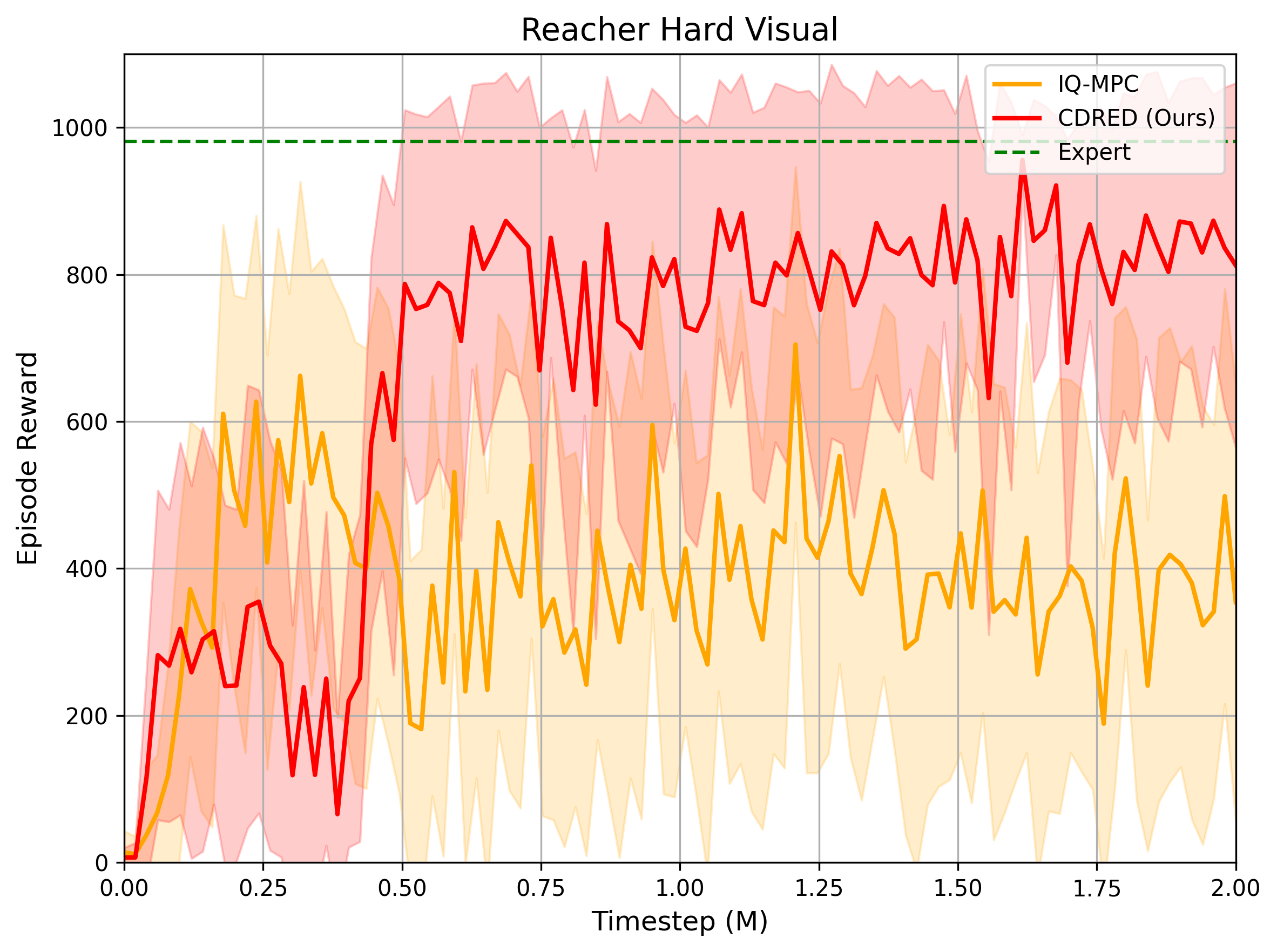}
    \end{minipage}
    \caption{\textbf{Visual Experiment Results} We compare the results of our model with IQ-MPC on tasks with visual observations. Our approach outperforms IQ-MPC in Cheetah Run and Reacher Hard tasks, while obtains comparable performance on Walker Run task. We conduct the experiments on 3 random seeds.}
    \label{fig:visual-locomotion-results}
\end{figure}

\section{Conclusion}

We propose a novel approach for world model-based online imitation learning, featuring an innovative reward model formulation. Unlike traditional adversarial approaches that may introduce instability during training, our reward model is grounded in density estimation for both expert and behavioral state-action distributions. This formulation enhances stability while maintaining high performance. Our model demonstrates expert-level proficiency across various tasks in multiple benchmarks, including DMControl, Meta-World, and ManiSkill2. Furthermore, it consistently retains stable performance throughout long-term online training. With its robust reward modeling and stability, our approach has the potential to tackle complex real-world robotics control tasks, where reliability and adaptability are crucial.

% \section*{References}

% References follow the acknowledgments in the camera-ready paper. Use unnumbered first-level heading for
% the references. Any choice of citation style is acceptable as long as you are
% consistent. It is permissible to reduce the font size to \verb+small+ (9 point)
% when listing the references.
% Note that the Reference section does not count towards the page limit.
% \medskip

% {
% \small

% [1] Alexander, J.A.\ \& Mozer, M.C.\ (1995) Template-based algorithms for
% connectionist rule extraction. In G.\ Tesauro, D.S.\ Touretzky and T.K.\ Leen
% (eds.), {\it Advances in Neural Information Processing Systems 7},
% pp.\ 609--616. Cambridge, MA: MIT Press.

% [2] Bower, J.M.\ \& Beeman, D.\ (1995) {\it The Book of GENESIS: Exploring
%   Realistic Neural Models with the GEneral NEural SImulation System.}  New York:
% TELOS/Springer--Verlag.

% [3] Hasselmo, M.E., Schnell, E.\ \& Barkai, E.\ (1995) Dynamics of learning and
% recall at excitatory recurrent synapses and cholinergic modulation in rat
% hippocampal region CA3. {\it Journal of Neuroscience} {\bf 15}(7):5249-5262.
% }
\bibliographystyle{plainnat}
\bibliography{main}

%%%%%%%%%%%%%%%%%%%%%%%%%%%%%%%%%%%%%%%%%%%%%%%%%%%%%%%%%%%%

\appendix

\section{Hyperparameters and Architectural Details}
\subsection{Architectural Details}
We show the overall model architecture via a Pytorch style notation. We leverage layernorm \citep{ba2016layer} and Mish activations \citep{misra2019mish} for our model. The detialed architecture is displayed as following:
\begin{lstlisting}
WorldModel(
    (_encoder): ModuleDict(
        (state): Sequential(
            (0): NormedLinear(in_features=state_dim, out_features=256, bias=True, act=Mish)
            (1): NormedLinear(in_features=256, out_features=512, bias=True, act=SimNorm)
        )
    )
    (_dynamics): Sequential(
        (0): NormedLinear(in_features=512+action_dim, out_features=512, bias=True, act=Mish)
        (1): NormedLinear(in_features=512, out_features=512, bias=True, act=Mish)
        (2): NormedLinear(in_features=512, out_features=512, bias=True, act=SimNorm)
    )
    (_reward): CDRED_Reward(
        (behavioral_predictor): Sequential(
            (0): NormedLinear(in_features=512+action_dim, out_features=512, bias=True, act=Mish)
            (1): NormedLinear(in_features=512, out_features=512, bias=True, act=Mish)
            (2): Linear(in_features=512, out_features=64, bias=True)
        )
        (expert_predictor): Sequential(
            (0): NormedLinear(in_features=512+action_dim, out_features=512, bias=True, act=Mish)
            (1): NormedLinear(in_features=512, out_features=512, bias=True, act=Mish)
            (2): Linear(in_features=512, out_features=64, bias=True)
        )
        (target_networks)[not learnable]: Vectorized ModuleList(
            (0-4): 5 x Sequential(
                (0): NormedLinear(in_features=512+action_dim, out_features=512, bias=True, act=Mish)
                (1): NormedLinear(in_features=512, out_features=512, bias=True, act=Mish)
                (2): Linear(in_features=512, out_features=64, bias=True)
            )
        )
    )
    (_pi): Sequential(
        (0): NormedLinear(in_features=512, out_features=512, bias=True, act=Mish)
        (1): NormedLinear(in_features=512, out_features=512, bias=True, act=Mish)
        (2): Linear(in_features=512, out_features=2*action_dim, bias=True)
    )
    (_Qs): Vectorized ModuleList(
        (0-4): 5 x Sequential(
            (0): NormedLinear(in_features=512+action_dim, out_features=512, bias=True, dropout=0.01, act=Mish)
            (1): NormedLinear(in_features=512, out_features=512, bias=True, act=Mish)
            (2): Linear(in_features=512, out_features=101, bias=True)
        )
    )
    (_target_Qs): Vectorized ModuleList(
        (0-4): 5 x Sequential(
            (0): NormedLinear(in_features=512+action_dim, out_features=512, bias=True, dropout=0.01, act=Mish)
            (1): NormedLinear(in_features=512, out_features=512, bias=True, act=Mish)
            (2): Linear(in_features=512, out_features=num_bins, bias=True)
        )
    )
)
\end{lstlisting}

\subsection{Hyperparameter Details}

The specific hyperparameters used in the CDRED reward model are as follows:
\begin{itemize}
    \item The predictors and target networks project latent state-action pairs to an embedding space with dimension $p = 64$.
    \item We use an ensemble of $5$ target networks for the CDRED reward model.
    \item The function $g(x) = x$ is used in all experiments.
    \item The value of $\zeta = 0.8$ is used across all experiments.
    \item We adopt $\alpha = 0.9$ for all experiments.
    \item A StepLR learning rate scheduler is employed with $\gamma_{\text{lr}} = 0.1$, with a scheduler step of $500K$ for Meta-World and ManiSkill2 experiments, and $2M$ for DMControl experiments.
\end{itemize}
The remaining hyperparameters are consistent with those used in TD-MPC2 \citep{hansen2023td}.

\newpage
\section{Training and Planning Algorithms}
\subsection{Training Algorithm}

In this section, we present the detailed training algorithm for the CDRED world model, as shown in Algorithm \ref{alg:training}. For clarity, let $\theta = \{\phi, \psi, \xi\}$ represent all learnable parameters of the world model, and $\theta^-$ denote a fixed copy of $\theta$.
\begin{algorithm}[H]
\caption{~~CDRED World Model {(\emph{training})}}
\label{alg:training}
    \begin{algorithmic}
    \REQUIRE $\theta, \theta^{-}$: randomly initialized network parameters\\
    ~~~~~~~~~~$\eta, \tau, \lambda, \mathcal{B}_\pi, \mathcal{B}_E$: learning rate, soft update coefficient, horizon discount coefficient, behavioral buffer, expert buffer
    \FOR{training steps}
    \STATE \emph{// Collect episode with CDRED world model from $\mathbf{s}_{0} \sim p_{0}$:}
    \FOR{step $t=0...T$}
    \STATE Compute $\mathbf{a}_{t}$ with $\pi_{\theta}(\cdot | h_{\theta}(\mathbf{s}_{t}))$ using Algorithm \ref{alg:inference}\hfill{$\vartriangleleft$ \emph{Planning with MPPI}}
    \STATE $(\mathbf{s}'_{t},r_{t}) \sim \text{env.step}(\mathbf{a}_t)$
    \STATE $\mathcal{B}_\pi \leftarrow \mathcal{B}_\pi \cup (\mathbf{s}_{t}, \mathbf{a}_{t}, r_{t}, \mathbf{s}'_{t})$\hfill{$\vartriangleleft$ \emph{Add to behavioral buffer}}
    \STATE $\mathbf{s}_{t+1}\leftarrow\mathbf{s}'_t$
    \ENDFOR
    \STATE { \emph{// Update reward-free world model using collected data in $\mathcal{B}_\pi$ and $\mathcal{B}_E$:}}
    \FOR{num updates per step}
    \STATE $(\mathbf{s}_{t}, \mathbf{a}_{t}, \mathbf{s}'_t)_{0:H} \sim \mathcal{B}_\pi\cup\mathcal{B}_E$\hfill{$\vartriangleleft$ \emph{Combine behavioral and expert batch}}
    \STATE $\mathbf{z}_{0} = h_{\theta}(\mathbf{s}_{0})$\hfill{$\vartriangleleft$ \emph{Encode first observation}}
    \STATE { \emph{// Unroll for horizon $H$}}
    \FOR{$t = 0...H$}
    \STATE $\mathbf{z}_{t+1} = d_{\theta}(\mathbf{z}_{t}, \mathbf{a}_{t})$\hfill{$\vartriangleleft$ \emph{Unrolling using the latent dynamics model}}
    \STATE $\hat q_t = Q(\mathbf{z}_t,\mathbf{a}_t)$\hfill{$\vartriangleleft$ \emph{Estimate the Q value}}
    \STATE $\mathbf{z}'_t=h(\mathbf{s}'_t)$\hfill{$\vartriangleleft$ \emph{Encode the ground-truth next state}}
    \STATE $\hat r_t = R(\mathbf{z}_t,\mathbf{a}_t)$\hfill{$\vartriangleleft$ \emph{Estimate Reward using the CDRED reward model}}
    \STATE $q_t = \hat r_t + \gamma Q(\mathbf{z}'_t,\pi(\mathbf{z}'_t))$\hfill{$\vartriangleleft$ \emph{Compute the TD target using the estimated reward}}
   
    \ENDFOR
     \STATE Compute model loss $\mathcal{L}$\hfill{$\vartriangleleft$ \emph{Equation \ref{eqn:model-loss}}}
     \STATE Compute policy prior loss $\mathcal{L}^\pi$\hfill{$\vartriangleleft$ \emph{Equation \ref{eqn:policy-learning}}}
    \STATE $\theta \leftarrow \theta - \frac{1}{H} \eta \nabla_{\theta} (\mathcal{L} + \mathcal{L}^\pi)$\hfill{$\vartriangleleft$ \emph{Update online network}}
    \STATE $\theta^{-} \leftarrow (1 - \tau) \theta^{-} + \tau \theta$\hfill{$\vartriangleleft$ \emph{Soft update}}
    \ENDFOR
    \ENDFOR
    \end{algorithmic}
\end{algorithm}

\newpage
\subsection{Planning Algorithm}

In this section, we present the detailed MPPI planning algorithm for the CDRED world model, as shown in Algorithm \ref{alg:inference}. For simplicity, let $\theta = \{\phi, \psi, \xi\}$ represent all learnable parameters of the world model.

\begin{algorithm}[H]
\caption{~~CDRED World Model {(\emph{inference})}}
\label{alg:inference}
\begin{algorithmic}[1]
\REQUIRE $\theta:$ learned network parameters\\
~~~~~~~~~~$\mu^{0}, \sigma^{0}$: initial parameters for $\mathcal{N}$\\
~~~~~~~~~~$N, N_{\pi}$: number of sample/policy trajectories\\
~~~~~~~~~~$\mathbf{s}_{t}, H$: current state, rollout horizon
\STATE Encode state $\mathbf{z}_{t} \leftarrow h_{\theta}(\mathbf{s}_{t})$
\FOR{each iteration $j=1..J$}
\STATE Sample $N$ trajectories of length $H$ from $\mathcal{N}(\mu^{j-1}, (\sigma^{j-1})^{2} \mathrm{I})$
\STATE {Sample $N_{\pi}$ trajectories of length $H$ using $\pi_{\theta}, d_{\theta}$} \\
{\emph{// Estimate trajectory returns $\phi_{\Gamma}$ using $d_{\theta},Q_{\theta},\pi_\theta,R_\theta$ starting from $\mathbf{z}_{t}$ and initialize $\phi_{\Gamma}=0$:}}
\FOR{all $N+N_{\pi}$ trajectories $(\mathbf{a}_{t}, \mathbf{a}_{t+1}, \dots, \mathbf{a}_{t+H})$}
\FOR{step $t=0..H-1$}
\STATE $\mathbf{z}_{t+1} \leftarrow d_{\theta}(\mathbf{z}_{t}, \mathbf{a}_{t})$ \hfill {$\vartriangleleft$ \emph{Latent transition}}
\STATE $\mathbf{\hat a}_{t+1}\sim\pi_\theta(\cdot|\mathbf{z}_{t+1})$
\STATE $\phi_{\Gamma} = \phi_{\Gamma} + \gamma^{t} R_\theta(\mathbf{z}_{t}, \mathbf{a}_{t})$  \hfill {$\vartriangleleft$ \emph{Estimate reward with CDRED reward model}}
\ENDFOR
\STATE $\phi_{\Gamma} = \phi_{\Gamma} + \gamma^{H} Q_\theta(\mathbf{z}_H, \mathbf{a}_H)$ \hfill {$\vartriangleleft$ \emph{Terminal Q value}}
\ENDFOR
\STATE {\emph{// Update parameters $\mu,\sigma$ for next iteration:}}
\STATE $\mu^{j}, \sigma^{j} \leftarrow$ MPPI update with $\phi_\Gamma$.
\ENDFOR
\STATE \textbf{return} $\mathbf{a} \sim \mathcal{N}(\mu^{J}, (\sigma^{J})^{2} \mathrm{I})$
\end{algorithmic}
\end{algorithm}

\section{Task Details and Environment Specifications}
\label{sec:env-spec}
We consider 12 continuous control tasks in locomotion control and robot manipulation. We leverage 6 manipulation tasks in Meta-World \citep{yu2020meta}, 6 locomotion tasks in DMControl \citep{tassa2018deepmind} and 3 tasks in ManiSKill2 \citep{gu2023maniskill2}. In this section, we list the environment specifications for completeness in Table \ref{tab:meta-world-spec},  Table \ref{tab:dmcontrol-spec} and Table \ref{tab:mani-skill-spec}.
\begin{table}[H]
    \centering
    \begin{tabular}{c|cc}
    \toprule
        \textbf{Task} & \textbf{Observation Dimension} & \textbf{Action Dimension} \\
    \midrule
        Box Close & 39 & 4 \\
        Bin Picking & 39 & 4 \\
        Reach Wall & 39 & 4 \\
        Stick Pull & 39 & 4 \\
        Stick Push & 39 & 4 \\
        Soccer & 39 & 4 \\
    \bottomrule
    \end{tabular}
    \vspace{+5pt}
    \caption{\textbf{Meta-World Tasks} We evaluate on 6 tasks in Meta-World. The Meta-World benchmark is specifically constructed to facilitate research in multitask and meta-learning, ensuring a consistent embodiment, observation space, and action space across all tasks.}
    \label{tab:meta-world-spec}
\end{table}

\begin{table}[H]
    \centering
    \begin{tabular}{c|ccc}
    \toprule
        \textbf{Task} & \textbf{Observation Dimension} & \textbf{Action Dimension} & \textbf{High-dimensional?}\\
    \midrule
        Reacher Hard & 6 & 2 & No \\
        Hopper Hop & 15 & 4 & No\\
        Cheetah Run & 17 & 6 & No \\
        Walker Run & 24 & 6 & No \\
    \midrule
        Humanoid Walk & 67 & 24 & Yes \\
        Dog Stand & 223 & 38 & Yes \\
    \bottomrule
    \end{tabular}
    \vspace{+5pt}
    \caption{\textbf{DMControl Tasks} We evaluate on 6 tasks in DMControl. DMControl is a benchmark for reinforcement learning, offering a range of continuous control tasks built on the MuJoCo physics engine. It provides diverse environments for testing algorithms on tasks from basic motions to complex behaviors, supporting standardized evaluation in control and planning research.}
    \label{tab:dmcontrol-spec}
\end{table}

\begin{table}[H]
    \centering
    \begin{tabular}{c|cc}
    \toprule
        \textbf{Task} & \textbf{Observation Dimension} & \textbf{Action Dimension} \\
    \midrule
        Lift Cube & 42 & 4 \\
        Pick Cube & 51 & 4 \\
        Turn Faucet & 40 & 7 \\
    \bottomrule
    \end{tabular}
    \vspace{+5pt}
    \caption{\textbf{ManiSkill2 Tasks} We evaluate on 3 tasks in ManiSkill2. The ManiSkill2 benchmark represents a sophisticated platform designed to advance large-scale robot learning capabilities. It distinguishes itself through comprehensive task randomization and an extensive array of task variations, enabling more robust and generalized robotic skill development.}
    \label{tab:mani-skill-spec}
\end{table}
\section{Additional Experiments}
\subsection{Experiments on ManiSkill2}
We further evaluate our method on additional manipulation tasks in ManiSkill2 \citep{gu2023maniskill2}, achieving stable and competitive results on the pick cube, lift cube, and turn faucet tasks. Notably, IQL+SAC \citep{garg2021iq} and IQ-MPC \citep{li2024rewardfreeworldmodelsonline} also perform relatively well in these scenarios. Table \ref{tab:manipulation-success-maniskill} summarizes the success rates of each method across the ManiSkill2 tasks.

\begin{table}[H]
    \centering
    \begin{tabular}{c|cccc}
        \toprule
        \textbf{Method} & \textbf{IQL+SAC} & \textbf{CFIL+SAC} & \textbf{IQ-MPC} & \textbf{CDRED(Ours)} \\
        \midrule
        Pick Cube & 0.61$\pm$0.13 & 0.00$\pm$0.00 & 0.79$\pm$0.05 & \textbf{0.87$\pm$0.04}\\
        Lift Cube & 0.85 $\pm$ 0.04 & 0.01$\pm$0.01 & 0.89$\pm$0.02 & \textbf{0.93$\pm$0.03}\\
        Turn Faucet & 0.82$\pm$0.04 & 0.00$\pm$0.00 & 0.73$\pm$0.08 & \textbf{0.84$\pm$0.08}\\
        \bottomrule
    \end{tabular}
    \vspace{+5pt}
    \caption{\textbf{Manipulation Success Rate Results in ManiSkill2} We evaluate the success rate of CDRED across three tasks in the ManiSkill2 environment. CDRED demonstrates superior performance compared to IQL+SAC, CFIL+SAC, and IQ-MPC on the Pick Cube and Lift Cube tasks, while achieving comparable results on Turn Faucet. The reported results are averaged over 100 trajectories and evaluated across three random seeds.}
    \label{tab:manipulation-success-maniskill}
\end{table}

\subsection{Ablation Studies}
\label{sec:ablation}
To evaluate the influence of different architecture choices and expert data amounts, we ablate over the expert trajectories number, the $g$ function choice, and the usage of coupling. We show that our approach is still robust under a small number of expert demonstrations.
\paragraph{Ablation on Expert Trajectories Number}
We evaluate the impact of the number of expert trajectories on model performance and find that our model can learn effectively with a limited number of expert trajectories. We conduct this ablation on the Bin Picking task in Meta-World and the Cheetah Run task in DMControl, observing that our model achieves expert-level performance with only five demonstrations. The results are presented in Figure \ref{fig:ablation-traj}. Our model can effectively learn with only 5 expert demonstrations for Cheetah Run and Bin Picking tasks.

\begin{figure}[h]
    \centering
    \begin{minipage}{0.45\textwidth}
        \centering
        \includegraphics[width=\textwidth]{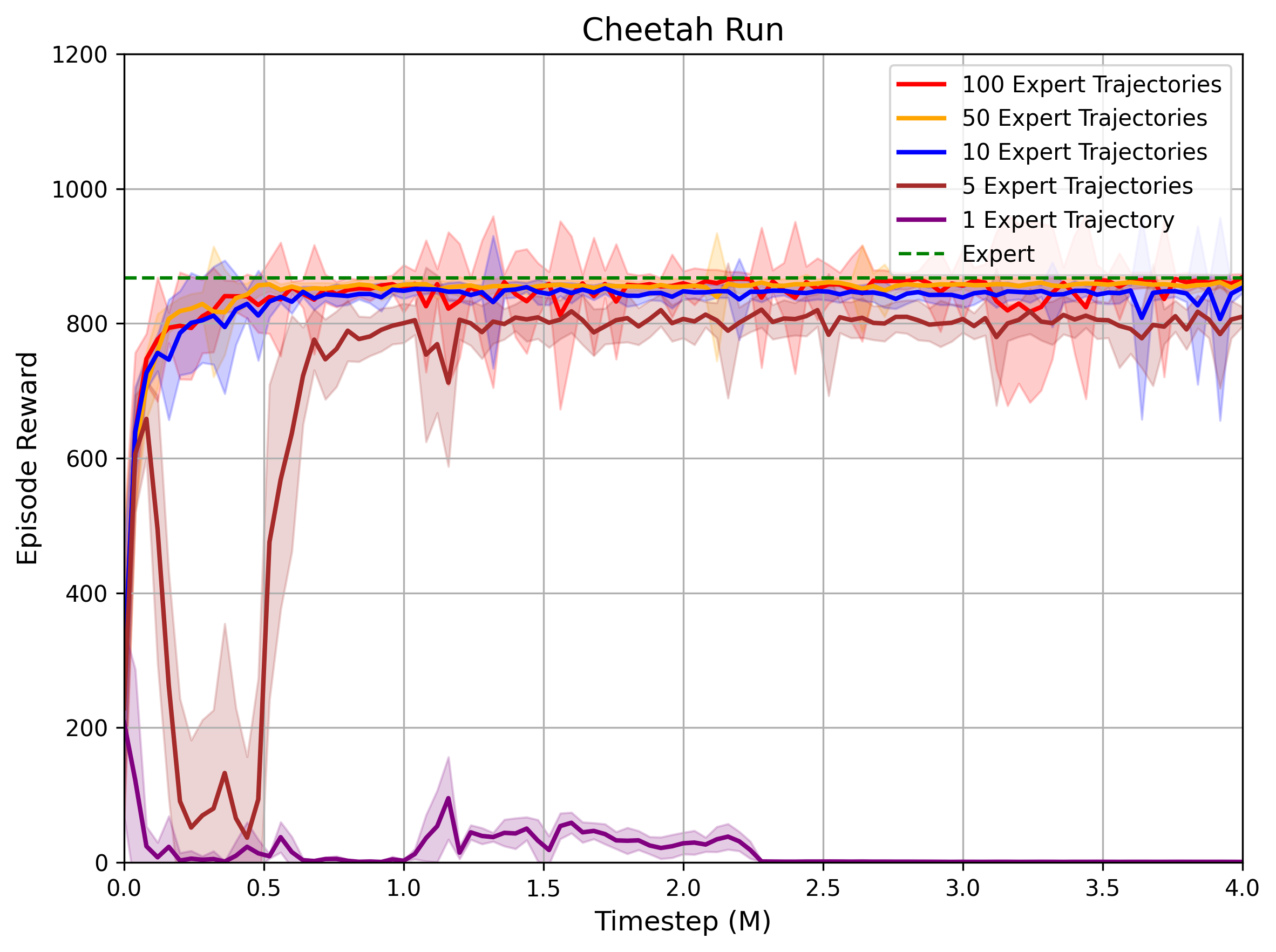}
    \end{minipage}
    \begin{minipage}{0.45\textwidth}
        \centering
        \includegraphics[width=\textwidth]{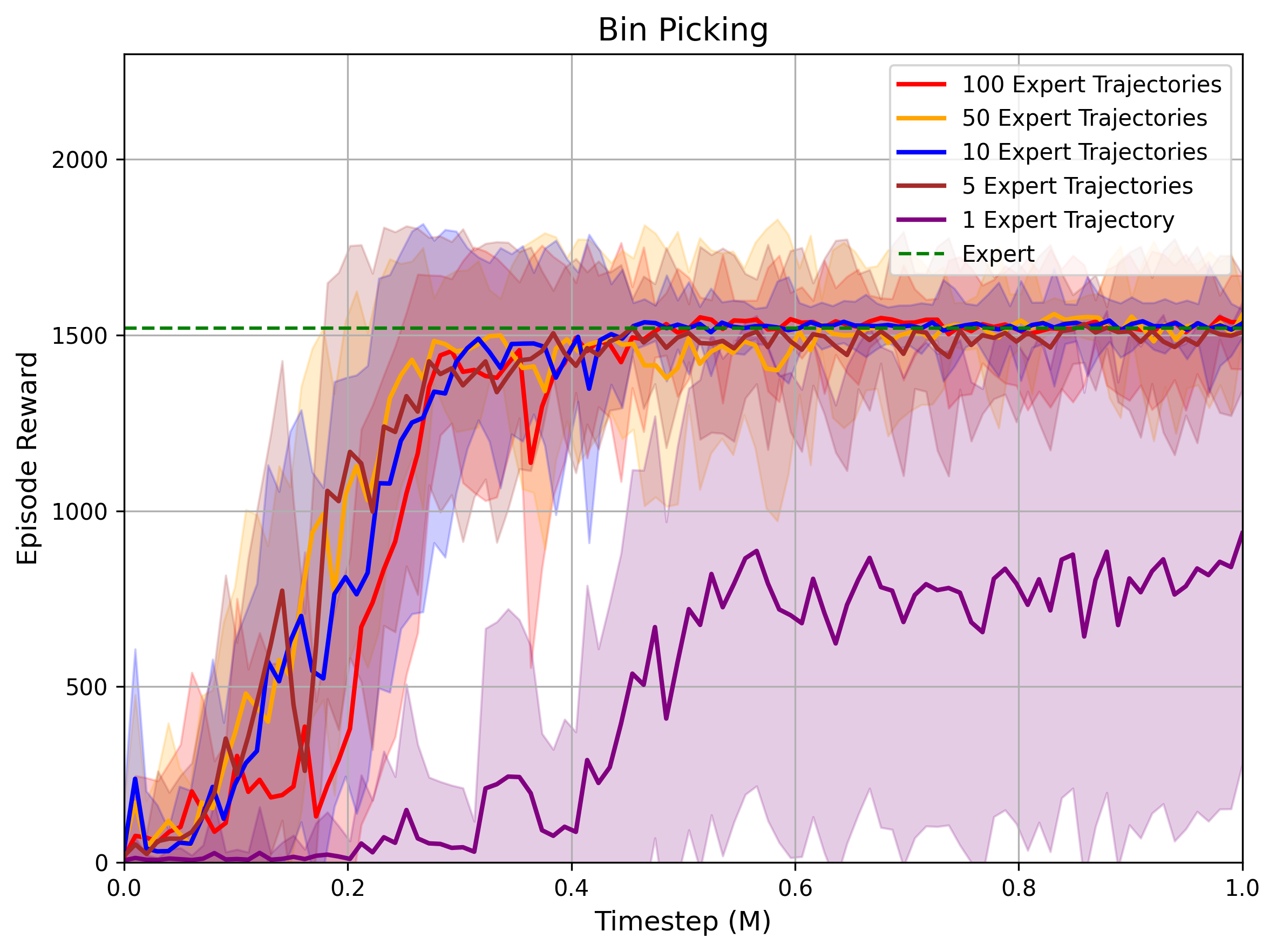}
    \end{minipage}
    \caption{\textbf{Ablation Study on Expert Trajectories Number} We conduct an ablation study on the number of expert trajectories for the Cheetah Run task in DMControl and the Bin Picking task in Meta-World. Our results demonstrate that our model can achieve expert-level performance using only 5 expert demonstrations for both tasks.}
    \label{fig:ablation-traj}
\end{figure}

\paragraph{Ablation on the $g$ Function Choice}
Function $g$ maps the neural network output bonus to the actual reward space. In order to keep the optimal point for the reward function unchanged, we need to leverage a monotonically increasing function. Empirically, we find $g(x)=x$ and $g(x)=\exp(x)$ can both work, but they have different performances in high-dimensional settings. We find $g(x)=x$ tends to provide a faster convergence in high-dimensional tasks such as Dog Stand compared to $g(x)=\exp(x)$. While we haven't observed any significant difference on low-dimensional tasks. We show the ablation in Figure \ref{fig:function-choice}.

\begin{figure}[h]
    \centering
    \begin{minipage}{0.45\textwidth}
        \centering
        \includegraphics[width=\textwidth]{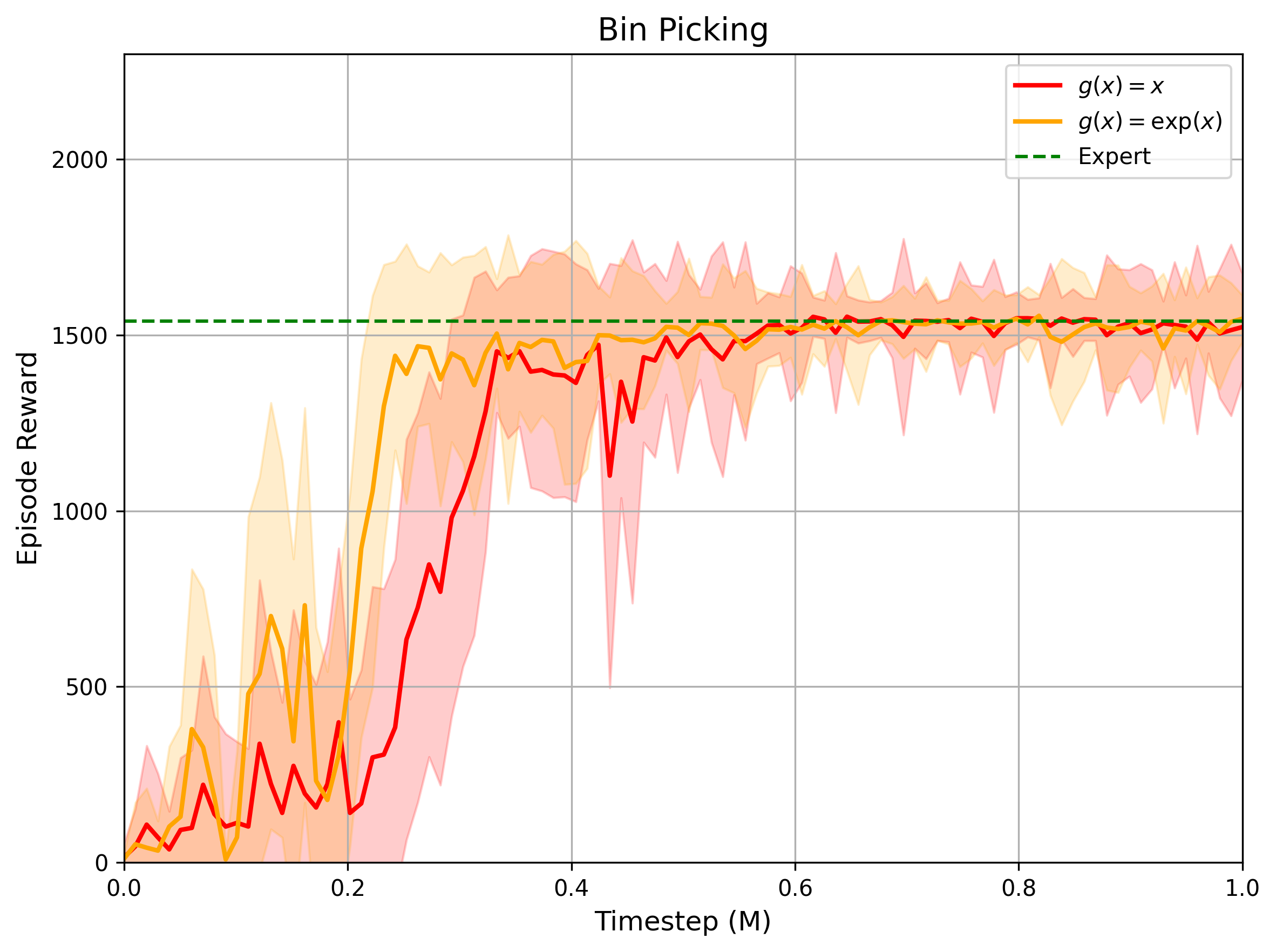}
        \\ \textbf{Low-dimensional task}
    \end{minipage}
    \begin{minipage}{0.45\textwidth}
        \centering
        \includegraphics[width=\textwidth]{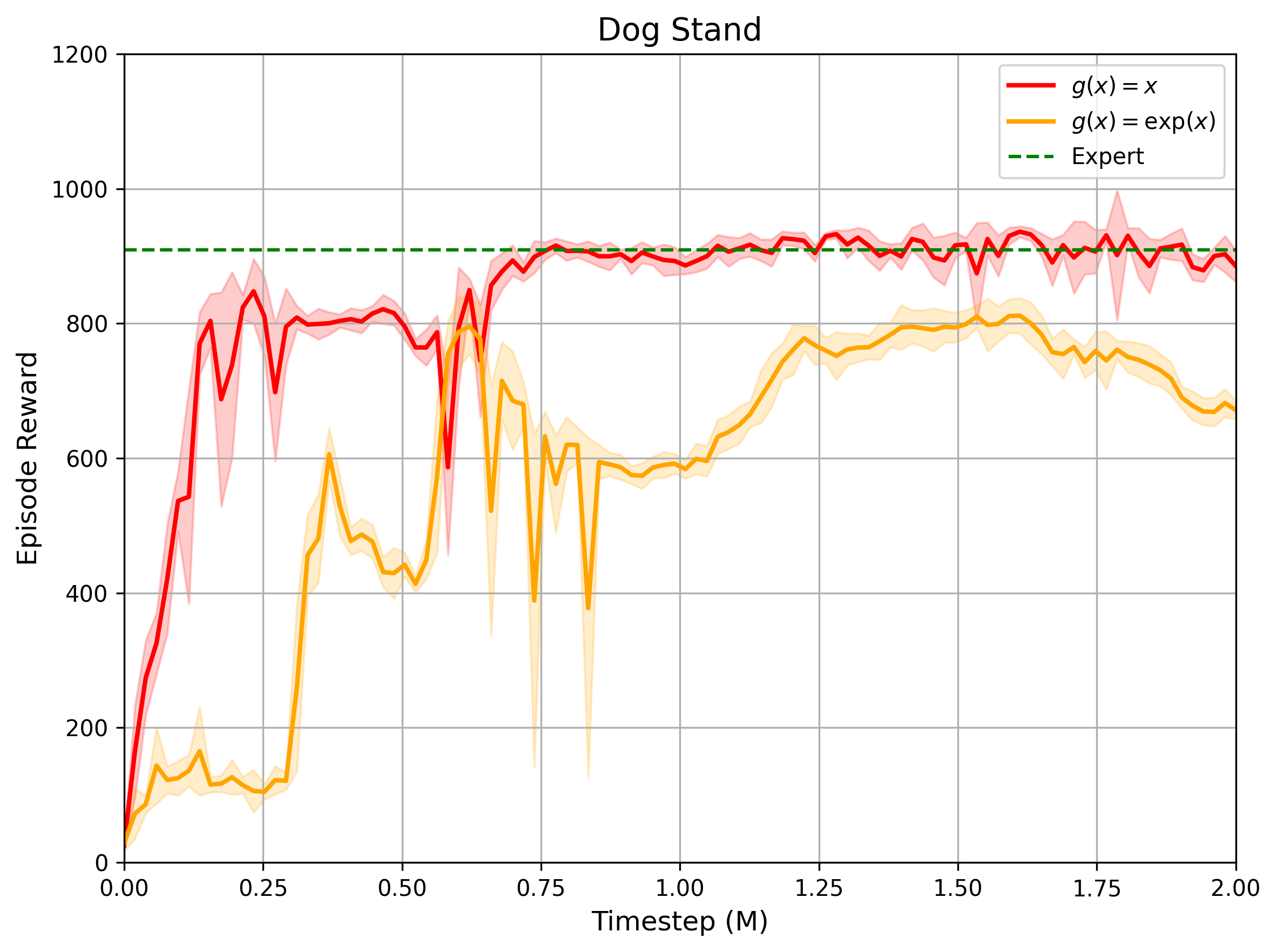}
        \\ \textbf{High-dimensional task}
    \end{minipage}
    \caption{\textbf{Ablation on $g$ function choice} For low-dimensional task (left), both forms of $g(x)$ demonstrate comparable performance. However, in high-dimensional task (right), $g(x) = \exp(x)$ exhibits instability and suboptimal behavior, whereas $g(x) = x$ maintains stability. The task dimensionality information is shown in Appendix \ref{sec:env-spec}.}
    \label{fig:function-choice}
\end{figure}

\begin{figure}[h]
    \centering
    \begin{minipage}{0.45\textwidth}
        \centering
        \includegraphics[width=\textwidth]{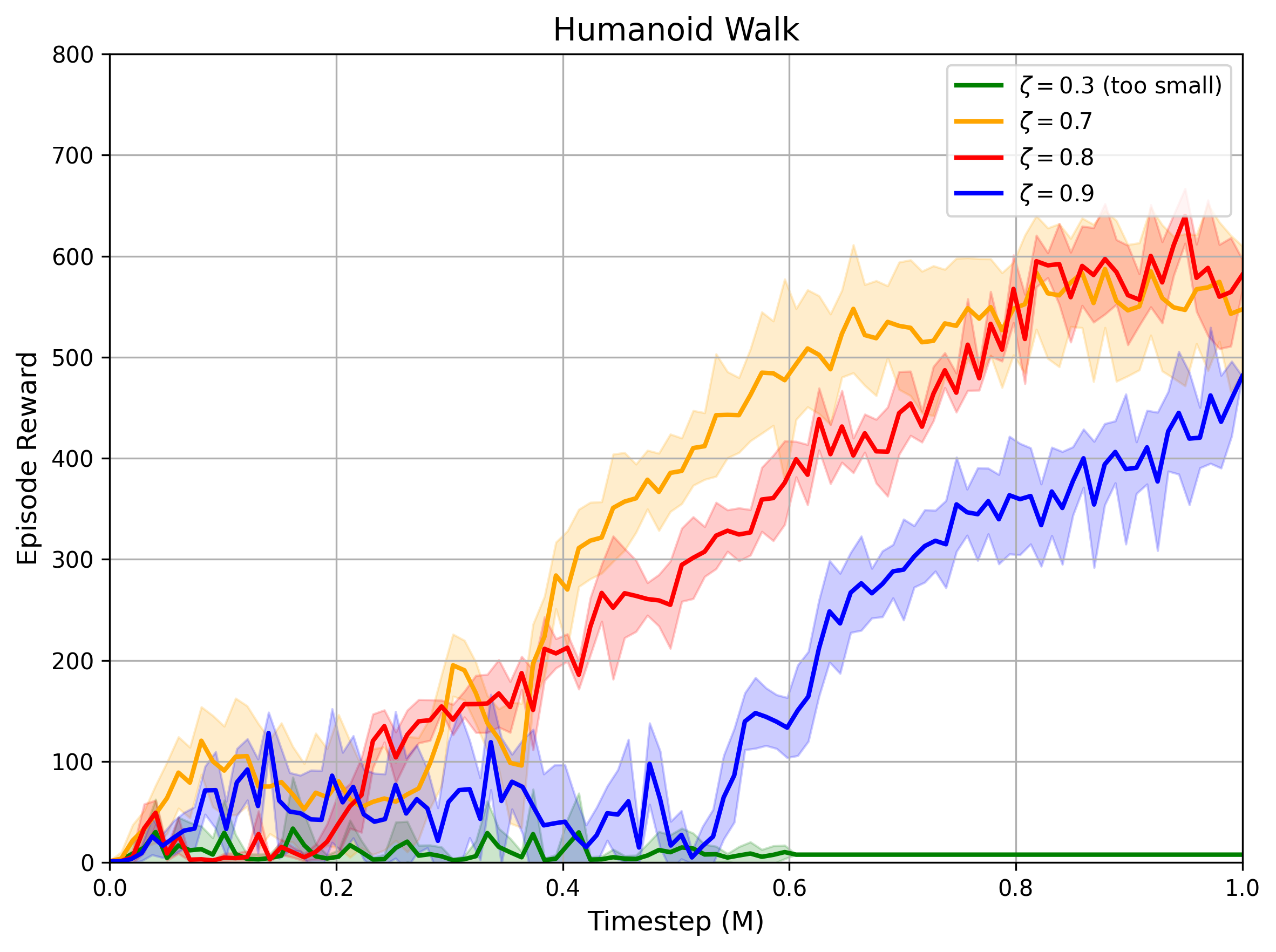}
        {\textbf{Ablation on $\zeta$}}
    \end{minipage}
    \begin{minipage}{0.45\textwidth}
        \centering
        \includegraphics[width=\textwidth]{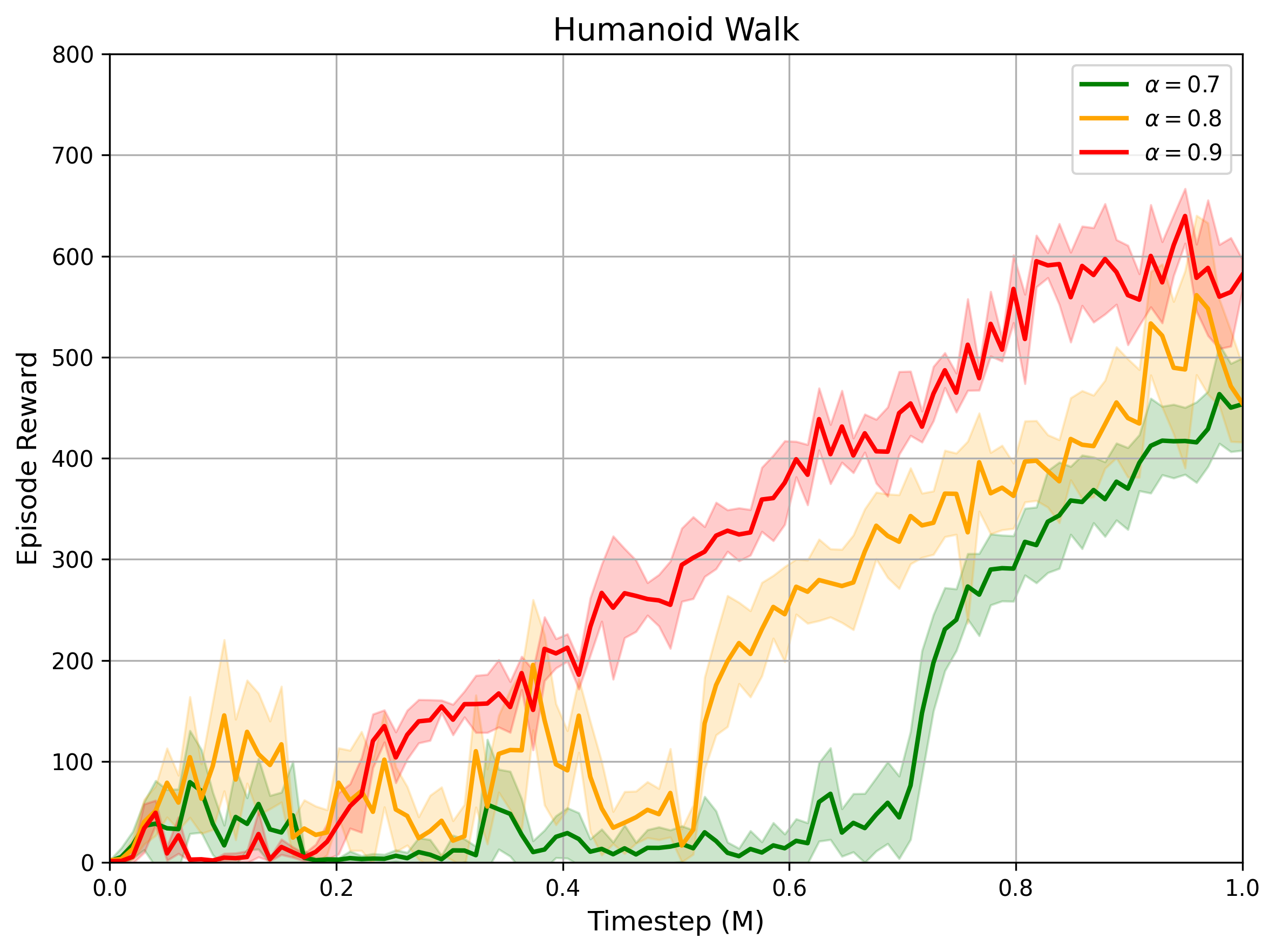}
        {\textbf{Ablation on $\alpha$}}
    \end{minipage}
    \caption{\textbf{Ablation Study on Hyperparameters} We conduct an ablation study on hyperparameter $\zeta$ and $\alpha$. The ablation study is conducted on Humanoid Walk task.}
    \label{fig:ablation-hyperparam}
\end{figure}

\paragraph{Ablation on the Hyperparameter Choice}
We conduct ablation studies on two hyperparameters, $\alpha$ and $\zeta$, introduced in Section \ref{sec:coupled-RED}, which are related to the construction of the reward model. Our experiments demonstrate that these parameters influence the model's convergence during the initial training phase, which is closely tied to the policy's exploration capability. For the hyperparameter $\zeta$, we find that smaller values may encourage exploration, leading to faster convergence. However, if $\zeta$ is too small, the model may fail to learn effectively. For the hyperparameter $\alpha$, larger values may enhance exploration, potentially promoting convergence. The results are aligned with our intuition given in Section \ref{sec:coupled-RED}. We perform the ablation study on the state-based Humanoid Walk task in the DMControl environment, and the results are presented in Figure \ref{fig:ablation-hyperparam}.

\subsection{Additional Comparison with HyPE}
\label{sec:hype}
Hybrid IRL \citep{ren2024hybrid} is a recently proposed method for performing inverse reinforcement learning and imitation learning using hybrid data. In this section, we compare our approach with the model-free method (HyPE) introduced in their work. Our method achieves superior empirical performance on three DMControl locomotion tasks, including the high-dimensional Humanoid Walk task. The results are presented in Figure \ref{fig:hype-comparison}.

\begin{figure}[h]
    \centering
    \begin{minipage}{0.3\textwidth}
        \centering
        \includegraphics[width=\textwidth]{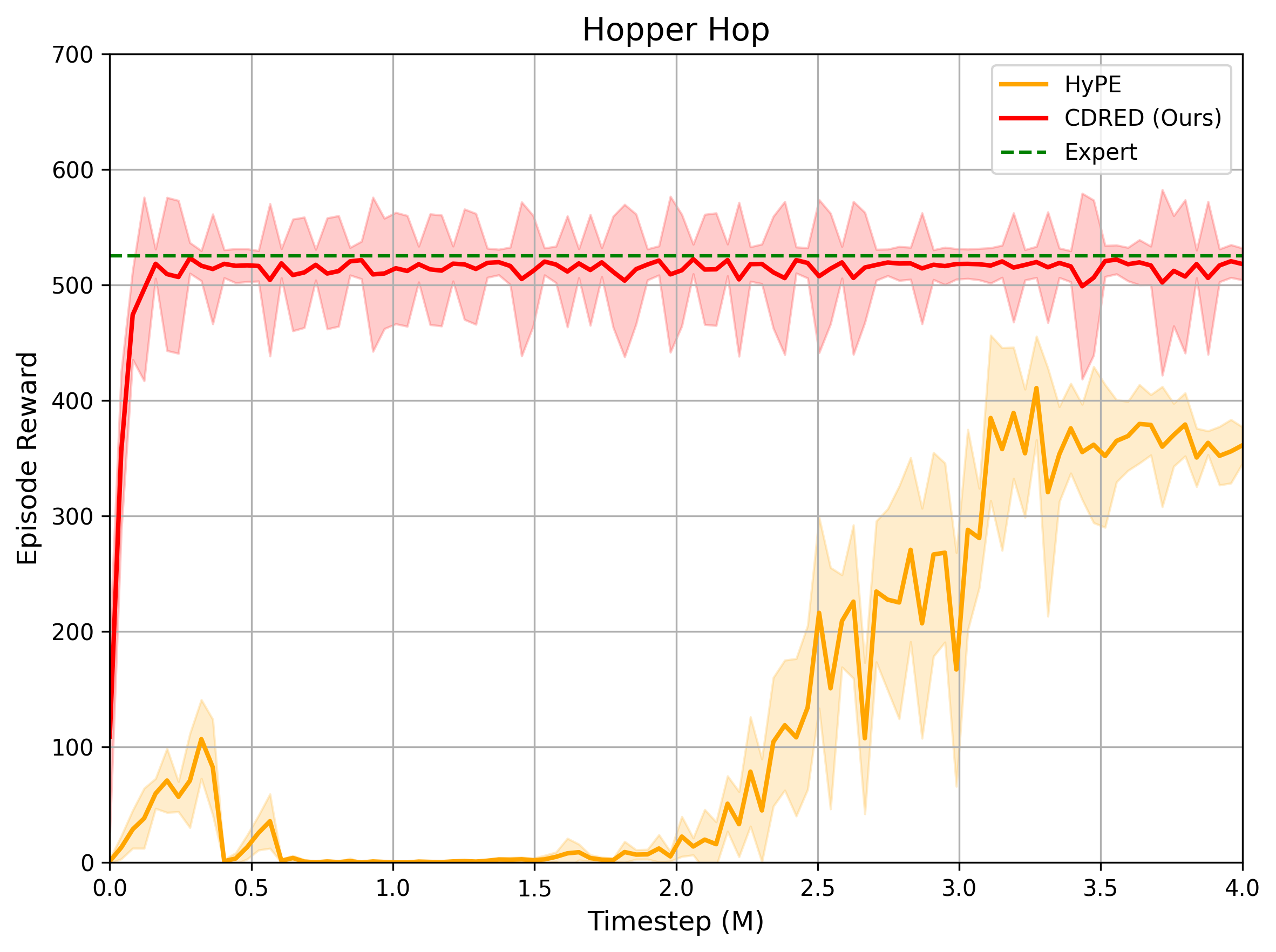}
    \end{minipage}
    \begin{minipage}{0.3\textwidth}
        \centering
        \includegraphics[width=\textwidth]{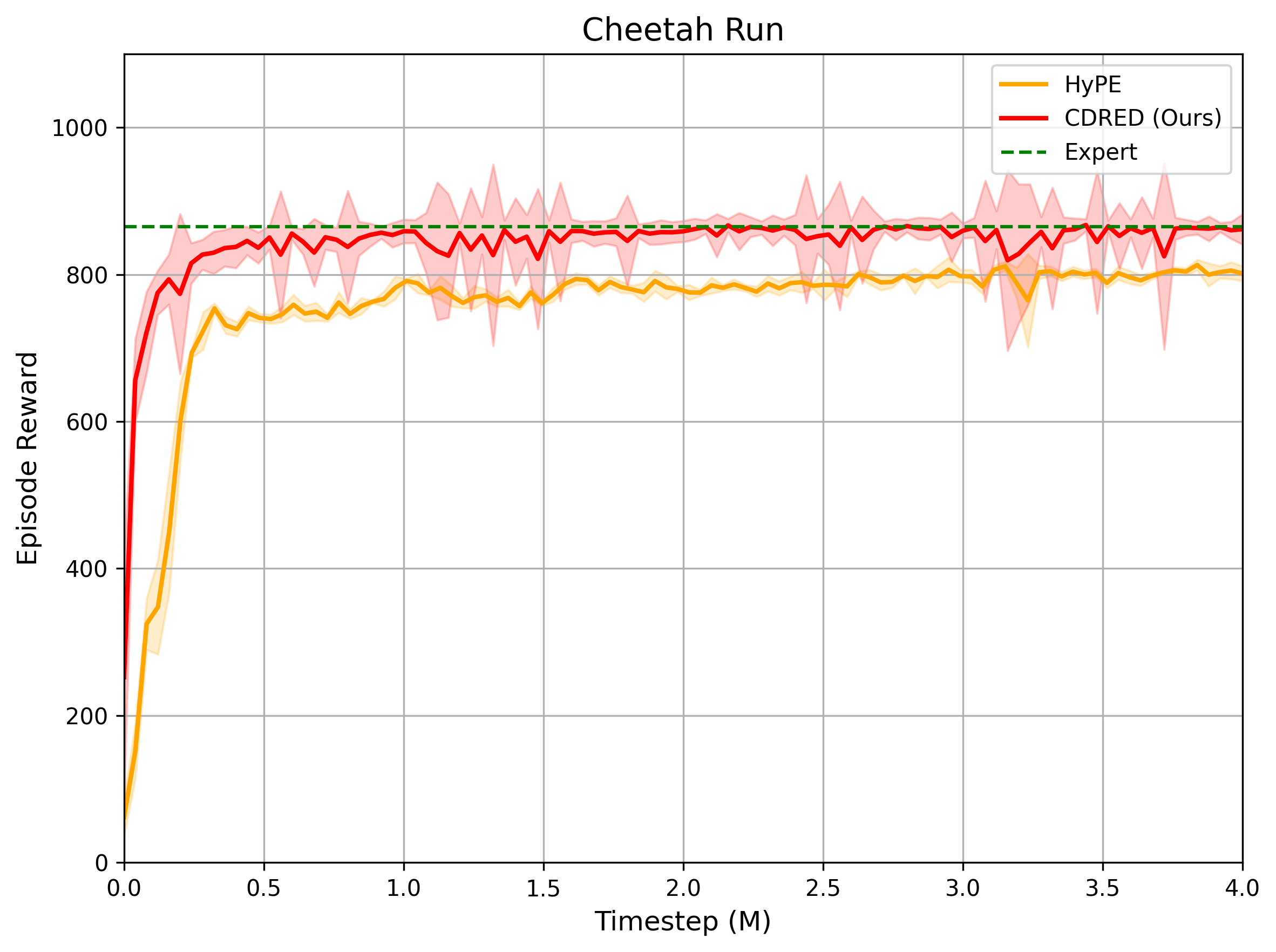}
    \end{minipage}
    \begin{minipage}{0.3\textwidth}
        \centering
        \includegraphics[width=\textwidth]{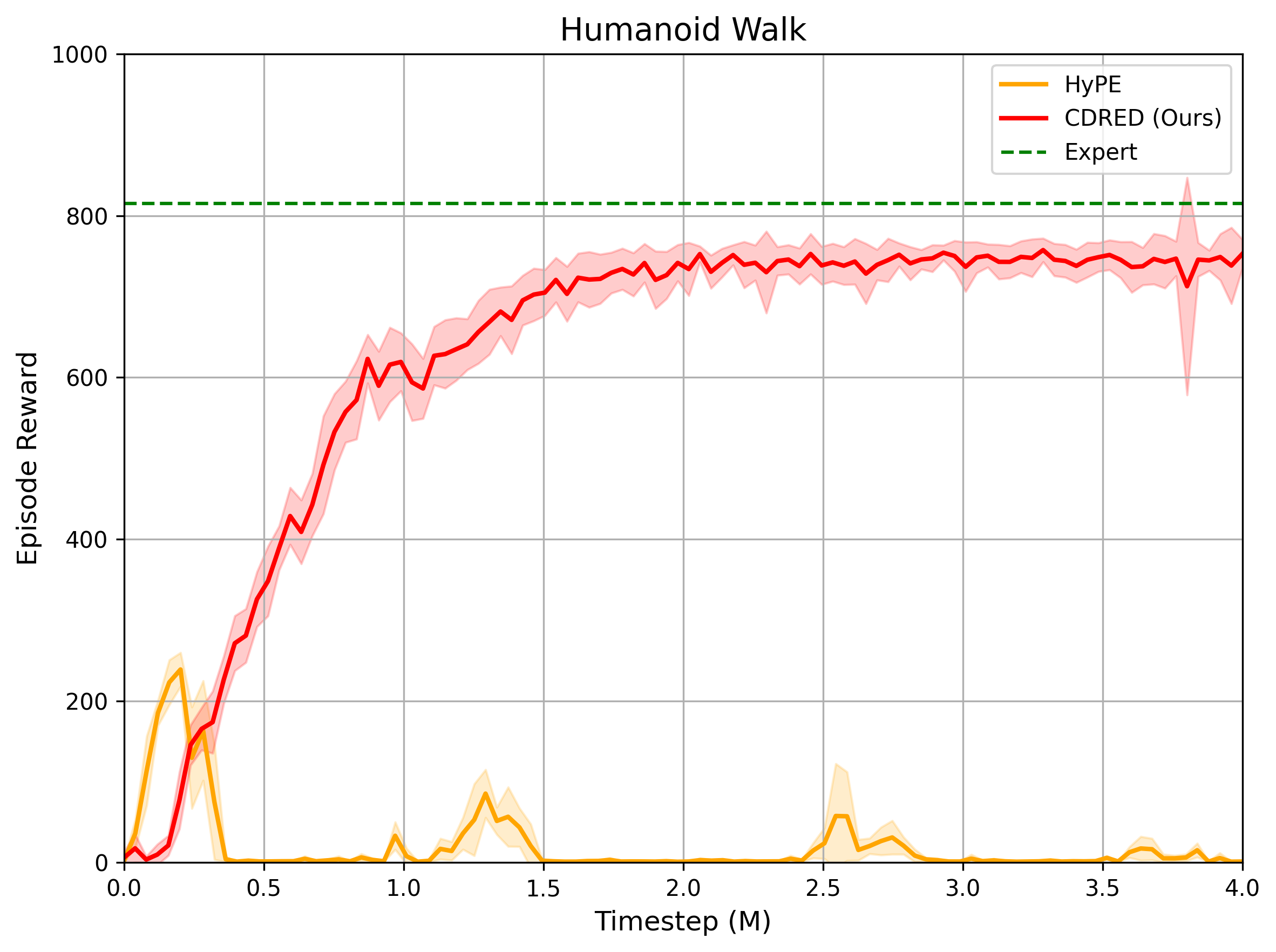}
    \end{minipage}
    \caption{\textbf{Comparison with HyPE} We compare our CDRED approach with the HyPE method \citep{ren2024hybrid} on the Hopper Hop, Cheetah Run, and Humanoid Walk tasks. Among these, the Humanoid Walk task is high-dimensional, while the others are low-dimensional. Our approach demonstrates superior empirical performance and improved sampling efficiency on these tasks.}
    \label{fig:hype-comparison}
\end{figure}

\subsection{Additional Comparison with SAIL}
\label{sec:sail}
Support-weighted Adversarial Imitation Learning (SAIL) \citep{SAIL} is an extension of Generative Adversarial Imitation Learning (GAIL) \citep{ho2016generative} that enhances performance by integrating Random Expert Distillation (RED) rewards \citep{wang2019random}. In this section, we present an additional comparative analysis between our proposed CDRED method and SAIL. The experimental results are illustrated in Figure \ref{fig:sail-comparison}.

\begin{figure}[h]
    \centering
    \begin{minipage}{0.3\textwidth}
        \centering
        \includegraphics[width=\textwidth]{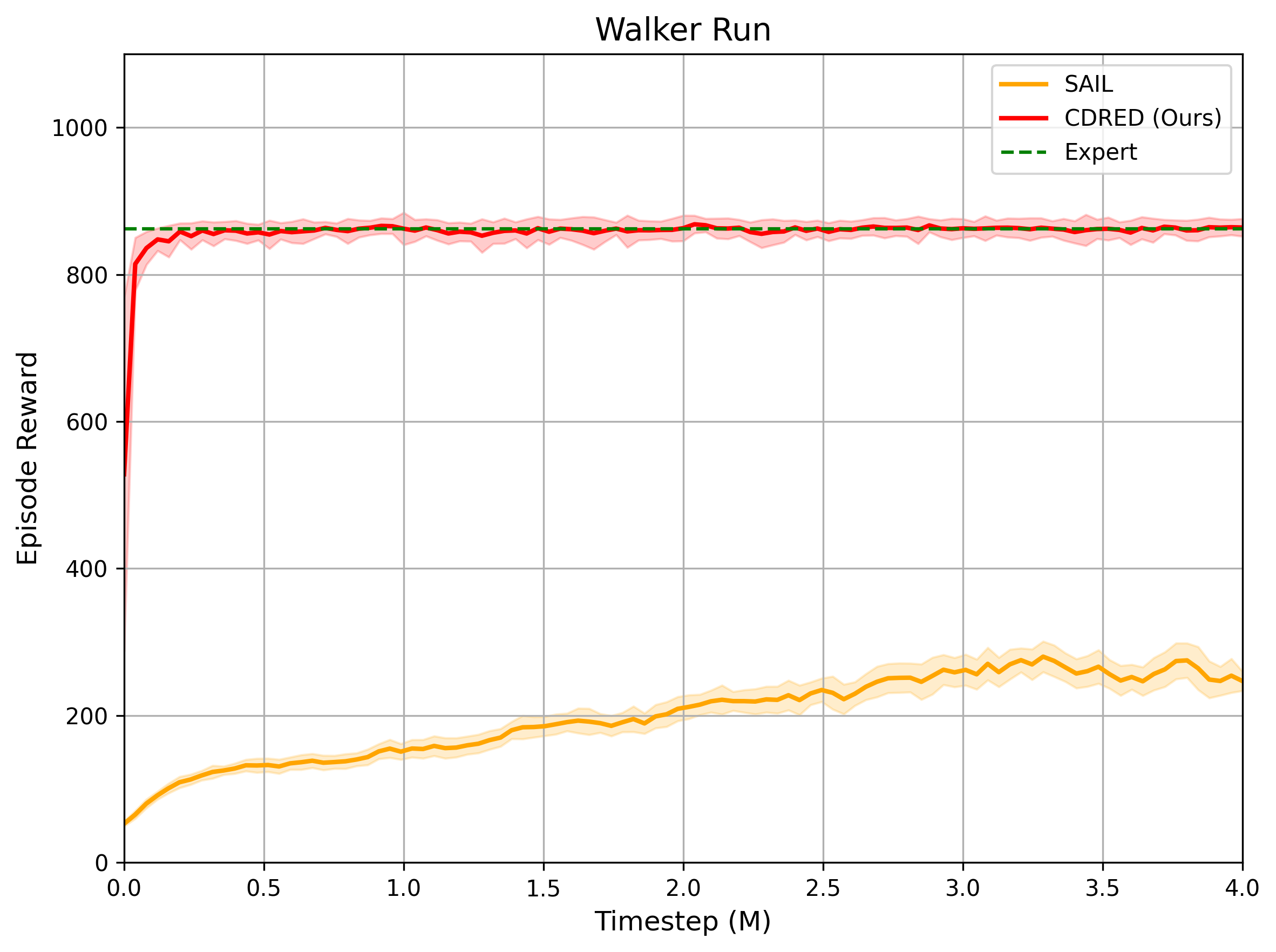}
    \end{minipage}
    \begin{minipage}{0.3\textwidth}
        \centering
        \includegraphics[width=\textwidth]{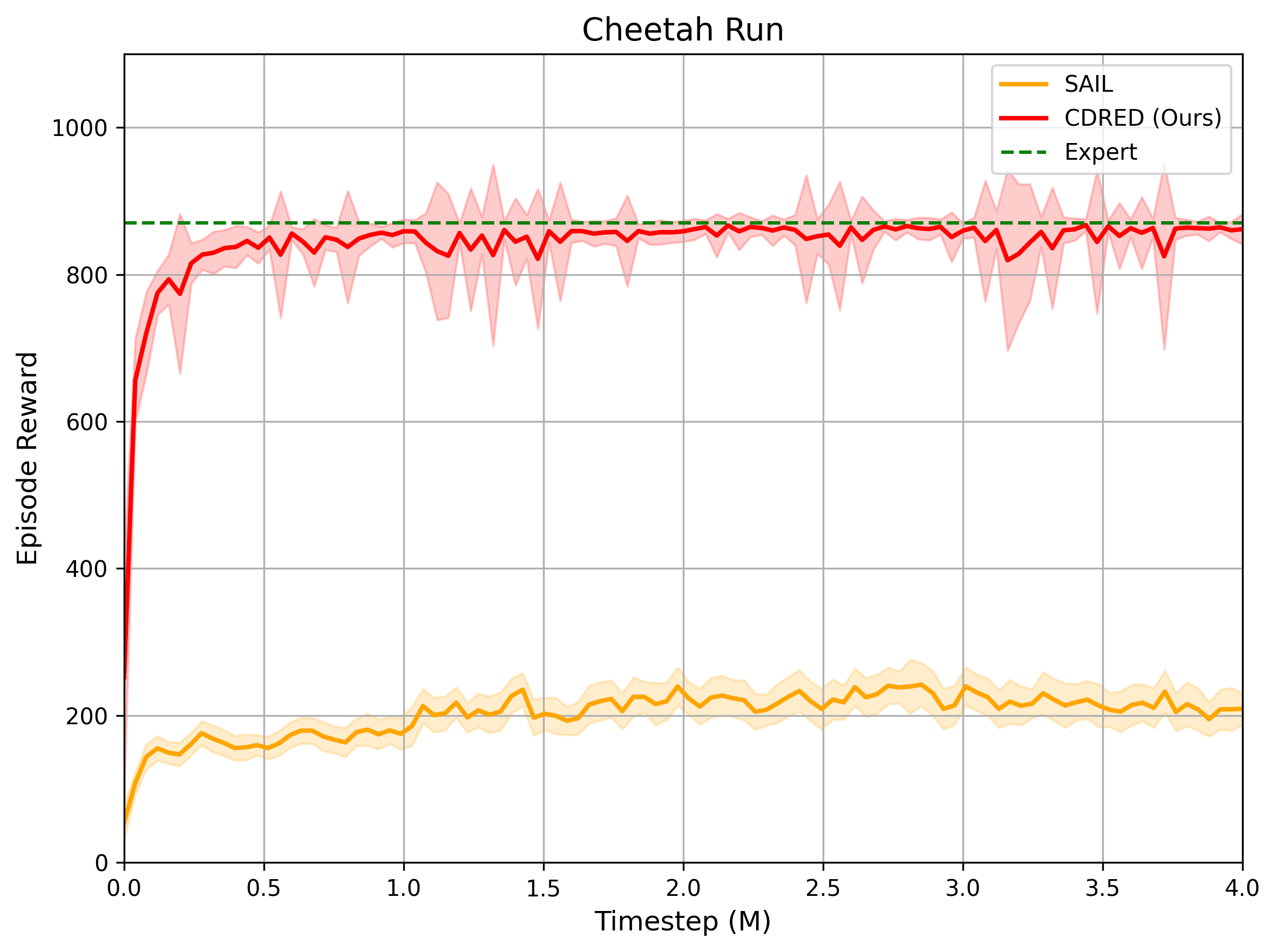}
    \end{minipage}
    \begin{minipage}{0.3\textwidth}
        \centering
        \includegraphics[width=\textwidth]{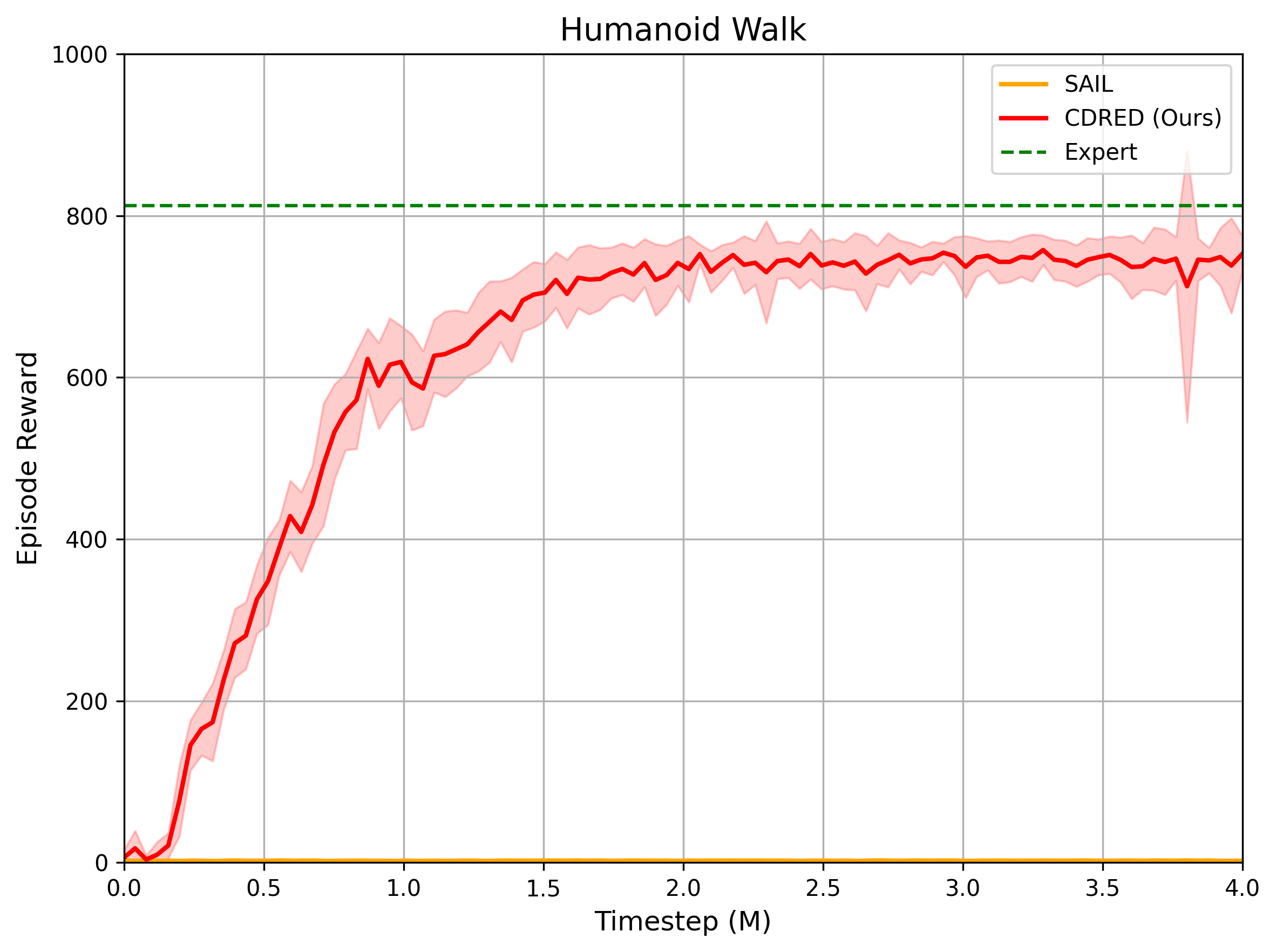}
    \end{minipage}
    \caption{\textbf{Comparison with SAIL} We compare our CDRED approach with the SAIL method \citep{SAIL} on the Walker Run, Cheetah Run, and Humanoid Walk tasks. Among these, the Humanoid Walk task is high-dimensional, while the others are low-dimensional. SAIL fails to learn in the high-dimensional Humanoid Walk task while our approach achieves nearly expert-level performance. Overall, our approach demonstrates superior empirical performance and improved sampling efficiency on these tasks.}
    \label{fig:sail-comparison}
\end{figure}

\subsection{Robustness Analysis under Noisy Dynamics}
\label{sec:noisy-dynamics}
We conduct an additional analysis to evaluate the robustness of our model under noisy environment dynamics. Following the evaluation protocol of Hybrid IRL \citep{ren2024hybrid}, we introduce noise by adding a trembling probability, $p_{\text{tremble}}$. During interactions with the environment, the agent executes a random action with probability $p_{\text{tremble}}$ and follows the action generated by the policy for the remaining time. Our empirical results demonstrate that our model exhibits robustness to noisy dynamics, as its performance only slightly deteriorates from the expert level when noise is introduced. The results for the Cheetah Run and Walker Run tasks are presented in Figure \ref{fig:robustness}.

\begin{figure}[h]
    \centering
    \begin{minipage}{0.45\textwidth}
        \centering
        \includegraphics[width=\textwidth]{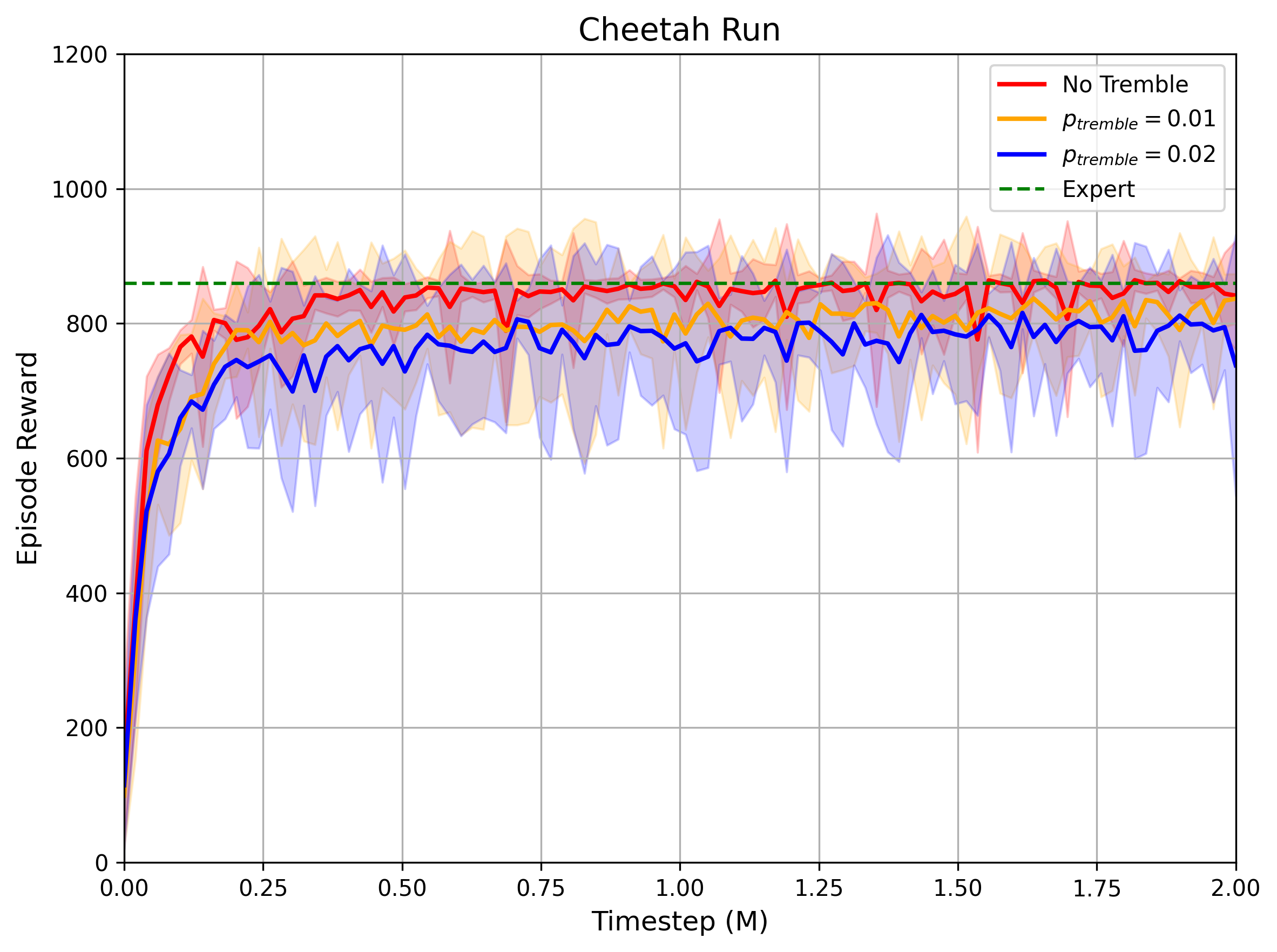}
    \end{minipage}
    \begin{minipage}{0.45\textwidth}
        \centering
        \includegraphics[width=\textwidth]{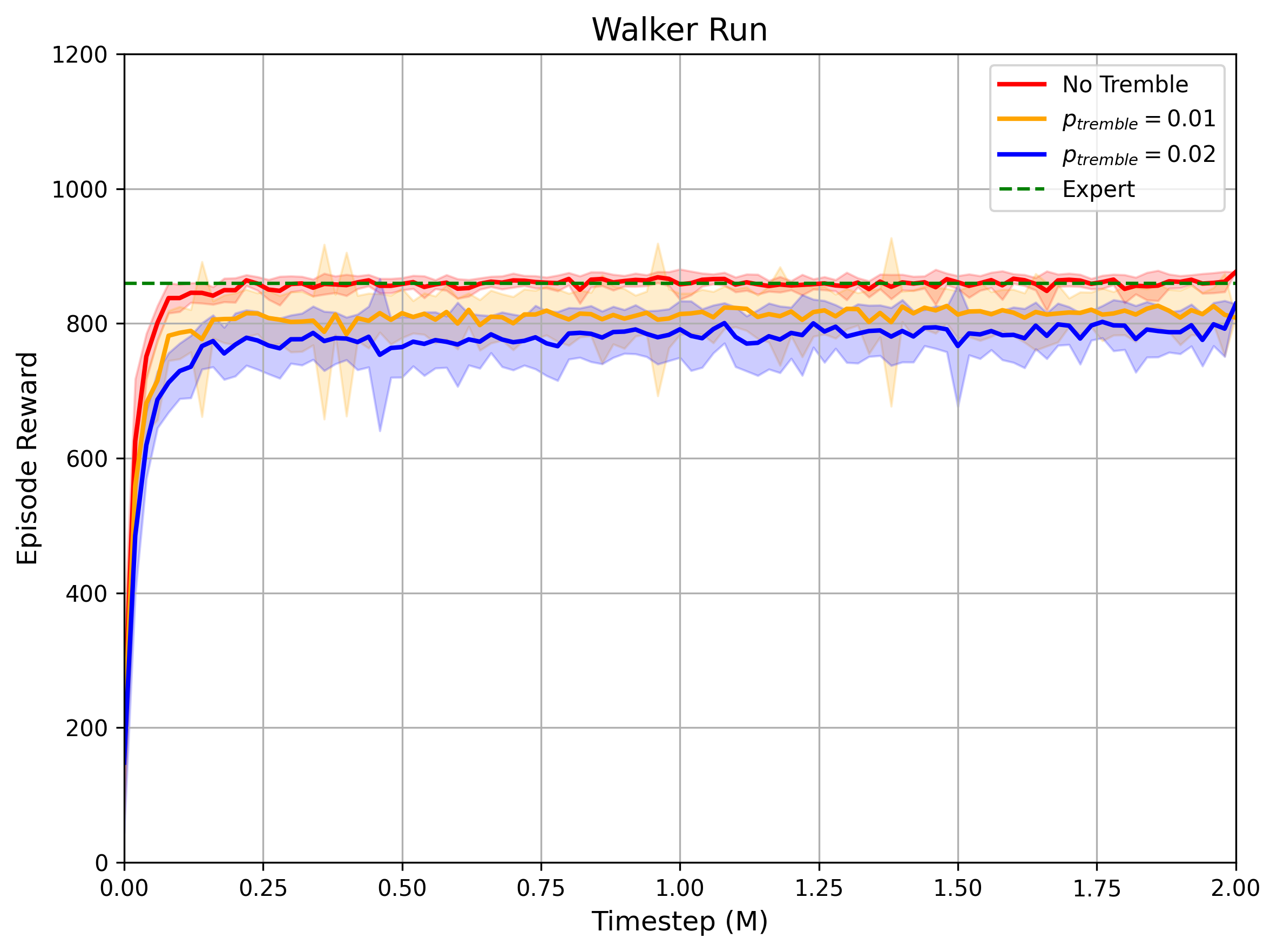}
    \end{minipage}
    \caption{\textbf{Robustness Analysis under Noisy Environment Dynamics} We analyze the performance of our model on the Cheetah Run and Walker Run tasks under stochastic environment dynamics. Our results demonstrate that the model shows notable robustness to noise in the environment dynamics.}
    \label{fig:robustness}
\end{figure}

\subsection{Quantitative Analysis of Training Stability}
\label{sec:training-stability}

To assess the training stability of our algorithm, we examine the mean and maximum gradient norms throughout the training process. This approach is similar to the analysis conducted in TD-MPC2 \citep{hansen2023td}. We compare the gradient norms of our method with those of IQ-MPC \citep{li2024rewardfreeworldmodelsonline}, a world model online imitation learning approach that employs an adversarial formulation, on DMControl tasks. Our results indicate that the gradient norms of our approach are significantly smaller than those of IQ-MPC, suggesting superior training stability. The detailed comparison is presented in Table \ref{tab:gradient_norm}.

\begin{table}[H]
    \centering
    \begin{tabular}{l|cc|cc}
        \toprule
        \textbf{Gradient Norm} & \textbf{IQ-MPC (mean)} & \textbf{CDRED (mean)} & \textbf{IQ-MPC (max)} & \textbf{CDRED (max)} \\
        \midrule
        Humanoid Walk & 12.6 & 0.073 & 198.3 & 0.32 \\
        Hopper Hop & 324.8 & 1.3 & 8538.6 & 4.6 \\
        Cheetah Run & 131.7 & 0.34 & 2342.6 & 3.1 \\
        Walker Run & 344.6 & 0.26 & 1534.7 & 1.8 \\
        Reacher Hard & 11.3 & 0.012 & 65.8 & 0.083 \\
        Dog Walk & 989.7 & 0.059 & 6824.3 & 0.13 \\
        \bottomrule
    \end{tabular}
    \vspace{+5pt}
    \caption{\textbf{Training Stability Analysis} Comparison of gradient norms between our CDRED approach and the IQ-MPC method. The significantly smaller gradient norms of our approach indicate enhanced training stability.}
    \label{tab:gradient_norm}
\end{table}

\subsection{Advantages Compared to Current Methods Involving Adversarial Training}
\label{sec:drawbacks-iqmpc}
The current existing methods \citep{li2024rewardfreeworldmodelsonline,kolev2024efficient,rafailov2021visual,yin2022planning} for world model online imitation learning often involve adversarial training, following the similar problem formulation as GAIL \citep{ho2016generative} or IQ-Learn \citep{garg2021iq}. IQ-MPC \citep{li2024rewardfreeworldmodelsonline} adopted inverse soft-Q objective for critic learning while CMIL \citep{kolev2024efficient}, V-MAIL \citep{rafailov2021visual} and EfficientImitate \citep{yin2022planning} leveraged GAIL style reward modeling. In terms of IQ-Learn, an improved version of GAIL, although its policy can be computed by applying a softmax to the Q-value in discrete control, effectively converting a min-max problem into a single maximization \citep{garg2021iq}, it still requires the maximum entropy RL objective for policy updates in continuous control settings. In such cases, IQ-Learn performs adversarial training between the policy and the critic, which leads to stability issues similar to those encountered in GAIL. IQ-MPC, while performing well in various complex scenarios such as high-dimensional locomotion control and dexterous hand manipulation, still encounters challenges in some cases. These challenges include an imbalance between the discriminator and the policy, as well as long-term instability. These issues stem from using an adversarially trained Q-function as the critic. While IQ-MPC attempts to mitigate them by incorporating regularization terms during the training process, it doesn't fully resolve the problem. Figure \ref{fig:iqmpc-drawbacks} illustrates the drawbacks of IQ-MPC in some cases, namely an overly powerful discriminator and long-term instability. We also demonstrate the quantitative results for training stability analysis in Appendix \ref{sec:training-stability}.

\begin{figure}[h]
    \centering
    \begin{minipage}{0.45\textwidth}
        \centering
        \includegraphics[width=\textwidth]{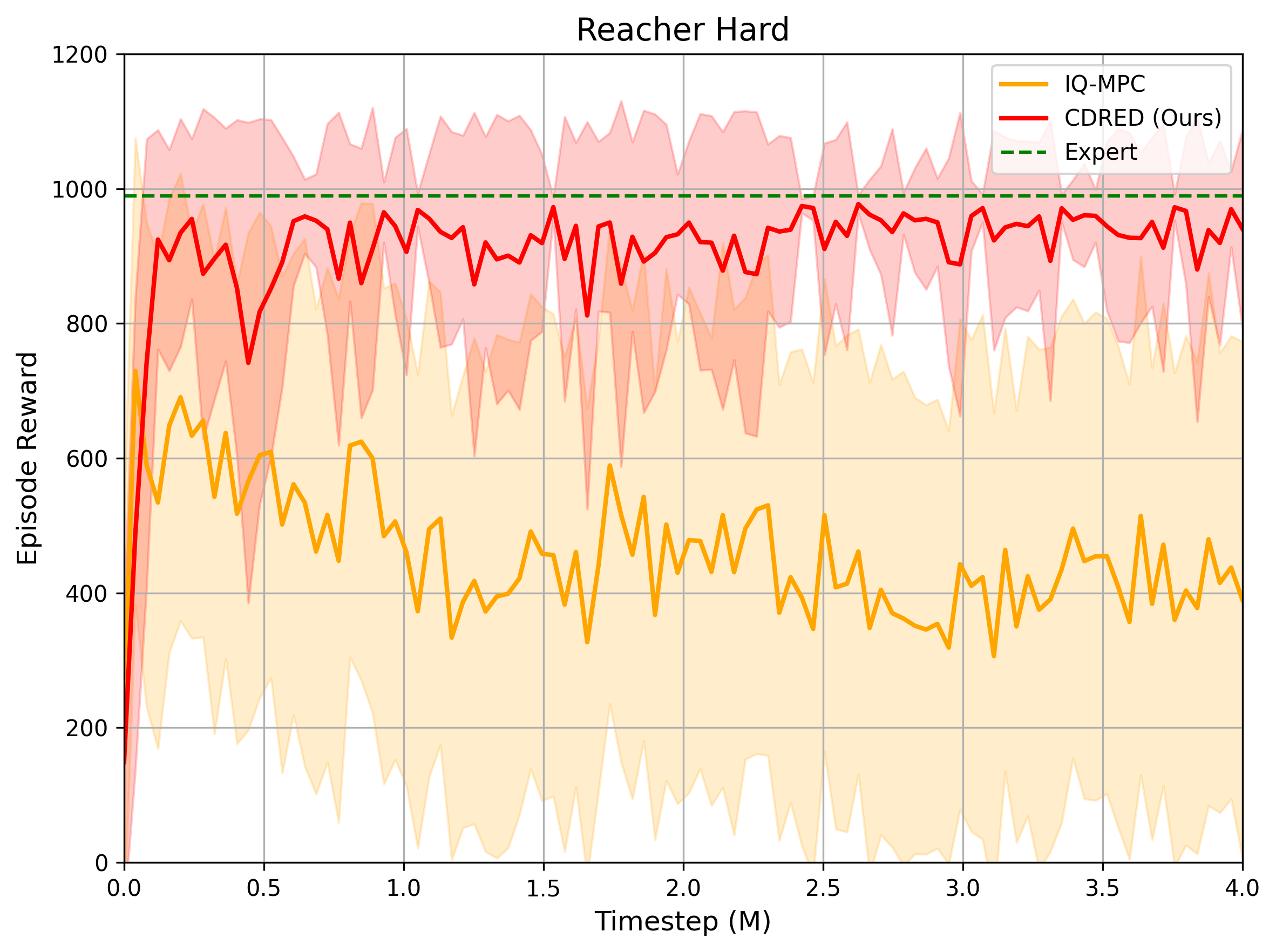}
        \\ \textbf{Overly powerful discriminator}
    \end{minipage}
    \begin{minipage}{0.45\textwidth}
        \centering
        \includegraphics[width=\textwidth]{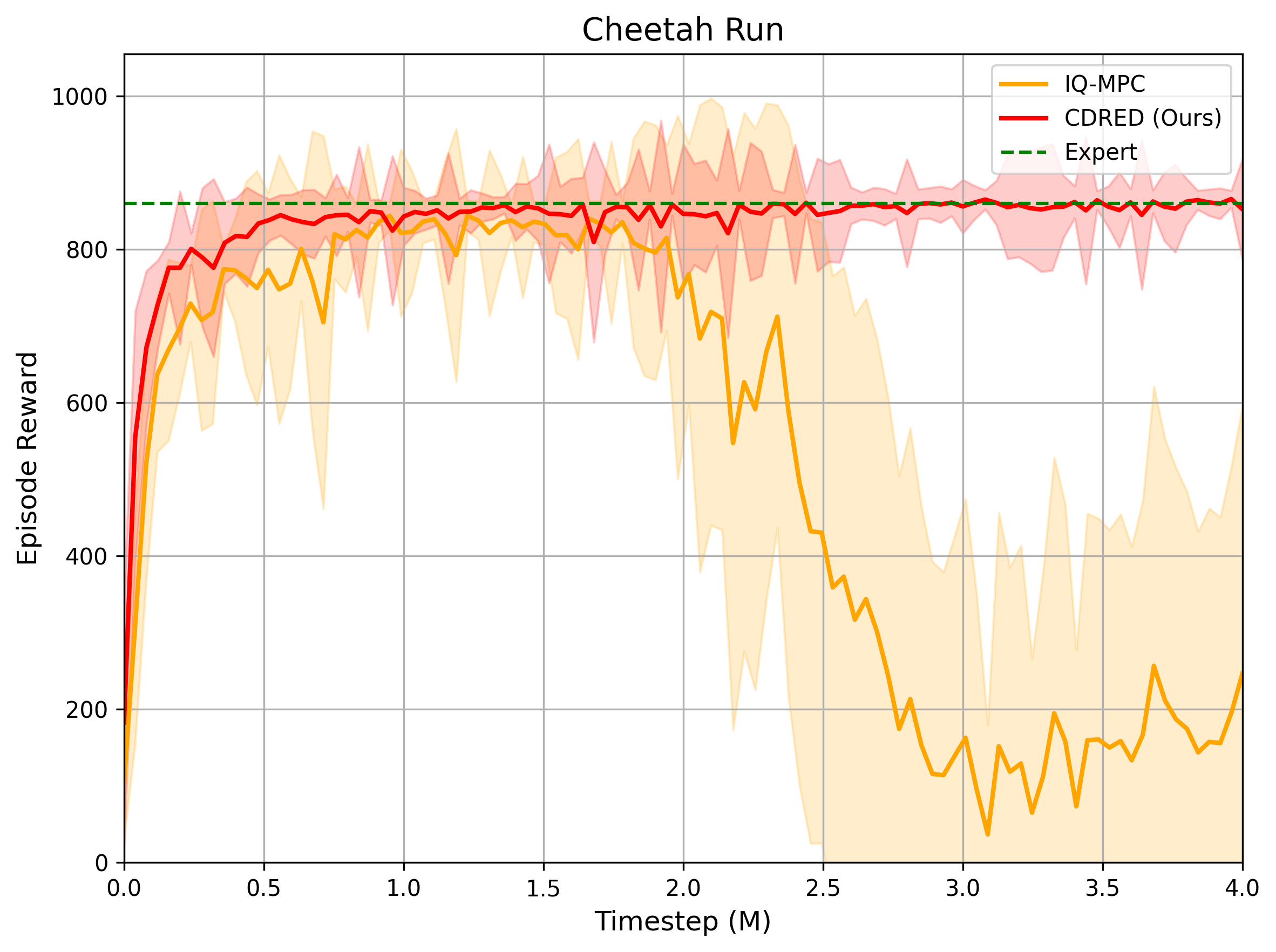}
        \\ \textbf{Long-term instability}
    \end{minipage}
    \caption{\textbf{Drawbacks of Methods Including Adversarial Training} We demonstrate the drawbacks of IQ-MPC \citep{li2024rewardfreeworldmodelsonline} in some tasks, which employs adversarial training for online imitation learning. An overly powerful discriminator (Left) leads to sub-optimal policy learning, while long-term instability (Right) of adversarial training prevents IQ-MPC from maintaining expert-level performance during extended online training. Our CDRED method, which replaces adversarial training with density estimation, is immune to these issues.}
    \label{fig:iqmpc-drawbacks}
\end{figure}

\paragraph{Overly Powerful Discriminator} The generative adversarial training process is often prone to instability \citep{gulrajani2017improved}. IQ-MPC employs generative adversarial training between the policy and the critic, and it also encounters this challenge. To mitigate this issue, IQ-MPC leverages gradient penalty from \cite{gulrajani2017improved} to enforce Lipschitz condition of the gradients in a form of:
\begin{equation}
        \mathcal{L}^{pen}=\sum_{t=0}^H\lambda^t\Bigg[\mathbb{E}_{(\mathbf{\hat s}_t, \mathbf{\hat a}_t)\sim\mathcal{B}}\Big(\Vert \nabla Q(\mathbf{\hat z}_t,\mathbf{\hat a}_t)\Vert_2-1\Big)^2\Bigg]
\end{equation}
In the gradient penalty, $(\mathbf{\hat s}_t,\mathbf{\hat a}_t)$ are data points on straight lines between expert and behavioral distributions, which are generated by linear interpolation. Although it counters the problem to some extent, the performance of IQ-MPC is still not satisfactory in some tasks such as Reacher in DMControl and Meta-World robotics manipulation tasks, for which we will refer to our experimental results in Section \ref{sec:experiments}. An overly powerful discriminator often causes the Q-value difference between the policy and expert distributions to diverge, as noted by \cite{li2024rewardfreeworldmodelsonline}. Specifically, this divergence is reflected in the gap between the expected Q-values under the expert distribution, $\mathbb{E}_{(\mathbf{s},\mathbf{a})_{(0:H)}\sim\mathcal{B}_E}Q(\mathbf{z}_t,\mathbf{a}_t)$, and the policy distribution, $\mathbb{E}_{(\mathbf{s},\mathbf{a})_{(0:H)}\sim\mathcal{B}_\pi}Q(\mathbf{z}_t,\mathbf{a}_t)$. While IQ-MPC can mitigate this divergence to some extent through gradient penalty, it does not eliminate the difference entirely, indicating that the policy does not achieve expert-level performance. We show the Q difference plot in a problematic case in Figure \ref{fig:q-diff}.

\begin{figure}
    \centering
    \includegraphics[width=0.5\linewidth]{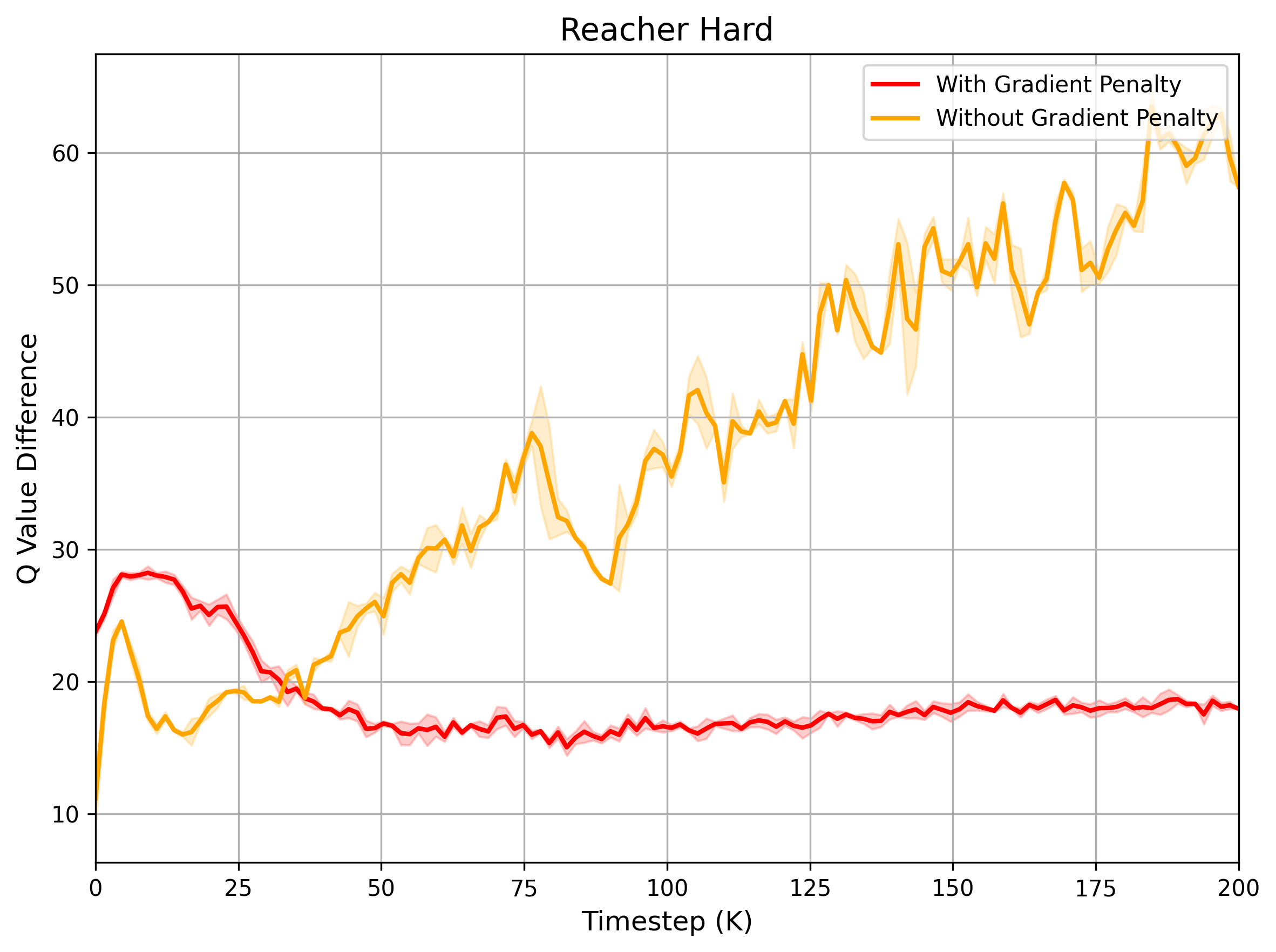}
    \caption{\textbf{IQ-MPC Q Value Difference Visualization} We present the Q-difference plot for IQ-MPC in a problematic scenario (Reacher Hard task in DMControl) where it is affected by an overly powerful discriminator. Although applying a gradient penalty prevents the Q-difference from diverging, it still fails to converge to a value near zero, resulting in a persistently large Q-difference throughout training.}
    \label{fig:q-diff}
\end{figure}

\paragraph{Long-term Instability} Since we're conducting online imitation learning, we prefer to train a policy that can reach expert-level and maintain stable expert-level performance during further training, which is the long-term training stability. Due to the use of adversarial training, we find it hard for IQ-MPC to maintain stable expert-level performance during extensive long-term online training.

\subsection{Improvement of Constructing the Reward Model on the Latent Space}
\label{sec:latent-space}

Original RND \citep{burda2018exploration} and Random Expert Distillation \citep{wang2019random} train their reward or bonus models directly on the original observation space. In contrast, we found that constructing the CDRED reward model using the latent representations from a world model yields better empirical performance. This highlights the superior properties of latent representations, which enable more accurate reward estimation. Furthermore, by training a latent dynamics model within this space, the representations become more dynamics-aware, facilitating the construction of a reward model that effectively captures the underlying dynamics.

To validate this, we compared training the CDRED reward model on the original observation space versus the latent space. Our results indicate that while training on the observation space may exhibit slightly suboptimal behavior in low-dimensional settings, it fails entirely in high-dimensional cases due to the challenges of density estimation on raw observations. These findings are illustrated in Figure \ref{fig:obs-latent-comparsion}.

\begin{figure}[h]
    \centering
    \begin{minipage}{0.45\textwidth}
        \centering
        \includegraphics[width=\textwidth]{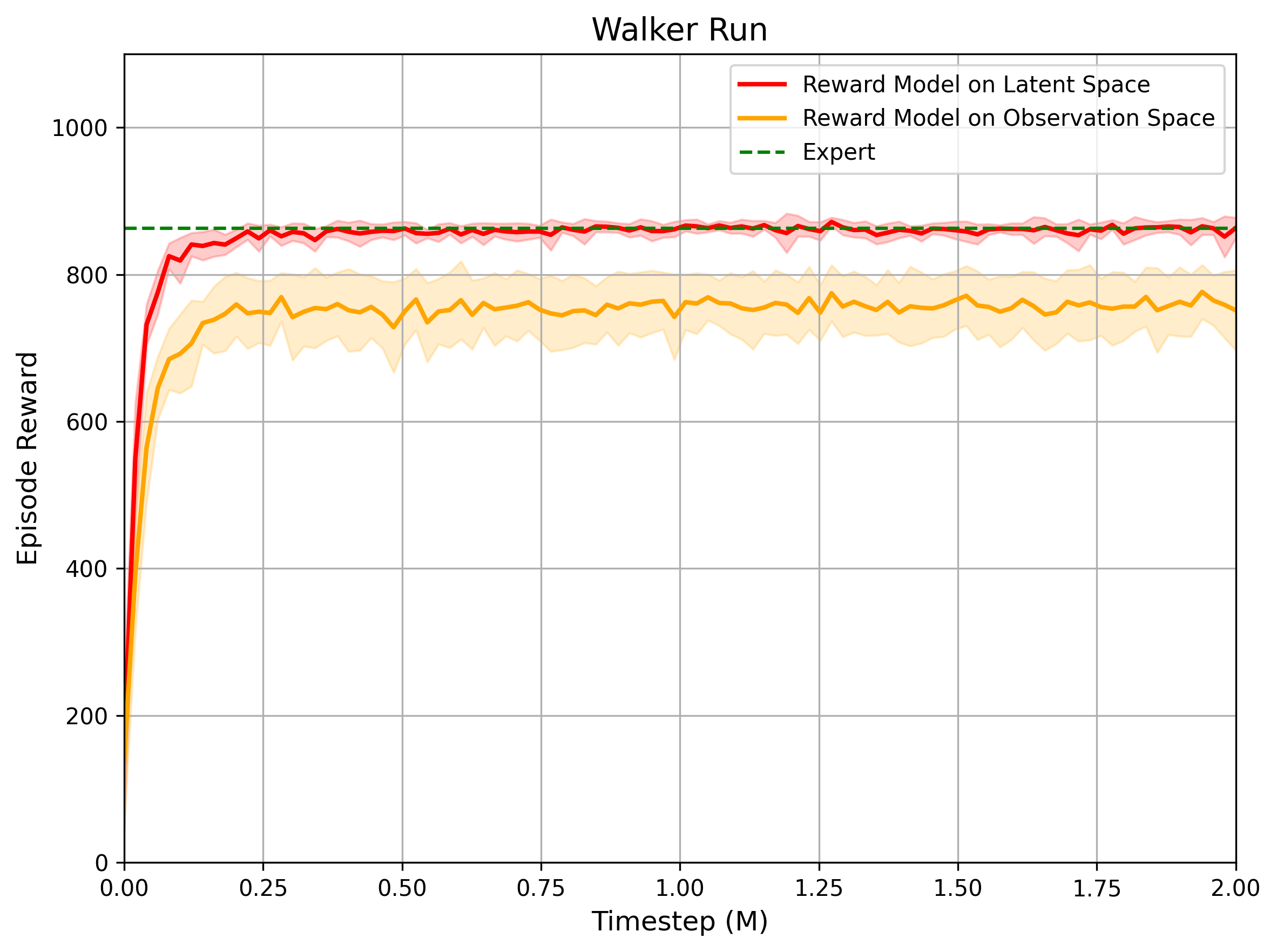}
    \end{minipage}
    \begin{minipage}{0.45\textwidth}
        \centering
        \includegraphics[width=\textwidth]{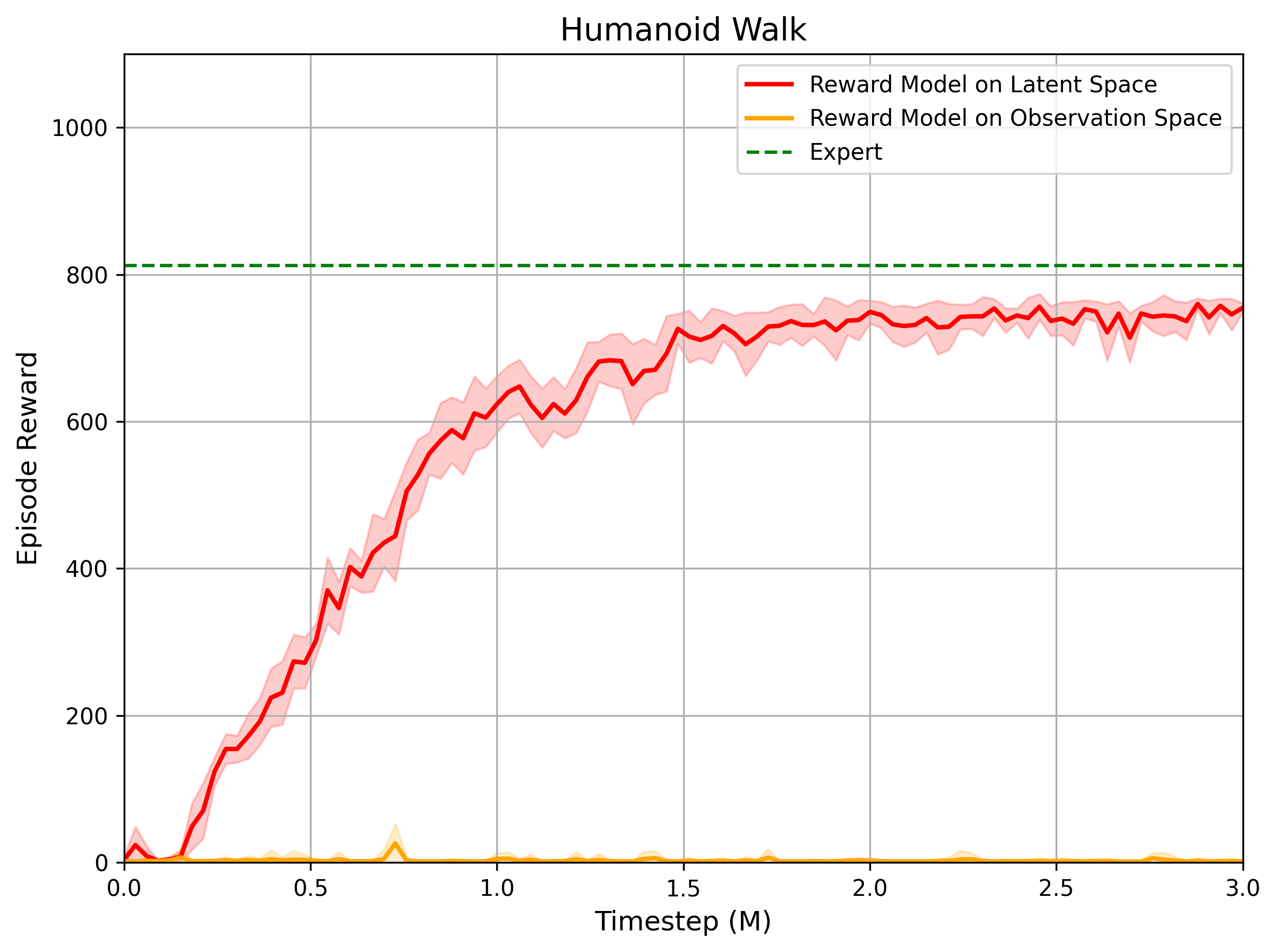}
    \end{minipage}
    \caption{\textbf{Effectiveness of the latent space CDRED reward model} We conduct comparative experiments to evaluate the performance of the CDRED reward model when trained on the latent space of the world model versus the original observation space. Our results show that training the CDRED reward model on the latent space yields superior empirical performance.}
    \label{fig:obs-latent-comparsion}
\end{figure}

\section{Proof of Lemma \ref{lem:unbiased-estimator}}
\label{sec:additional-proof}

For completeness, we adapt the proof from \cite{yang2024exploration} to construct the proof of Lemma \ref{lem:unbiased-estimator}. For a latent state-action pair $(\mathbf{z},\mathbf{a})$ sampled from a latent state-action distribution $\rho$. We denote the moments of the distribution of random variable $c(\mathbf{z},\mathbf{a})$ as:
\begin{align*}
&\mu_{\bar\theta}(\mathbf{z},\mathbf{a}) = \mathbb{E}\Big[{f_{\bar\theta_k}(\mathbf{z},\mathbf{a})}\Big] = \frac{1}{K} \sum_{k=0}^{K-1}f_{\bar\theta_k}(\mathbf{z},\mathbf{a}), \quad\quad\quad
B_2(\mathbf{z},\mathbf{a}) = \mathbb{E}\Big[{(f_{\bar\theta_k}(\mathbf{z},\mathbf{a}))}^2\Big] = \frac{1}{K} \sum_{k=0}^{K-1}(f_{\bar\theta_k}(\mathbf{z},\mathbf{a}))^2,\\
&B_3(\mathbf{z},\mathbf{a}) = \mathbb{E}\Big[{(f_{\bar\theta_k}(\mathbf{z},\mathbf{a}))}^3\Big] = \frac{1}{K} \sum_{k=0}^{K-1}(f_{\bar\theta_k}(\mathbf{z},\mathbf{a}))^3,\quad\quad
B_4(\mathbf{z},\mathbf{a}) = \mathbb{E}\Big[{(f_{\bar\theta_k}(\mathbf{z},\mathbf{a}))}^4\Big] = \frac{1}{K} \sum_{k=0}^{K-1}(f_{\bar\theta_k}(\mathbf{z},\mathbf{a}))^4.
\end{align*}
The calculation for the moments of $f^*(\mathbf{z},\mathbf{a})$ is as follows:
\begin{equation*}
\mathbb{E}[f_*(\mathbf{z},\mathbf{a})] = \mathbb{E}[\frac{1}{n} \sum_{i=1}^{n}c_i(\mathbf{z},\mathbf{a})] 
= \frac{1}{n}\mathbb{E}[\sum_{i=1}^{n}c_i(\mathbf{z},\mathbf{a})]
= \mu_{\bar\theta}(\mathbf{z},\mathbf{a}).
\end{equation*}
\begin{align*}
\mathbb{E}[f_*^2(\mathbf{z},\mathbf{a})] &= \mathbb{E}[(\frac{1}{n} \sum_{i=1}^{n}c_i(\mathbf{z},\mathbf{a}))^2] \\
&= \frac{1}{n^2}\mathbb{E}[(\sum_{i=1}^{n}c^2_i(\mathbf{z},\mathbf{a})+\sum_{i=1}^{n}\sum_{j\ne i}^{n}c_i(\mathbf{z},\mathbf{a})c_j(\mathbf{z},\mathbf{a}))] \\
&= \frac{1}{n^2}\mathbb{E}[nc^2(\mathbf{z},\mathbf{a}) + n(n-1)\mu_{\bar\theta}^2(\mathbf{z},\mathbf{a})]\\
&= \frac{B_2(\mathbf{z},\mathbf{a})}{n}+ \frac{n-1}{n}\mu_{\bar\theta}^2(\mathbf{z},\mathbf{a}).
\end{align*}
\begin{align*}
\mathbb{E} [f^4_*(\mathbf{z},\mathbf{a})] &= \frac{1}{n^{4}} \mathbb{E}\left[\sum_{i = 1}^{n} c_{i}(\mathbf{z},\mathbf{a})\right]^{4} \\
&= \frac{1}{n^{4}}\left(\mathbb{E}\left[\sum_{i = 1} c_{i}(\mathbf{z},\mathbf{a})^{4}\right]+\right. \quad 4 \mathbb{E}\left[\sum_{i \neq j} c_{i}^{3}(\mathbf{z},\mathbf{a}) c_{j}(\mathbf{z},\mathbf{a})\right]+\quad 3 \mathbb{E}\left[\sum_{i \neq j} c^{2}_{i}(\mathbf{z},\mathbf{a}) c^{2}_{j}(\mathbf{z},\mathbf{a})\right]\\
&\left.\quad+6E\left[\sum_{i \neq j \neq k} c_{i}(\mathbf{z},\mathbf{a}) c_{j}(\mathbf{z},\mathbf{a}) c_{k}^{2}(\mathbf{z},\mathbf{a})\right] +\quad \mathbb{E}\left[\sum_{i \neq j \neq k \neq l} c_{i}(\mathbf{z},\mathbf{a}) c_{j}(\mathbf{z},\mathbf{a}) c_{k}(\mathbf{z},\mathbf{a}) c_{l}(\mathbf{z},\mathbf{a})\right]\right)  \\
&= \frac{n B_4(\mathbf{z},\mathbf{a})+4 A_{n}^{2} \mu_{\bar\theta}(\mathbf{z},\mathbf{a}) B_3(\mathbf{z},\mathbf{a})+3 A_{n}^{2}B_2^2(\mathbf{z},\mathbf{a})+6A_{n}^{3} \mu_{\bar\theta}^{2}(\mathbf{z},\mathbf{a}) B_2(\mathbf{z},\mathbf{a})+A_{n}^{4} \mu_{\bar\theta}^{4}(\mathbf{z},\mathbf{a})}{n^{4}}.\\
&(A_n^i=\frac{n!}{(n-i)!} )
\end{align*}
The statistic $y(\mathbf{z},\mathbf{a})$ is defined as follows in Lemma \ref{lem:unbiased-estimator}:
$$
y(\mathbf{z},\mathbf{a}) = \frac{f_*^2(\mathbf{z},\mathbf{a}) - \mu_{\bar\theta}^2(\mathbf{z},\mathbf{a})}{B_2(\mathbf{z},\mathbf{a}) - \mu_{\bar\theta}^2(\mathbf{z},\mathbf{a})},
$$
and its expectation is:
$$
\mathbb{E}[y(\mathbf{z},\mathbf{a})] = \frac{\mathbb{E}[f_*^2(\mathbf{z},\mathbf{a})] - \mu_{\bar\theta}^2(\mathbf{z},\mathbf{a})}{B_2(\mathbf{z},\mathbf{a}) - \mu_{\bar\theta}^2(\mathbf{z},\mathbf{a})} = \frac{1}{n}.
$$
This implies that the statistic $y(\mathbf{z}, \mathbf{a})$ serves as an unbiased estimator for the reciprocal of the frequency of $(\mathbf{z}, \mathbf{a})$. The variance of $y(\mathbf{z}, \mathbf{a})$ is given by:
\begin{align*}
Var[y(\mathbf{z},\mathbf{a})] & = \frac{Var[f_*^2(\mathbf{z},\mathbf{a})]}{(B_2(\mathbf{z},\mathbf{a}) - \mu_{\bar\theta}^2(\mathbf{z},\mathbf{a}))^2}\\ & = \frac{\mathbb{E}[f^4_*(\mathbf{z},\mathbf{a})]-\mathbb{E}^2[f^2_*(\mathbf{z},\mathbf{a})]}{(B_2(\mathbf{z},\mathbf{a}) - \mu_{\bar\theta}^2(\mathbf{z},\mathbf{a}))^2}\\ & = \frac{K_1B_4(\mathbf{z},\mathbf{a})+K_2\mu_{\bar\theta}(\mathbf{z},\mathbf{a})B_3(\mathbf{z},\mathbf{a})+K_3B^2_2(\mathbf{z},\mathbf{a})+K_4\mu_{\bar\theta}^2(\mathbf{z},\mathbf{a})B_2(\mathbf{z},\mathbf{a})+K_5\mu_{\bar\theta}^4(\mathbf{z},\mathbf{a})}{n^3(B_2(\mathbf{z},\mathbf{a}) - \mu_{\bar\theta}^2(\mathbf{z},\mathbf{a}))^2}
\end{align*}
where
\begin{align*}
&K_1 = 1, \quad K_2 = 4n-4, \quad K_3 = 2n-3,\\
&K_4 = 4n^2-16n+12, \quad K_5 = -5n^2+10n-6.
\end{align*}
so we have:
$$
\lim_{n \to \infty} Var[y(\mathbf{z},\mathbf{a})] = 0.
$$
As $n$ approaches infinity, the variance of the statistic approaches zero, indicating the stability and consistency of $y(\mathbf{z}, \mathbf{a})$.

\section{Related Works}

Our work builds on previous advancements in imitation learning and model-based reinforcement learning.

\paragraph{Imitation Learning} Recent advancements in imitation learning (IL) have leveraged deep neural networks and diverse methodologies to enhance performance. Generative Adversarial Imitation Learning (GAIL) \citep{ho2016generative} laid the foundation for adversarial reward learning by formulating it as a min-max optimization problem inspired by Generative Adversarial Networks (GANs) \citep{goodfellow2014generative}. Several approaches have built on GAIL. Model-based Adversarial Imitation Learning (MAIL) \citep{baram2016model} extended GAIL with a forward model trained via data-driven methods. ValueDICE \citep{kostrikov2019imitation} transformed the adversarial framework by focusing on off-policy learning through distribution ratio estimation.

Offline imitation learning has seen significant advancements through approaches like Diffusion Policy \citep{chi2023diffusionpolicy}, which applied diffusion models for behavioral cloning, and Ditto \citep{demoss2023ditto}, which combined Dreamer V2 \citep{hafner2020mastering} with adversarial techniques. Implicit BC \citep{florence2022implicit} demonstrated that supervised policy learning with implicit models improves empirical performance in robotic tasks. DMIL \citep{zhang2023discriminator} leveraged a discriminator to assess dynamics accuracy and the suboptimality of model rollouts against expert demonstrations in offline IL.

Other innovations focused on integrating advanced reinforcement learning techniques. Inverse Soft Q-Learning (IQ-Learn) \citep{garg2021iq} reformulated GAIL's learning objectives, applying them to soft actor-critic \citep{haarnoja2018soft} and soft Q-learning agents. SQIL \citep{reddy2019sqil} contributed an online imitation learning algorithm utilizing soft Q-functions. CFIL \citep{freund2023coupled} introduced a coupled flow method for simultaneous reward generation and policy learning from expert demonstrations. Random Expert Distillation (RED) \citep{wang2019random} proposed an alternative method for constructing reward models by estimating the support of the expert policy distribution.

Model-based methods have also played a pivotal role in advancing IL. V-MAIL \citep{rafailov2021visual} employed variational models to facilitate imitation learning, while CMIL \citep{kolev2024efficient} utilized conservative world models for image-based manipulation tasks. Prior works \citep{englert13model, hu2022model, igl2022symphony} highlighted the potential of model-based imitation learning in real-world robotics control and autonomous driving. A model-based inverse reinforcement learning approach by \citet{das2021model} explored key-point prediction to improve performance in imitation tasks. Hybrid Inverse Reinforcement Learning \citep{ren2024hybrid} offered a novel strategy blending online and expert demonstrations, enhancing agent robustness in stochastic settings. EfficientImitate \citep{yin2022planning} fused EfficientZero \citep{ye2021mastering} with adversarial imitation learning, achieving impressive performance on DMControl tasks \citep{tassa2018deepmind}.

\paragraph{Model-based Reinforcement Learning} Recent advancements in model-based reinforcement learning (MBRL) utilize learned dynamics models, constructed via data-driven methodologies, to enhance agent learning and decision-making. MBPO \citep{janner2019trust} introduced a model-based policy optimization algorithm that ensures stepwise monotonic improvement. Extending this to offline RL, MOPO \citep{yu2020mopo} incorporated a penalty term in the reward function based on the uncertainty of the dynamics model to manage distributional shifts effectively. MBVE \citep{feinberg2018model} augmented model-free agents with model-based rollouts to improve value estimation.

Many approaches focus on constructing dynamics models in latent spaces. PlaNet \citep{hafner2019learning} pioneered this direction by proposing a recurrent state-space model (RSSM) with an evidence lower bound (ELBO) training objective, addressing challenges in partially observed Markov decision processes (POMDPs). Building on PlaNet, the Dreamer algorithms \citep{hafner2019dream, hafner2020mastering, hafner2023mastering} leveraged learned world models to simulate future trajectories in a latent space, enabling efficient learning and planning. The TD-MPC series \citep{hansen2022temporal, hansen2023td} further refined latent-space modeling by developing a scalable world model for model predictive control, utilizing a temporal-difference learning objective to improve performance. Similarly, MuZero \citep{schrittwieser2020mastering} combined a latent dynamics model with tree-based search to achieve strong performance in discrete control tasks, blending planning and policy learning seamlessly. The EfficientZero series \citep{ye2021mastering,wang2024efficientzero} enhances MuZero, achieving superior sampling efficiency in visual reinforcement learning tasks.

\end{document}